%% file: main.tex
\documentclass{article}

 \usepackage[final]{neurips_2023}

\usepackage[utf8]{inputenc} 
\usepackage[T1]{fontenc}    
\usepackage{url}            
\usepackage{booktabs}       
\usepackage{amsfonts}       
\usepackage{nicefrac}       
\usepackage{microtype}      
\usepackage{subcaption}
\usepackage{amsmath}
\usepackage{bm}
\usepackage{bbm}
\usepackage{algorithmic}
\usepackage{algorithm}
\usepackage{tabularx}
\usepackage{amssymb,amsthm}
\usepackage{natbib}
\usepackage[pdftex]{graphicx}

\usepackage[usenames,dvipsnames]{xcolor}
\definecolor{shadecolor}{gray}{0.9}

\usepackage[colorlinks,linktoc=all]{hyperref}
\hypersetup{citecolor=MidnightBlue}
\hypersetup{linkcolor=MidnightBlue}
\hypersetup{urlcolor=MidnightBlue}
\usepackage{url}

\newcommand{\Real}[0]{\mathbb{R}}
\newcommand{\Exp}[0]{\mathbb{E}}

\theoremstyle{plain}

\theoremstyle{plain}

\theoremstyle{plain}

\theoremstyle{plain}
\newtheorem{prop}{\protect\propositionname}
\theoremstyle{plain}
\newtheorem{lem}{\protect\lemmaname}
\theoremstyle{plain}

\theoremstyle{plain}

\providecommand{\assumptionname}{Assumption}
\providecommand{\lemmaname}{Lemma}
\providecommand{\propositionname}{Proposition}
\providecommand{\theoremname}{Theorem}
\providecommand{\definitionname}{Definition}
\providecommand{\examplename}{Example}
\providecommand{\remarkname}{Remark}

\newcommand{\comm}[1]{}

\newcommand{\invfisher}{A} 
\newcommand{\fisher}{\mathcal{I}} 

\newcommand{\sqrtmat}{R}

\title{
Optimal Preconditioning and Fisher Adaptive Langevin Sampling 
 }

\author{%
 Michalis K. Titsias \\
 Google DeepMind \\
  \texttt{mtitsias@google.com} \\
}

\begin{document}

\maketitle

\begin{abstract}
We define  an optimal preconditioning for the 
Langevin diffusion by analytically optimizing the expected squared jumped distance. 
This yields as the optimal preconditioning an inverse Fisher information covariance matrix, where the covariance 
matrix is computed as the outer product of log target gradients averaged under the target.  We apply this result to the Metropolis adjusted Langevin algorithm (MALA)  and derive a computationally efficient adaptive MCMC scheme
that learns the preconditioning from the history of gradients produced as the algorithm runs. 
We show in several experiments that the proposed 
algorithm is very robust in high dimensions and significantly outperforms other methods, including a closely related   
adaptive MALA scheme that learns the preconditioning with standard adaptive MCMC as well as the position-dependent  Riemannian manifold MALA sampler.
 \end{abstract}


\input{sec_introduction}

\input{sec_background}

\input{sec_method}

\input{sec_experiments}

\input{sec_related}

\input{sec_conclusion}

\begin{ack}
We are grateful to the reviewers for their comments. 
Also, we wish to thank Arnaud Doucet, Sam Power, Francisco Ruiz, Jiaxin Shi, Yee Whye Teh, Siran Liu, Kazuki Osawa and James Martens 
for useful discussions.      
\end{ack}

{
\bibliographystyle{plain}
\bibliography{my_bibliography}
}
\clearpage

\input{sec_appendix}

\end{document}

%% file: sec_introduction.tex
\section{Introduction}
\label{sec:introduction}

Markov chain Monte Carlo (MCMC) is a general framework for simulating from arbitrarily complex distributions, and it has shown to be   
useful for statistical inference in a wide range of problems \cite{gilks1995markov, brooks2011handbook}.  
The main idea of an MCMC algorithm is quite simple.  Given a complex target $\pi(x)$, 
a Markov chain is constructed using a $\pi$-invariant transition kernel that allows to simulate dependent realizations $x_1,x_2,\ldots$ that eventually converge to samples from $\pi$. These samples 
can  be used for Monte Carlo integration by forming ergodic averages.  A general way  to define  $\pi$-invariant transition kernels is the 
Metropolis-Hastings accept-reject mechanism in which the chain moves from state $x_n$ to the next state $x_{n+1}$ by first generating a candidate state $y_n$ from a proposal distribution 
$q(y_n|x_n)$ and then it sets $x_{n+1}=y_n$ with probability $\alpha(x_n, y_n)$:
\begin{equation}
\alpha(x_n, y_n) = \text{min}(1, a_n), \ \ a_n = \frac{\pi(y_n)}{\pi(x_n)} \frac{q(x_n|y_n)}{q(y_n|x_n)},
\label{eq:alpha}
\end{equation}
or otherwise rejects $y_n$ and sets $x_{n+1}=x_n$.  The choice of the proposal distribution $q(y_n|x_n)$ is crucial because it determines the mixing 
of the chain, i.e.\ the dependence of samples across time. For example, a "slowly mixing" chain even after convergence may not be useful for Monte Carlo integration since
it will output a highly dependent set of samples producing ergodic estimates of very high variance.  Different ways of defining $q(y_n|x_n)$ lead to common algorithms  
such as random walk Metropolis  (RWM), Metropolis-adjusted Langevin algorithm (MALA) \cite{Rossky78,  roberts1998optimal} and Hamiltonian Monte Carlo (HMC) \cite{Duane1987,neal2010}.  
Within each class of these algorithms adaptation of parameters of the proposal distribution, such as a step size, is also important and this has been widely studied in the literature 
by producing optimal scaling results \citep{roberts1997weak,  roberts1998optimal, roberts2001optimal, haario2005componentwise, bedard2007weak, bedard2008optimal, roberts2009examples, bedard2008efficient,  rosenthal2011optimal, beskos2013optimal},  and also by developing adaptive MCMC algorithms 
\citep{haario2001adaptive, atchade2005adaptive, roberts2007coupling, giordani2010adaptive, andrieu2006ergodicity, andrieu2007efficiency, atchade2009adaptive, ensemblepreconditioning}. 
The standard adaptive MCMC procedure in \cite{haario2001adaptive} uses the history of the chain to recursively compute 
an empirical covariance of the target $\pi$ and build a multivariate Gaussian proposal distribution.  However, 
this type of covariance adaptation can be too slow and not so robust in high dimensional settings \citep{roberts2009examples,andrieu2008tutorial}.  

In this paper, we derive a fast and very robust  adaptive MCMC technique in high dimensions that learns a preconditioning matrix for the MALA method, which is  the standard gradient-based 
MCMC algorithm obtained by a first-order discretization of the continuous-time Langevin diffusion.    
Our first contribution 
is to define an optimal preconditioning  by analytically optimizing a criterion on the Langevin diffusion.  
The criterion is the well-known expected squared jumped distance \cite{pasarica2010adaptively} which at optimum yields as a preconditioner the inverse matrix $\mathcal{I}^{-1}$ of the following
Fisher information covariance matrix $\mathcal{I} = \Exp_{\pi(x)} 
\left[ \nabla \log \pi(x) 
\nabla \log \pi(x)^\top \right]$. This contradicts the common belief 
 in adaptive MCMC 
 that the covariance  of $\pi$ is the best preconditioner. While this is a  surprising result we show that  $\mathcal{I}^{-1}$ connects with a certain quantity appearing in 
 optimal scaling 
 of RWM \citep{roberts1997weak,  roberts2001optimal}.   

Having recognized $\mathcal{I}^{-1}$ as the optimal preconditioning we derive an easy to implement and computationally efficient adaptive MCMC algorithm
that learns from the history of gradients produced as MALA runs. This method sequentially updates an empirical inverse Fisher estimate $\hat{\mathcal{I}}_n^{-1}$ 
using a 
 recursion having quadratic cost  $O(d^2)$ ($d$ is the dimension of $x$) per iteration.  In practice, 
since for sampling we need a square root matrix of  $\hat{\mathcal{I}}_n^{-1}$  we implement the recursions over a square root matrix by adopting classical results from  
Kalman filtering \citep{Potter1963,Bierman1977}.  We compare our method against MALA that learns the preconditioning with standard adaptive MCMC \citep{haario2001adaptive}, 
a position-dependent Riemannian manifold MALA  \cite{GirolamiCalderhead11}
as well as simple MALA (without preconditioning) and HMC. In several experiments we show that the proposed 
algorithm significantly outperforms all other methods.          

%% file: sec_background.tex
\section{Background}
\label{sec:background}

We consider an intractable target distribution $\pi(x)$ with $x \in \Real^d$, 
known up to some normalizing constant,  and we assume that 
$ \nabla \log \pi(x) := \nabla_x \log \pi(x)$ is well defined. A 
continuous time process with stationary distribution $\pi$
is the overdamped Langevin diffusion 
\begin{equation}
d x_t = \frac{1}{2} A \nabla \log \pi(x_t) d t + \sqrt{A} d B_t, 
\label{eq:langevin}
\end{equation}
where $B_t$ denotes $d$-dimensional Brownian motion. 
This is a stochastic differential equation (SDE) that generates sample paths such that for large $t$, $x_t \sim \pi$. 
We also incorporate a \emph{preconditioning matrix} $A$, which is a symmetric positive definite covariance matrix, while  $\sqrt{A}$ is such that  $\sqrt{A} \sqrt{A}^\top = A$. 

Simulating from the SDE in \eqref{eq:langevin} is intractable and the standard  approach is to use a first-order Euler-Maruyama discretization combined with a Metropolis-Hastings adjustment. 
This leads to the so called preconditioned \emph{Metropolis-adjusted Langevin algorithm} (MALA) where at each iteration given the current state $x_n$ (where $n=1,2,\ldots$) we sample $y_n$ from the proposal distribution 
\begin{equation}
q(y_n|x_n)  = \mathcal{N}(y_n|x_n + \frac{\sigma^2}{2} 
A \nabla \log \pi(x_n), \sigma^2 
A),
\label{eq:MALAproposal}
\end{equation}
where  the step size $\sigma^2 > 0 $ appears due to 
time 
discretization. We accept $y_n$ with probability 
$\alpha(x_n, y_n) = \min\left(1, a_n \right)$  where 
$a_n$ follows the 
form in \eqref{eq:alpha}. 
The obvious way to compute 
$a_n$ is 
$$
a_n = \frac{\pi(y_n)}{\pi(x_n)}\frac{q(x_n|y_n)}{q(y_n|x_n)}
= \frac{\pi(y_n)}{\pi(x_n)}\frac{\exp\{
- \frac{1}{2 \sigma^2} ||x_n - y_n - \frac{\sigma^2}{2} A \nabla \log \pi(y_n)||_{A^{-1}}^2 
\}}{
\exp\{
- \frac{1}{2 \sigma^2} ||y_n - x_n - \frac{\sigma^2}{2} A \nabla \log \pi(x_n)||_{A^{-1}}^2
\}},
$$
where $||z||_{A^{-1}}^2 = z^\top A^{-1} z$. However, in some cases that involve high dimensional targets, this can be costly 
since in the ratio of proposal 
densities both the  preconditioning matrix $A$ and its inverse 
$A^{-1}$ appear. In turns out
that we can avoid $A^{-1}$ and simplify the computation as stated below. 

\begin{prop}
\label{prop:accratio}
For preconditioned MALA with proposal density given by  \eqref{eq:MALAproposal} the ratio of proposals in the M-H acceptance probability can be written as
$$
\frac{q(x_n|y_n)}{q(y_n|x_n)} 
 = \exp\{
h(x_n,y_n) - h(y_n,x_n)
\}, \ \ \ h(z, v) 
= \frac{1}{2}
\left(z \! - \! v - \! \frac{\sigma^2}{4} A \nabla \log \pi(v)  
\right)^\top \! \! \! \nabla \log \pi(v). 
$$
\end{prop}
This expression does not depend on the inverse $A^{-1}$, and this leads to  computational gains and simplified implementation that we exploit in the adaptive MCMC algorithm presented in Section \ref{sec:adaptiveMCMC}.   

The motivation behind the use of preconditioned MALA is that with a suitable preconditioner $A$ the mixing of the chain  
can be drastically improved,  especially for very anisotropic target distributions.  
A very general way to specify $A$ is by applying an adaptive MCMC algorithm, which learns $A$ online. To design such an algorithm it is useful to first specify a notion of optimality. A common argument in the literature, that is used for both RWM and MALA, is that a suitable $A$ is the unknown covariance matrix $\Sigma$ \cite{haario2001adaptive, roberts2009examples, ensemblepreconditioning} of the target $\pi$.
This means that we should learn $A$ so that to approximate $\Sigma$.
However,  this argument is rather heuristic since it is not based on an optimality criterion. One of our contributions is to specify an optimal $A^*$ based on an 
optimization procedure, that we describe in Section \ref{sec:optimalA}. This $A^*$ will turn out to be not the covariance matrix of the target but an inverse
 Fisher information matrix.

%% file: sec_method.tex
\section{Optimal preconditioning using expected squared jumped distance  
\label{sec:optimalA}}

Preconditioning aims to improve sampling when different directions (or individual variables $x_i$) in the state space can have different scalings under the target $\pi$.
Here, we develop a method for selecting the preconditioning through the optimization of an objective function.  This method uses the observation that an effective preconditioning  correlates with large values of the global step size $\sigma^2$ in MALA, i.e.\  $\sigma^2$ is allowed to increase when preconditioning 
becomes effective as shown in the sampling efficiency scores in Table \ref{table:large_table} and the corresponding estimated step sizes  
reported in Appendix \ref{app:step_size}.  
 
In our analysis we consider the rejection-free or unadjusted Langevin sampler  
where we discretize the  time continuous Langevin diffusion in \eqref{eq:langevin} with a small finite $\delta :=\sigma^2 > 0$  
so that 
\begin{equation} 
x_{t + \delta} - x_t = \frac{\delta}{2} A \nabla \log \pi(x_t)  + \sqrt{A} (B_{t+ \delta} - B_t),  \ \ \text{where} \ \  B_{t+ \delta} - B_t \sim \mathcal{N}(0,\delta I).
\label{eq:discretizesLangevin}
\end{equation}
We will use  the expected squared jumped distance $J(\delta, A) = \Exp[||x_{t+\delta}  - x_t ||^2]$ 
computed as follows.  

\begin{prop}
\label{prop:langevincovar}
If $x_t \sim \pi(x_t)$ the vector 
 $x_{t + \delta} - x_t $ defined by \eqref{eq:discretizesLangevin} has zero mean 
 and covariance
\begin{equation}
\Exp[(x_{t + \delta} - x_t) (x_{t + \delta} - x_t )^\top] 
= \frac{\delta^2}{4} 
A \Exp_{\pi(x_t)} 
\left[ \nabla \log \pi(x_t) 
\nabla \log \pi(x_t)^\top \right] A 
+ \delta A.
\label{eq:langevincovar} 
\end{equation}
Further,  $\text{tr}\left( \Exp[(x_{t + \delta} - x_t) (x_{t + \delta} - x_t )^\top] \right) 
= \Exp[ \text{tr}\left((x_{t + \delta} - x_t) (x_{t + \delta} - x_t )^\top \right)] = 
  \Exp[||x_{t + \delta} - x_t||^2]$, which shows that $J(\delta,A)$ is the trace of the covariance matrix in 
 \eqref{eq:langevincovar}.
\end{prop}
To control discretization error we impose an upper bound constraint $J(\delta, A) \leq \epsilon$ for a small $\epsilon > 0$. A preconditioning 
 that "symmetrizes" the target 
 can be obtained by maximizing the discretization step size $\delta$
 subject to $J(\delta, A) \leq \epsilon$. Since $J(\delta, A)$ monotonically increases with $\delta$, the maximum 
 $\delta^*$ satisfies $\min_A J(\delta^*, A) = \epsilon$. This means that 
 the \emph{optimal} preconditioning $A^*$ is obtained by 
 minimizing $J(\delta, A)$ under some global scale constraint on $A$, as stated next.

\begin{prop}
\label{prop:optimalprecond}
Suppose $A$ is a symmetric positive definite matrix satisfying $\text{tr}(A) = c$, with $c>0$ a constant.  Then the objective $J(\delta, A)$, for any $\delta>0$,  is miminized for $A^*$ given by
\begin{equation}
A^* = k \mathcal{I}^{-1}, \ \ k=\frac{c}{\sum_{i=1}^d \frac{1}{\mu_i} }, \ \ \ \mathcal{I} = \Exp_{\pi(x)} 
\left[ \nabla \log \pi(x) 
\nabla \log \pi(x)^\top \right],
\label{eq:optimalA_and_fisher}
\end{equation}
where $\mu_i$s are the eigenvalues of $\mathcal{I}$ assumed to satisfy $0 < \mu_i < \infty$. 
\end{prop}
The positive multiplicative scalar $k$ in \eqref{eq:optimalA_and_fisher}
is not important since the specific value $c>0$ is arbitrary, e.g.\ 
 if we choose $c = \sum_{i=1}^d \frac{1}{\mu_i}  $ then $k=1$ and $A^* =\mathcal{I}^{-1}$. 
In other words, what matters is that the optimal $A^*$ is proportional to the inverse matrix $\mathcal{I}^{-1}$, so it follows the curvature of $\mathcal{I}^{-1}$.  For a multivariate Gaussian 
$\pi(x) = \mathcal{N}(x|\mu, \Sigma)$ it holds 
$\mathcal{I}^{-1} = \Sigma$, so the optimal preconditioner 
coincides with the covariance matrix of $x$. More generally though, 
for non-Gaussian targets this  will not hold.

\paragraph{Connection with classical Fisher information matrix.}  The matrix $\mathcal{I}$  is 
positive definite since it is the 
covariance of the gradient $\nabla \log \pi(x) := \nabla_x \log \pi(x)$ where $\Exp_{\pi(x)} [\nabla \log \pi(x)] = 0$. 
Also, $\mathcal{I}$ is similar to the classical Fisher information matrix.
To illustrate some differences suppose that the target $\pi(x)$ is a Bayesian posterior $\pi(\theta | Y) \propto p(Y|\theta) p(\theta) = p(Y,\theta) $ where  $Y$ are the observations and $\theta:=x$ are the random parameters. 
The classical Fisher information is a \emph{frequentist} quantity where we fix some parameters $\theta$ and compute
$G(\theta) = \Exp_{p(Y|\theta)} [\nabla_\theta \log p(Y|\theta)  \nabla_\theta \log  p(Y|\theta)^\top]$ by averaging over data. In contrast, $\mathcal{I} = \Exp_{p(\theta|Y)} [ \nabla_\theta \log p(Y, \theta)  \nabla_\theta \log p(Y, \theta)^\top] $  
is more like a \emph{Bayesian} quantity where we fix the data $Y$ and average over the parameters $\theta$. Importantly, $\mathcal{I}$ is not a function of $\theta$ while $G(\theta)$ is.  
Similarly to the classical Fisher information, $\mathcal{I}$ also satisfies the following standard property: Given that $\log \pi(x)$ is twice differentiable
and $\nabla^2_x \log \pi(x)$ is the Hessian matrix,  
$\mathcal{I}$ from \eqref{eq:optimalA_and_fisher} is also written as
$
\mathcal{I} = - \Exp_{\pi(x)} 
[ \nabla^2_x \log \pi(x)]
$.  Next, 
we refer to $\mathcal{I}$ as the Fisher matrix.  

\paragraph{Connection with optimal scaling.} The Fisher matrix $\mathcal{I}$ 
connects also with the optimal scaling result for the RWM algorithm \citep{roberts1997weak, roberts2001optimal}. Specifically, 
for  targets of the form   
$\pi(x) = \prod_{i=1}^d f(x_i)$, the RWM proposal $q(y_n|x_n) = \mathcal{N}(y_n | x_n, (\sigma^2/d) I_d)$ and as $d \rightarrow \infty$  
the optimal parameter $\sigma^2$ is $\sigma^2 = \frac{2.38}{J}$ 
where $J = \Exp_{f(x)}[( \frac{d \log f(x)}{d x} )^2 ]$ is the (univariate) Fisher information for the univariate density $f(x)$, and the preconditioning involves as in our case the inverse Fisher $\mathcal{I}^{-1} = \frac{1}{J} I_d$. This result has been generalized also for heterogeneous targets in 
\cite{roberts2001optimal} where again the inverse Fisher information 
matrix (having now a more general diagonal form) appears as the optimal preconditioner.


\section{Fisher information adaptive MALA
\label{sec:adaptiveMCMC}
}

Armed with the previous optimality result, we wish to develop an adaptive MCMC algorithm to optimize the proposal in \eqref{eq:MALAproposal} by learning online the global variance $\sigma^2$ and the preconditioner $A$. For $\sigma^2$ we follow the standard practice to tune this parameter in order to reach an average acceptance rate around $0.574$ as suggested by optimal scaling results  \citep{roberts1998optimal,roberts2001optimal}. For the matrix $A$ 
we want to adapt it so that  approximately it becomes proportional to the inverse Fisher $\mathcal{I}^{-1}$ 
from \eqref{eq:optimalA_and_fisher}. We also incorporate  a parametrization that helps the adaptation of $\sigma^2$ to be more independent from the one of $A$. Specifically, we remove the global scale from $A$  
by defining the overall proposal as  
\begin{equation}
q(y_n|x_n) = \mathcal{N}\left(y_n|x_n + \frac{\sigma^2}{ \frac{2}{d} \text{tr}(A)} 
A \nabla \log \pi(x_n), \frac{\sigma^2}{\frac{1}{d} \text{tr}(A)}  
A\right), 
\label{eq:normalizedPMALA}
\end{equation}
where $\sigma^2$ is normalized by $\frac{1}{d} \text{tr}(A)$, i.e.\  the average eigenvalue of $A$. Another way to view this is that the effective preconditioner is $A / (\frac{1}{d}\text{tr}(A))$ which has an average eigenvalue equal to one. The proposal in \eqref{eq:normalizedPMALA} is invariant to any scaling of $A$, i.e.\ if $A$ is replaced by $k A$ (with $k>0$) the proposal remains the same. Also, note that when $A$ is the identity matrix $I_d$ (or a multiple of identity) then $\frac{1}{d} \text{tr}(I_d)=1$ and the above proposal reduces to standard MALA with isotropic step size $\sigma^2$. 

It is straightforward to adapt $\sigma^2$ 
towards an average acceptance rate $0.574$; see pseudocode in Algorithm \ref{alg:FisherMALA}. Thus our main focus next is to describe the learning update for $A$, in fact 
eventually not for $A$ itself but for a square root matrix$\sqrt{A}$ which is what we need to sample from the proposal in \eqref{eq:normalizedPMALA}. 

To start with, let us simplify notation by writing the score function at the $n$-th MCMC iteration as
$s_n : = \nabla_{x_n} \log \pi(x_n)$. We introduce the 
$n$-sample empirical Fisher estimate
\begin{equation}
\hat{\mathcal{I}}_n = \frac{1}{n} 
\sum_{i=1}^n s_i s_i^\top + \frac{\lambda}{n} I_d,
\label{eq:empiricalFisher}
\end{equation}
where $\lambda > 0$ is a fixed damping parameter. Given that certain conditions apply \citep{haario2001adaptive, roberts2009examples} so that the chain converges 
and ergodic averages converge to exact expected values,  $\hat{\mathcal{I}}_n$ is a consistent estimator satisfying $\lim_{n \rightarrow \infty} \hat{\mathcal{I}}_n = \mathcal{I}$ since as $n \rightarrow \infty$ the damping part $\frac{\lambda}{n} I_d$ vanishes. Including the damping 
is very important since it offers a Tikhonov-like regularization, similar to ridge regression, and it ensures that for any finite $n$ the eigenvalues of $\hat{\mathcal{I}}_n$ are strictly positive. 
An  estimate then for the 
preconditioner $A_n$ can be set to be proportional to the inverse of the empirical Fisher $\hat{\mathcal{I}}_n$, i.e.\ 
\begin{equation}
A_n \propto \left( \frac{1}{n} \sum_{i=1}^n s_i s_i^\top + \frac{\lambda}{n} I_d \right)^{-1}
=  n \left(\sum_{i=1}^n s_i s_i^\top + \lambda I_d \right)^{-1}. 
\end{equation}
Since any positive multiplicative scalar in front of $A_n$ plays no
role, we can ignore the scalar $n$ and define  $A_n = (\sum_{i=1}^n s_i s_i^\top + \lambda I_d )^{-1}$. Then, as MCMC iterates we can adapt $A_n$ in $O(d^2)$ cost per iteration   
based on the recursion
\begin{align}
& 
\text{Initialization:} \ 
A_1 = \left( 
s_1 s_1^\top + \lambda I_d 
\right)^{-1}
= \frac{1}{\lambda}
\left(I_d - \frac{s_1 s_1^\top}{\lambda + s_1^\top s_1} \right),  \\
& \text{Iteration:} \  A_n  = \left(A_{n-1}^{-1} + s_{n} s_{n}^\top  \right)^{-1} = A_{n-1} -  
\frac{A_{n-1} s_n s_n^\top A_{n-1}}{1 + s_n^\top A_{n-1} s_n},
\end{align}
where we applied Woodbury matrix identity. This estimation in the limit can give the optimal preconditioning in the sense  
that under the ergodicity assumption, 
$
\lim_{n \rightarrow \infty} 
\frac{A_n}{\text{tr}(A_n)} = \frac{\mathcal{I}^{-1}}
{\text{tr}(\mathcal{I}^{-1})}
$.
In practice we do not need to compute directly the matrix $A_n$  but  a square root matrix $R_n := \sqrt{A}_n$,
such that $R_n R_n^\top = A_n$, since we need a square root matrix to draw samples from the proposal in \eqref{eq:normalizedPMALA}.  
To express the corresponding recursion for $R_n$ we will rely on a technique that dates back to the early days of Kalman filtering \citep{Potter1963,Bierman1977},
which applied to 
our case gives the following result.   
\begin{prop}
\label{prop:recursion}
A square root matrix $R_n$, such that $R_n R_n^\top = A_n$, can be computed recursively in $O(d^2)$ time per iteration as follows:
\begin{align}
& \text{Initialization:} 
\ R_1
= \frac{1}{\sqrt{\lambda}}
\left(I_d - r_1 \frac{s_1 s_1^\top}{\lambda + s_1^\top s_1} \right),  \ \ r_1 = \frac{1}{1 + \sqrt{\frac{\lambda}{\lambda + s_1^\top s_1}}} \\
& 
\text{Iteration:} \ R_n  
= R_{n-1} -  
r_n \frac{(R_{n-1} \phi_n) \phi_n^\top}{1 + \phi_n^\top \phi_n}, \ \ \phi_n = R_{n-1}^\top s_n, \ r_n = \frac{1}{1 + \sqrt{\frac{1}{1 + \phi_n^\top \phi_n}}}.
\end{align}
\label{propSQRT}
\end{prop}
A way to generalize the above recursive 
estimation of a square root
for the inverse Fisher matrix  is to consider the stochastic approximation framework 
\citep{RobbinsMonro1951}. This requires to write  an online learning update for the empirical Fisher of the form 
\begin{equation}
\hat{\mathcal{I}}_n 
= \hat{\mathcal{I}}_{n-1} + 
\gamma_n (s_n s_n^\top - \hat{\mathcal{I}}_{n-1}),  \ \ \text{initialized at} \ \hat{\mathcal{I}}_1 = s_1 s_1^\top + \lambda I_d,
\label{eq:empiricalFisherGamma}
\end{equation}
where the learning rates $\gamma_n$ satisfy the standard conditions $\sum_{n=1}^\infty \gamma_n = \infty, \sum_{n=1}^\infty \gamma_n^2 < \infty$. 
Then, it is straightforward to generalize the recursion for the square root matrix in Proposition \ref{propSQRT} to account for this more general case; see Appendix \ref{app:generallearningrates}. The  recursion in Proposition \ref{propSQRT}  is a special case when   
$\gamma_{n} = \frac{1}{n}$. In our simulations we did not observe significant improvement by using more general learning rate sequences, and therefore in all our experiments in Section \ref{sec:experiments} we use the \emph{standard} learning rate $\gamma_n = \frac{1}{n}$. Note that this learning rate is also used by other adaptive MCMC methods \citep{haario2001adaptive}.

An  adaptive algorithm 
 that learns online from the score function vectors 
 $s_n$ can work well in some cases, but still it can be unstable in general. One reason is that $s_n = \nabla \log \pi(x_n)$ will not have zero expectation when the chain is transient 
 and states $x_n$ are not yet draws from the stationary  distribution $\pi$.  To analyze this, note that 
 the learning signal $s_n$ enters in the empirical Fisher estimator  $\mathcal{\bar{I}}_n$ through the outer product $s_n s_n^\top$ as shown by 
 Eqs.\ \eqref{eq:empiricalFisher} and \eqref{eq:empiricalFisherGamma}. However, in the transient phase $s_n s_n^\top$  will be biased since the expectation 
  $\Exp[s_n s_n^\top] =  \Exp[(s_n -  \Exp[s_n]) (s_n -  \Exp[s_n])^\top]  +  \Exp[s_n]  \Exp [s_n]^\top \neq \mathcal{I}$, where the expectations are taken under the 
  marginal distribution of the chain at time $n$.  In practice the mean vector $\Exp[s_n]$ can 
  take large absolute values, which can introduce significant bias through the additive term $\Exp[s_n]  \Exp [s_n]^\top$. Thus, to reduce some bias we could track the empirical 
 mean $\bar{s}_n = \frac{1}{n} \sum_{i=1}^n s_i$ and  center the signal $s_n - \bar{s}_n$ so that the Fisher matrix is estimated by the empirical covariance 
 $\frac{1}{n-1} \sum_{i=1}^n (s_i - \bar{s}_n)( s_i - \bar{s}_n)^\top$.  The recursive estimation becomes similar to standard adaptive MCMC \cite{haario2001adaptive} 
 where we recursively propagate an online empirical estimate for the mean of $s_n$ and incorporate it into the online empirical estimate of the covariance matrix (in our case the inverse Fisher matrix);  see  Eq.\  \eqref{eq:adaMALA} in Section \ref{sec:experiments} for the standard adaptive MCMC recursion \cite{haario2001adaptive} and Appendix \ref{app:FisherMALA2} for our Fisher method. 
 While this can make learning quite stable we experimentally discovered that there is another scheme, presented next in Section \ref{sec:increments_andRB}, that  
 is significantly better and stable especially for very anisotropic high dimensional targets; see detailed results 
 in Appendix \ref{sec:effectRaoblackwell}.  
   

 
 \subsection{Adapting to 
 score function increments
\label{sec:increments_andRB}
 }
 
 
            
An MCMC algorithm updates at each iteration its state according to 
 $
 x_{n+1} = x_n + \mathbf{I}(u_n < \alpha(x_n,y_n)) (y_n - x_n)
 $
 where $\alpha(x_n,y_n)$ is the M-H probability, $u_n \sim U(0,1)$ is an uniform random number 
 and $\mathbf{I}(\cdot)$ is the indicator function. This sets $x_{n+1}$ to either the proposal $y_n$ 
 or the previous state $x_n$ based on the binary value 
$\mathbf{I}(u_n < \alpha(x_n,y_n))$. Similarly, we can consider the update   
 of the score function $s(x) = \nabla \log \pi(x)$ and conveniently re-arrange it as an increment,  
\begin{equation}
s_n^{\delta}  = s(x_{n+1}) - s(x_n) = \mathbf{I}(u_n < \alpha(x_n,y_n)) (s(y_n) - s(x_n)). 
\label{eq:differenceNoRB}
 \end{equation}
While both $s_n$ and $s_n^\delta$ have zero expectation when $x_n$ is from stationarity, i.e.\ $x_n \sim \pi$,
 the increment $s_n^{\delta}$ (unlike $s_n$) tends in practice to be more centered and close to zero even when the chain is transient, e.g.\  note 
 that $s_n^{\delta}$ is zero when $y_n$ is rejected.  Further, since the difference $s_n^\delta = s(x_{n+1}) - s(x_n)$ conveys information 
about the covariance of the score function 
we can use it in the recursion of Proposition \ref{prop:recursion} to learn the preconditioner $A$, where
we simply replace $s_n$ by $s_n^\delta$.  As shown in the experiments this leads to a remarkably fast and effective adaptation of the inverse Fisher matrix $\mathcal{I}^{-1}$ without observable bias, or at least  no 
observable for Gaussian targets where the true  $\mathcal{I}^{-1}$ is known.    
We can further apply 
Rao-Blackwellization to reduce some variance of $s_n^\delta$.  Since $s_n^\delta$ enters 
into the estimation of the empirical Fisher, see Eq.\ \eqref{eq:empiricalFisher} or \eqref{eq:empiricalFisherGamma}, through the outer product $s_n^\delta (s_n^\delta)^\top = 
 \mathbf{I}(u_n < \alpha(x_n,y_n)) (s(y_n) - s(x_n))  (s(y_n) - s(x_n))^\top$ we can 
 marginalize out the r.v.\ $u_n$ which yields  $\Exp_{u_n}[s_n^\delta (s_n^\delta)^\top]  = \alpha(x_n,y_n) (s(y_n) - s(x_n))  (s(y_n) - s(x_n))^\top$. After this 
 Rao-Blackwellization an alternative vector to use for adaptation is
 \begin{equation}
s_n^{\delta}  = \sqrt{\alpha(x_n,y_n)} (s(y_n) - s(x_n)), 
\label{eq:differenceRB}
\end{equation}
which depends on the square root $\sqrt{\alpha(x_n,y_n)}$ of the M-H probability.  As long 
as $\alpha(x_n,y_n)>0$, the learning signal in \eqref{eq:differenceRB} depends on the proposed sample $y_n$ even when it is rejected.    
 
Finally, we can express the full algorithm for Fisher information adaptive MALA 
as outlined by Algorithm \ref{alg:FisherMALA}, which adapts by using   
the Rao-Blackwellized score function increments from Eq.\  \eqref{eq:differenceRB}. Note that, while Algorithm  \ref{alg:FisherMALA} uses $s_n^\delta$ from Eq.\  \eqref{eq:differenceRB}, 
the initial signal from Eq.\ \eqref{eq:differenceNoRB} works equally well; see 
Appendix \ref{sec:effectRaoblackwell}.  Also, the algorithm includes an initialization phase where simple MALA  runs for few iterations to move away from the initial state, 
as discussed further  
in Section \ref{sec:experiments}.         
 

\begin{algorithm}[tb]
   \caption{Fisher adaptive MALA (blue lines are ommitted when not adapting $(R,\sigma^2)$)}
   \label{alg:FisherMALA}
\begin{algorithmic}
   \STATE {\bfseries Input:} Log target $\log \pi(x)$;  gradient $\nabla \log \pi(x)$;  $\lambda>0$ (default $\lambda=10$);  $\alpha_*=0.574$      
   \STATE Initialize $x_1$ and $\sigma^2$ by running simple MALA (i.e.\ with $\mathcal{N}(y|x + (\sigma^2 / 2) \nabla \log \pi(x), \sigma^2 I)$) 
 for $n_0$  (default $500$)  iterations where $\sigma^2$ is adapted towards acceptance rate $\alpha_*$
 \STATE Initialize square root matrix $R = I_d$ and compute $(\log \pi(x_1), \nabla \log \pi(x_1))$   
 \STATE  Initialize $\sigma^2_R = \sigma^2$    \ \ \ \ \ \ \ \ \ \ \  \  \ \ \  \ \ \ \ \ \ \ \ \ \ \ \ \ \   {\em  \#  placeholder for the normalized step size  $\sigma^2 /  \frac{1}{d} \text{tr}(R R^\top)$ }
   \FOR{For $n=1,2,3,\ldots,$}
   \STATE:  Propose $y_n = x_n + (\sigma_R^2 /  2) R (R^\top \nabla \log \pi(x_n))  + \sigma_R R \eta, \  \ \eta \sim \mathcal{N}(0, I_d)$
   \STATE: Compute $(\log \pi(y_n), \nabla \log \pi(y_n))$ 
   \STATE: Compute
   $\alpha(x_n, y_n) = \text{min}\left(1, e^{
\log \pi(y_n) + h(x_n,y_n) - \log \pi(x_n) - h(y_n,x_n)}
 \right)$ 
 by using Proposition \ref{prop:accratio}  
   \STATE: \textcolor{blue}{Compute adaptation signal $s_n^{\delta}  = \sqrt{\alpha(x_n, y_n)} (\nabla \log \pi(y_n) -  \nabla \log \pi(x_n))$} 
   \STATE: \textcolor{blue}{Use $s_n^\delta$ to adapt $R$ based on Proposition \ref{prop:recursion} (if $n=1$ use (12) and if $n>1$ use (13))}
   \STATE: \textcolor{blue}{Adapt step size $\sigma^2 \leftarrow \sigma^2 \left[1 + \rho_n (\alpha(x_n, y_n) -  \alpha_*)\right]$}   {\em \# default const learning rate $\rho_n\! =\!0.015$}
    \STATE: \textcolor{blue}{Normalize step size $\sigma_R^2 = \sigma^2 / \frac{1}{d} \text{tr}(R R^\top)$}  \ \ \  \ \ \  \ \ \ \ \ \ \  {\em \#  $\text{tr}(R R^\top) = sum(R \circ R)$ which is $O(d^2)$}
   \STATE: Accept/reject $y_n$ with probability $\alpha(x_n, y_n)$ to obtain  $(x_{n+1}, \log \pi(x_{n+1}), \nabla \log \pi(x_{n+1}))$  
   \ENDFOR
\end{algorithmic}
\end{algorithm}

%% file: sec_experiments.tex

\section{Experiments}
\label{sec:experiments}


\subsection{Methods and experimental setup}


We apply the Fisher information adaptive MALA algorithm (FisherMALA)  
to high dimensional problems and we compare it with the following other samplers. 
 {\bf (i)} The simple MALA sampler with proposal 
$\mathcal{N}(y_n|x_n + (\sigma^2 / 2) \nabla \log \pi(x_n), \sigma^2 I)$ , which adapts only a step size $\sigma^2$ 
 without having a preconditioner.   
 {\bf (ii)} A preconditioned adaptive MALA (AdaMALA) where the proposal follows exactly the from in
 \eqref{eq:normalizedPMALA}  but where the preconditioning matrix is learned using standard adaptive MCMC
 based on the well-known recursion from \cite{haario2001adaptive}:  
\begin{equation}
\mu_n = \frac{n-1}{n}
\mu_{n-1} + \frac{1}{n} x_n,  \ 
\Sigma_n 
= \frac{n-2}{n-1} 
\Sigma_{n-1} + \frac{1}{n} 
(x_n - \mu_{n-1}  
)(x_n - \mu_{n-1})^\top,
\label{eq:adaMALA}
\end{equation}
where the recursion 
is initialized at $\mu_1= x_1$ 
and  
$\Sigma_2 = \frac{1}{2} (x_2 - \mu_1)(x_2 - \mu_1)^\top +  \lambda I$,  and $\lambda>0$ is the damping parameter that plays the same role as in FisherMALA. 
 {\bf (iii)} The  Riemannian manifold MALA (mMALA)  \citep{GirolamiCalderhead11}  which uses position-dependent preconditioning matrix $A(x)$. mMALA 
 in high dimensions runs  slower  than other schemes since the computation of $A(x)$ may involve second derivatives and requires  
 matrix decomposition that costs $O(d^3)$ per iteration.     
  {\bf (iv)} Finally, we include in the comparison  the
 Hamiltonian Monte Carlo (HMC) sampler with a fixed number of 10 leap frog steps
and identity mass matrix.  We leave the possibility to learn with our method a preconditioner in HMC for future work since this is more involved; 
see discussion at Section \ref{sec:conclusion}.
 
For all experiments and samplers
we consider $2 \times 10^4$ burn-in iterations and $2 \times 10^4$ iterations for 
collecting samples. We set $\lambda=10$ in FisherMALA  and AdaMALA.  Adaptation of the proposal distributions, i.e.\ the parameter  $\sigma^2$,  the preconditioning or the step size of HMC,  occurs only  during burn-in and at collection of samples stage the proposal parameters are kept fixed. 
For all three  MALA schemes the global step size $\sigma^2$ is adapted to achieve an acceptance rate around $0.574$ (see Algorithm \ref{alg:FisherMALA}) while the corresponding parameter for HMC is adapted towards  $0.651$ rate \citep{beskos2013optimal}. In FisherMALA from the $2 \times 10^4$ burn-in iterations the first $500$ iterations are used as the initialization phase in
Algorithm \ref{alg:FisherMALA} where samples are generated by just MALA with adaptable $\sigma^2$.  Thus,  only the last $1.95 \times 10^4$ burn-in 
iterations are used to  
adapt the preconditioner. For AdaMALA this initialization scheme proved to be unstable 
and we used a more elaborate scheme, 
as described in Appendix \ref{sec:initializeAdaMALA}.        
 
We compute effective sample size (ESS) scores for each method by using the  $2 \times 10^4$ samples from the collection phase.
We estimate ESS across each dimension of the state  vector $x$, and we report maximum, median and minimum values,  by    
using the built-in method in TensorFlow Probability Python package. 
Also,  we show visualizations that indicate sampling efficiency or effectiveness in estimating the preconditioner (when the ground truth preconditioner is known).

  
 \subsection{Gaussian targets}
 
 
 We consider three examples of multivariate Gaussian targets of the form $\pi(x) = \mathcal{N}(x|\mu,\Sigma)$, where the optimal preconditioner (up to any positive scaling) is the covariance matrix $\Sigma$  since the inverse 
 Fisher is  $\mathcal{I}^{-1} = \Sigma$.  For such case the Riemannian manifold sampler mMALA  \citep{GirolamiCalderhead11} is the optimal MALA sampler since
 it uses precisely $\Sigma$ as the preconditioning.  In contrast to mMALA which somehow has access to the ground-truth oracle,  both FisherMALA and AdaMALA 
 use adaptive recursive estimates of the preconditioner     
that should converge to the optimal
 $\Sigma$, and thus the question is which of them learns faster. 
 To quantify this we compute the Frobenius norm  $||\widetilde{A}_n -  \widetilde{\Sigma}||_F$ across adaptation iterations $n$, where $\widetilde{B}$ denotes  
 the matrix normalized by the average trace, i.e.\  $\widetilde{B} = B / (\frac{\text{tr}(B)}{d})$, for either $A_n$ given by FisherMALA 
 or $A_n :=\Sigma_n$ given by AdaMALA and  where  $\widetilde{\Sigma}$ is the optimal normalized preconditioner. The faster 
the Frobenius norm goes to zero the more effective is the corresponding adaptive scheme.  For  
all three Gaussian targets the mean vector $\mu$ was taken to be the vector of ones and samplers  were initialized by drawing from standard 
normal. 
The first example is a two-dimensional Gaussian target with covariance matrix $\Sigma = [1 \ 0.995; 0.995 \ 1]$. 
Both FisherMALA and AdaMALA perform almost the same (FisherMALA has faster convergence) in this low dimensional example as shown by Frobenius norm  
in Figure \ref{fig:2dgaussian_and_gp}a; see also Figure \ref{fig:2dgaussian_appendix} in the Appendix for visualizations of the adapted preconditioners.    
The following two examples involve 100-dimensional targets. 


\paragraph{Gaussian process correlated target.} We consider a Gaussian process to  construct a 100-dimensional Gaussian  
where the covariance matrix is  obtained  by a non-stationary covariance function comprising the product of linear and squared exponential kernels plus small 
 white noise, i.e.\ $[\Sigma]_{i,j}  = s_i s_j \exp\{ - \frac{1}{2} \frac{(s_i - s_j)^2}{0.09} \} + 0.001\delta_{i,j}$ where the scalar inputs $s_i$ 
 form a regular grid in the range $[1,2]$.  Figure  \ref{fig:2dgaussian_and_gp}b shows the evolution of the Frobenius norms and panels d,c  
 depict as $100 \times 100$ images  the true covariance matrix and the preconditioner estimated by FisherMALA. 
 For AdaMALA see Figure \ref{fig:gp_appendix} in the Appendix.  Clearly, FisherMALA learns much faster and achieves more accurate estimates of the optimal preconditioner. 
 Further Table  \ref{table:large_table}  shows that FisherMALA achieves significantly better ESS than AdaMALA and reaches the same 
 performance with mMALA. 



\begin{figure}[t]
\centering
\begin{tabular}{cccc}
\includegraphics[scale=0.18]
{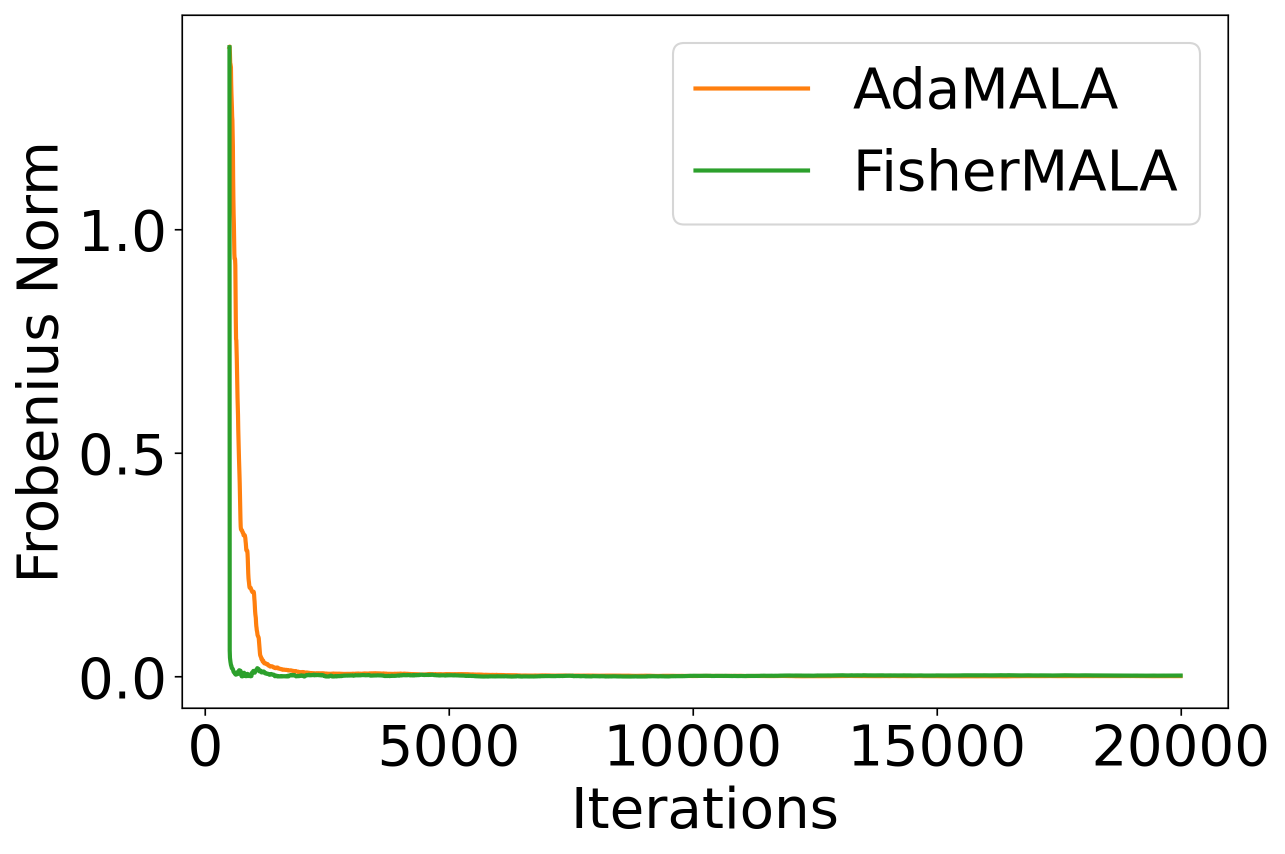} &
\includegraphics[scale=0.18]
{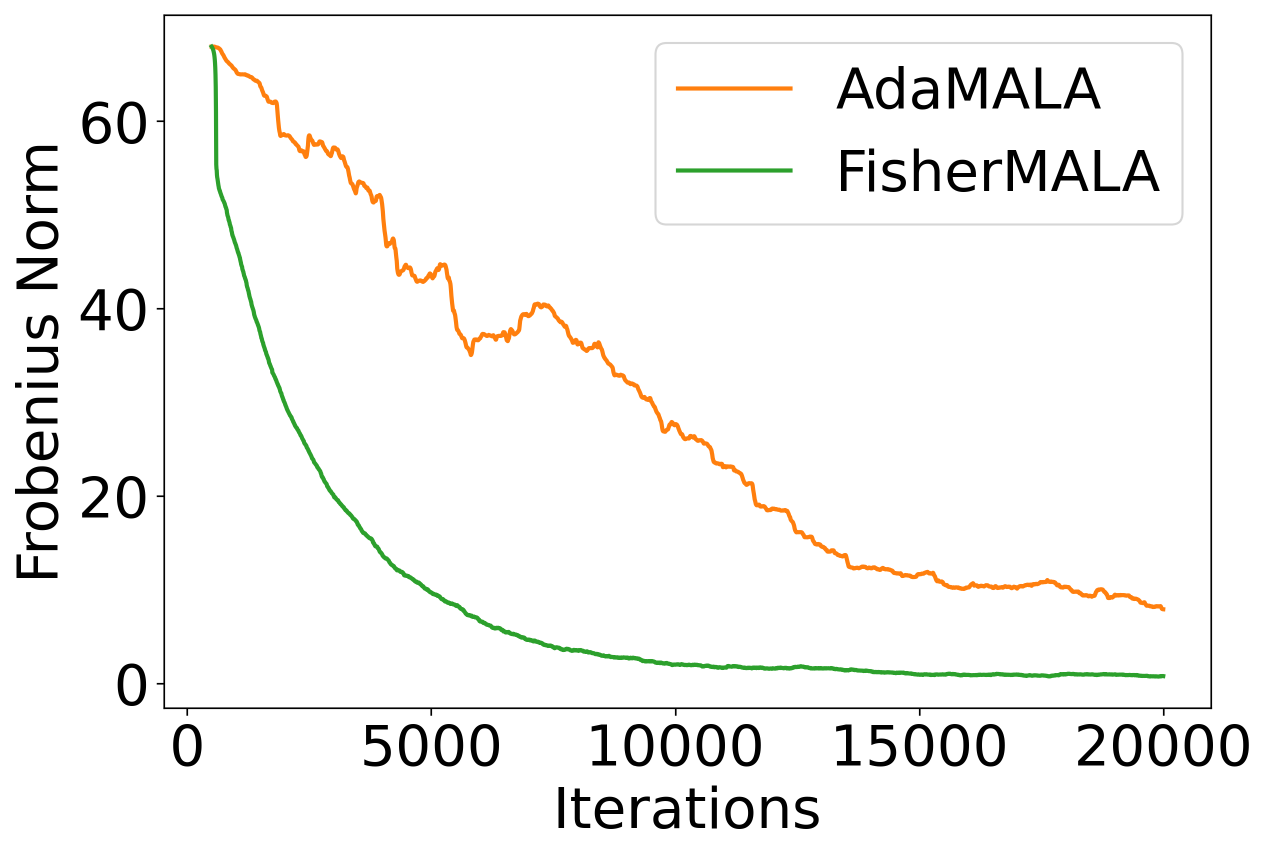}  &
\includegraphics[scale=0.19]
{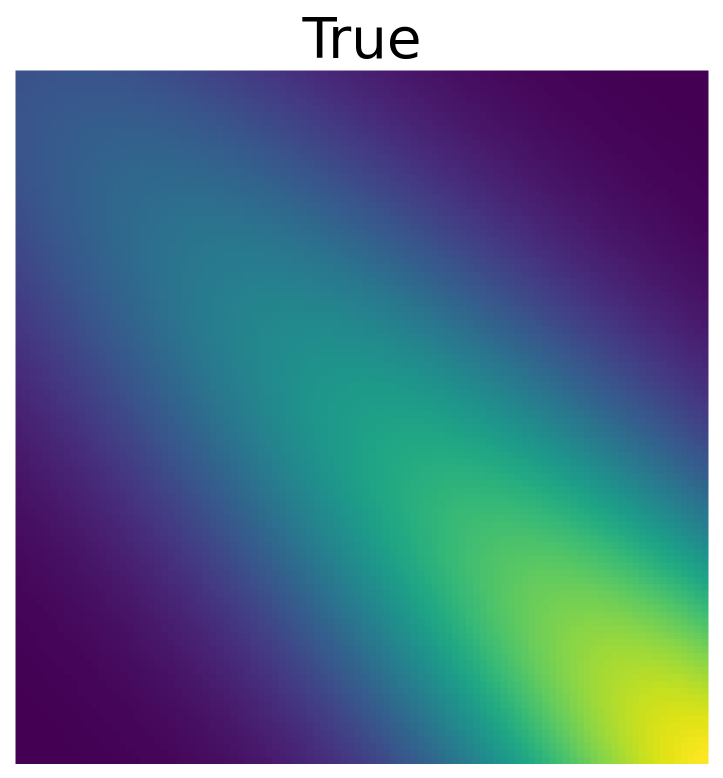}  &
\includegraphics[scale=0.19]
{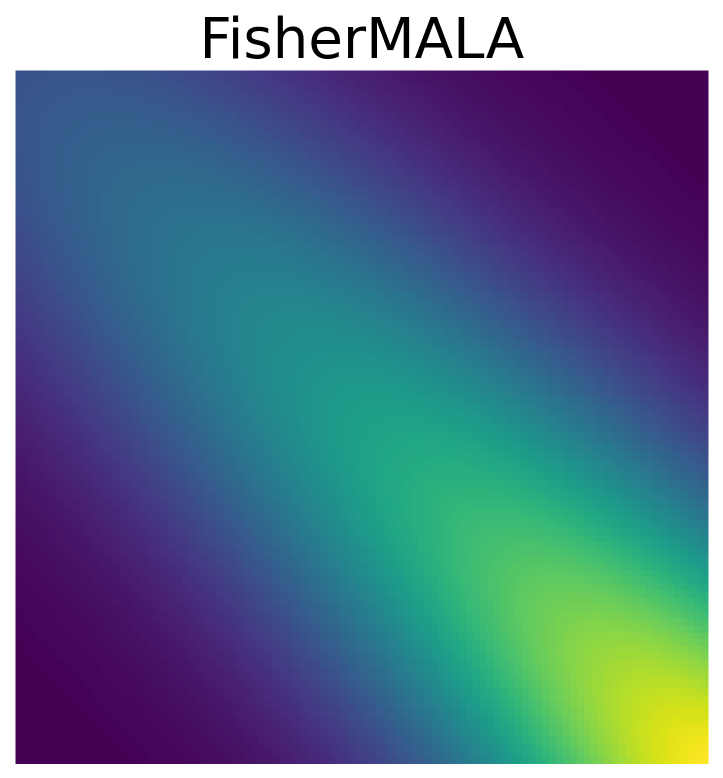} \\
 (a) & (b) & (c) & (d)
\end{tabular}
\caption{Panel (a)  shows the Frobenius norm across burn-in iterations for the 2-D Gaussian 
and (b) for the GP target. The exact GP covariance matrix is shown in (c) and the estimated one
by FisherMALA in (d).} 
\label{fig:2dgaussian_and_gp}
\end{figure}

\begin{figure}[t]
\centering
\begin{tabular}{ccc}
\includegraphics[scale=0.19]
{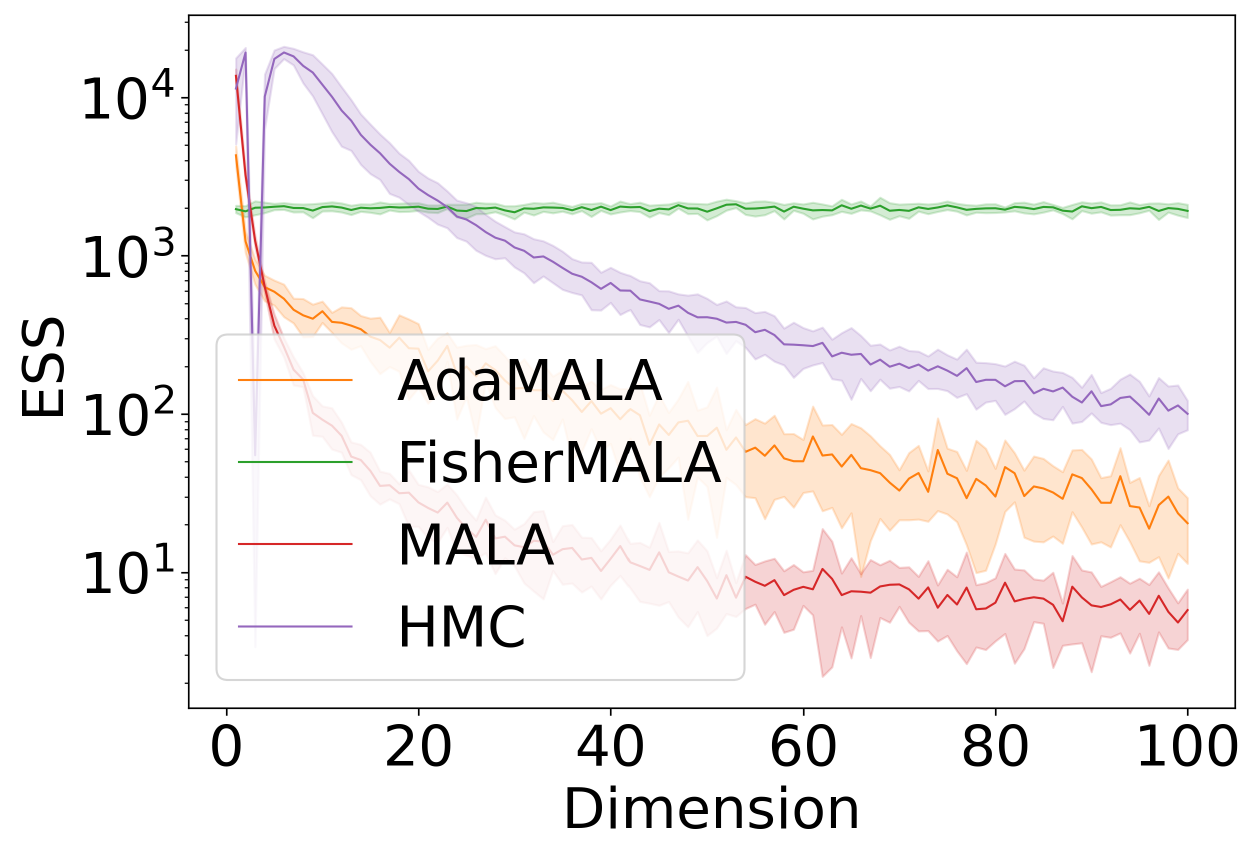} &
\includegraphics[scale=0.19]
{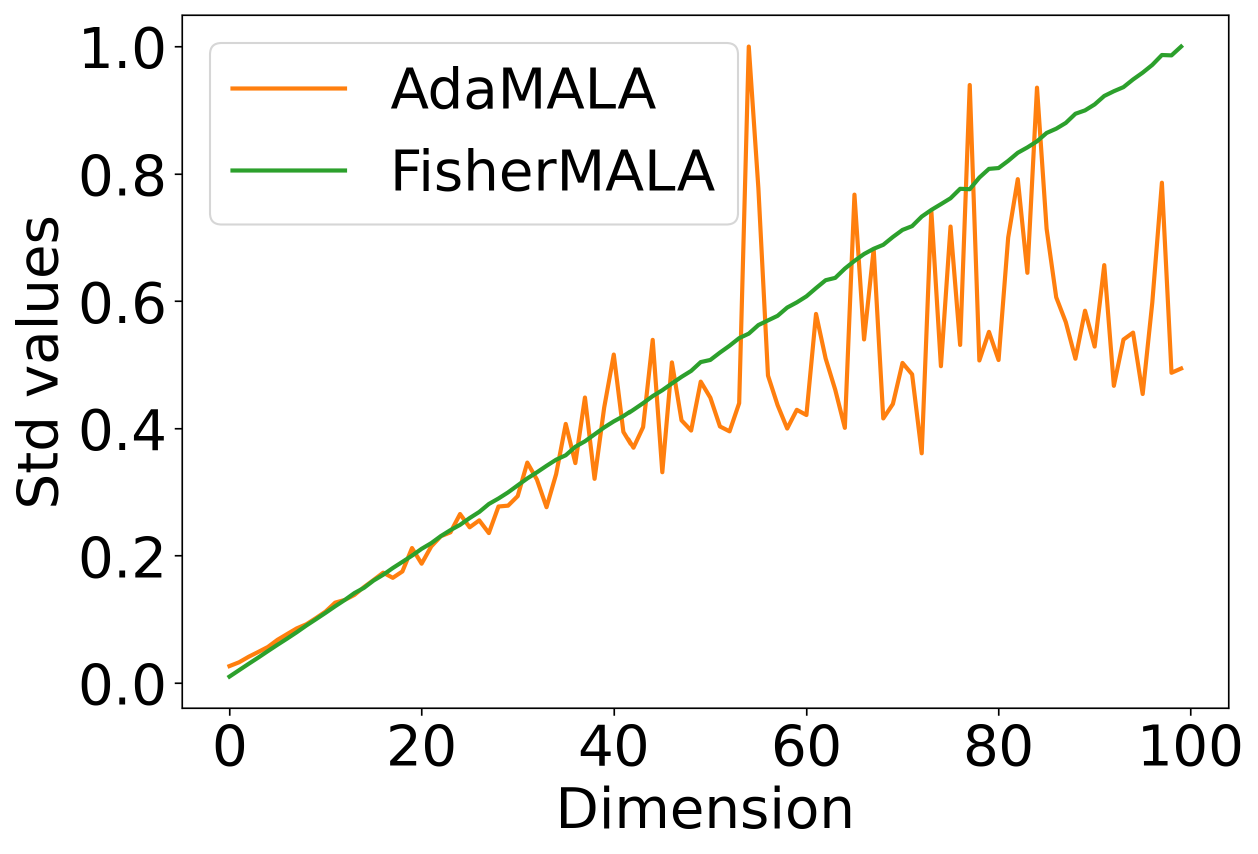} &
\includegraphics[scale=0.19]
{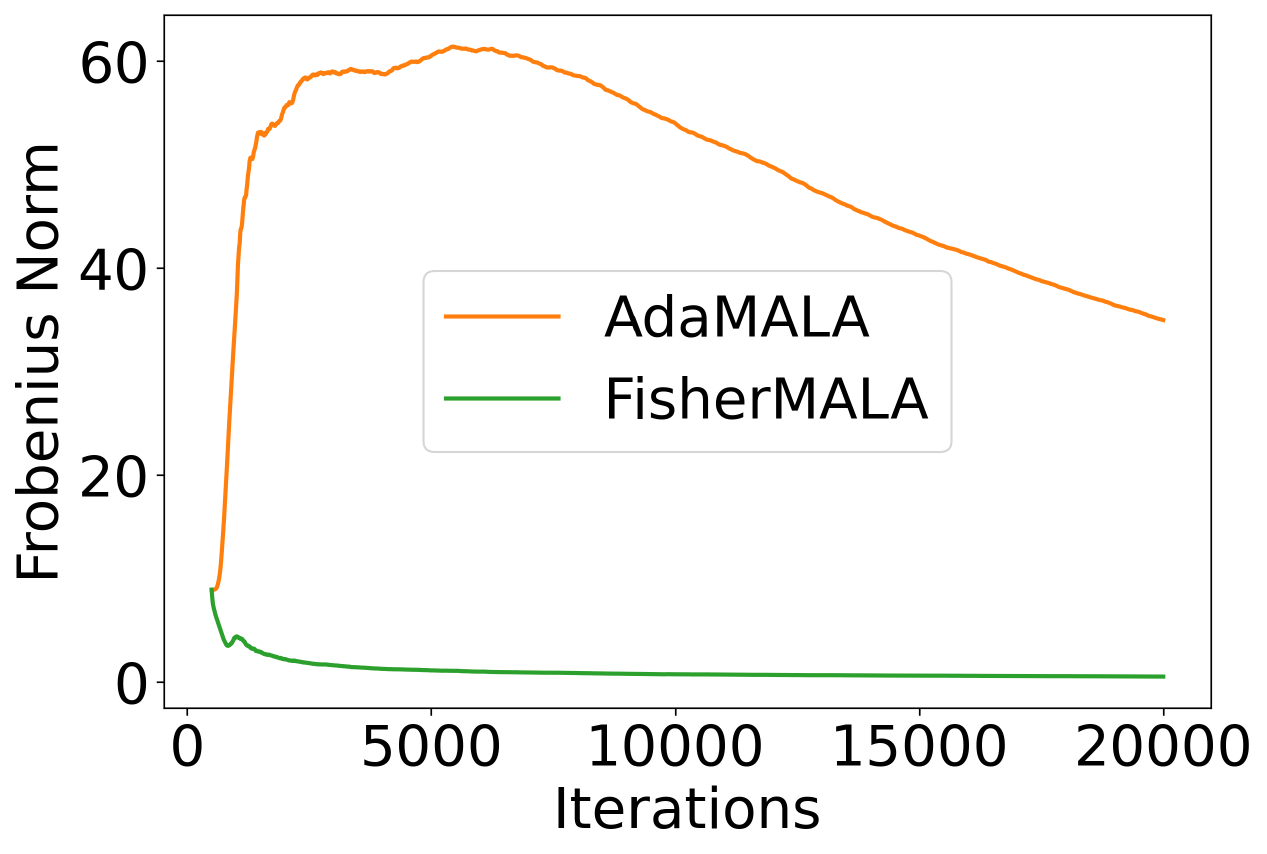} \\
(a) & (b) & (c)
\end{tabular}
\caption{Results in the inhomogeneous Gaussian target.} 
\label{fig:neal}
\end{figure}

\begin{figure}
\centering
\begin{tabular}{cccc}
\includegraphics[scale=0.19]
{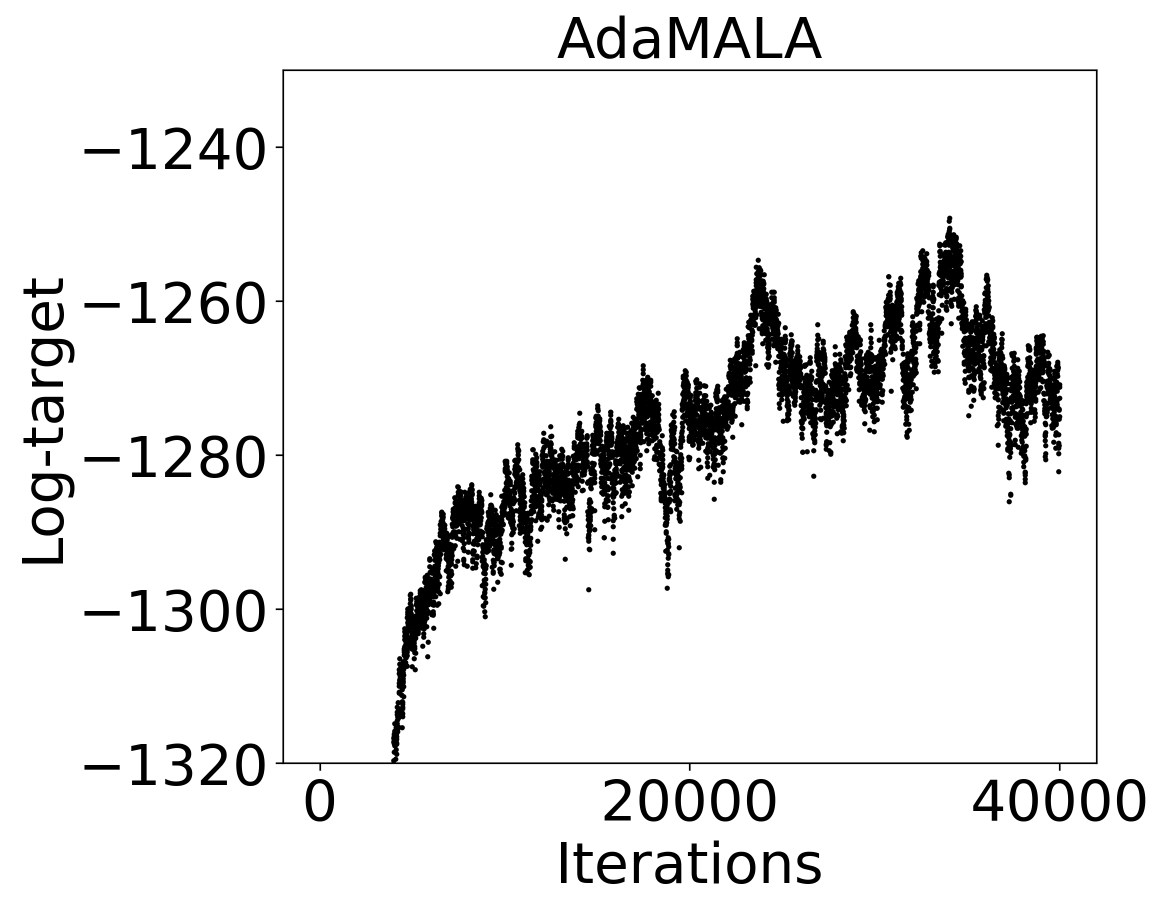} &
\includegraphics[scale=0.19]
{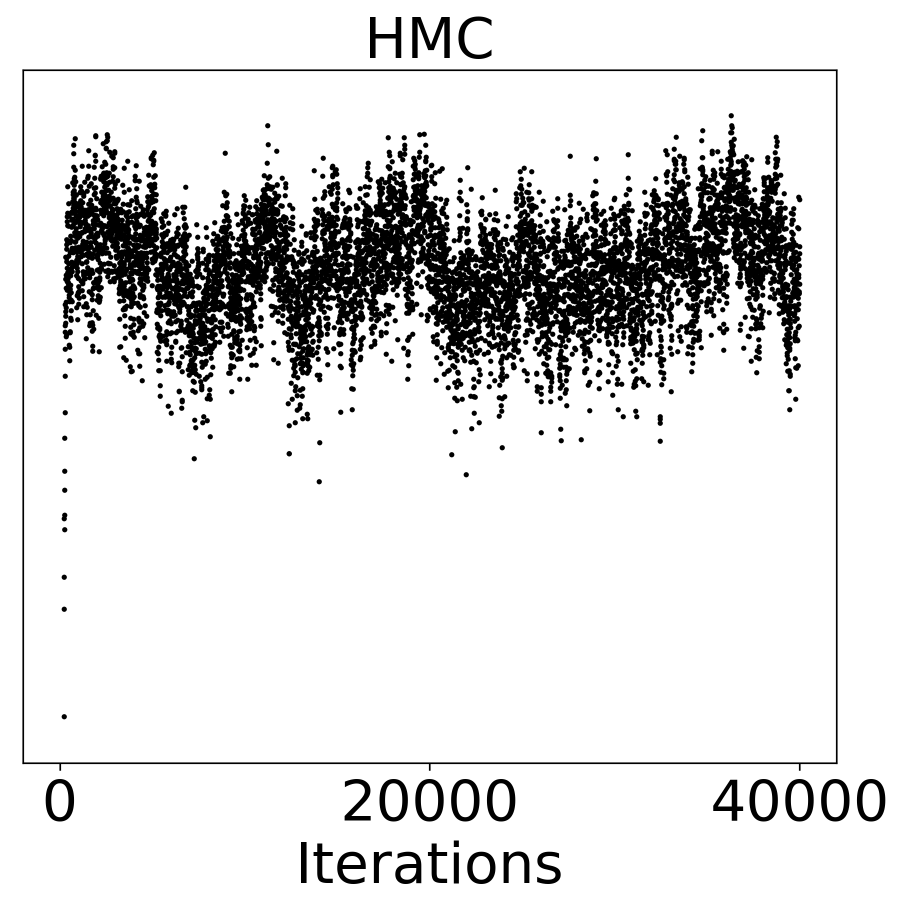} &
\includegraphics[scale=0.19]
{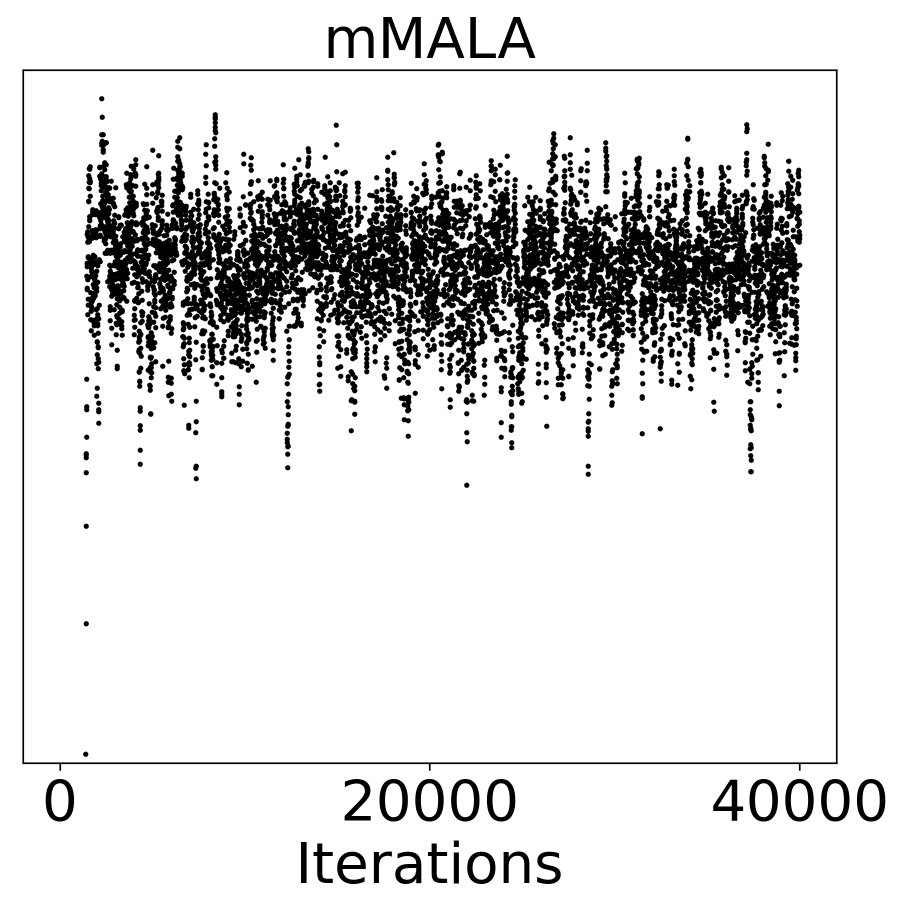} &
\includegraphics[scale=0.19]
{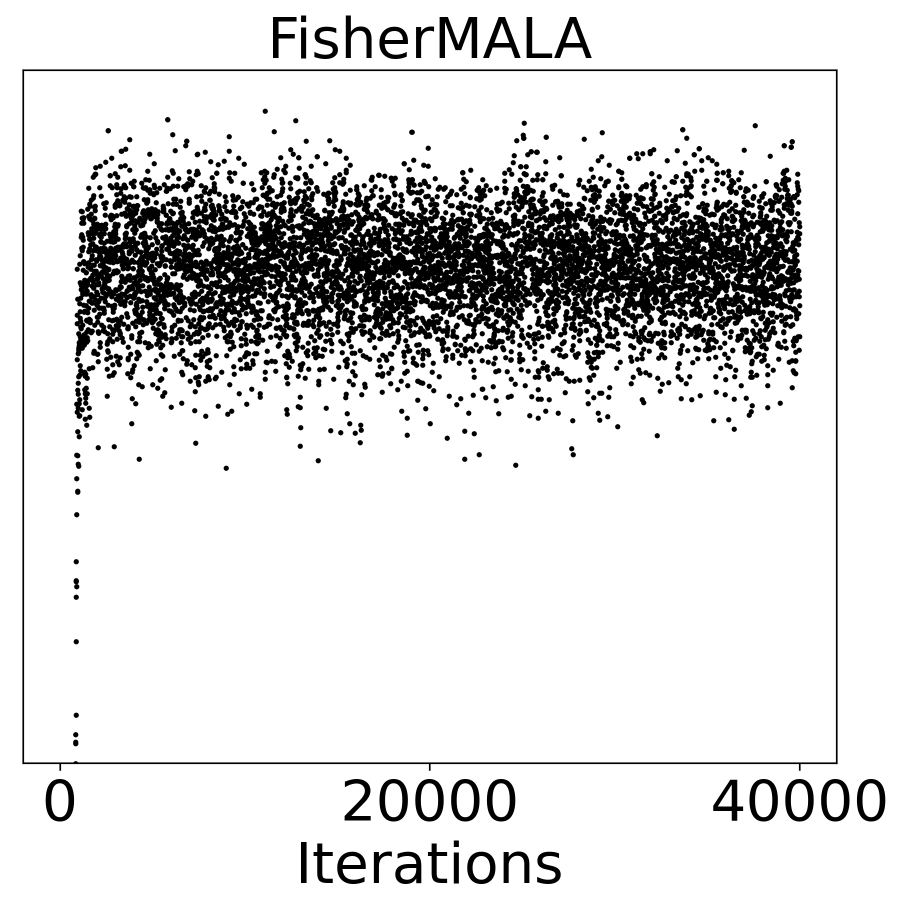}
\end{tabular}
\caption{The evolution of  the log-target across iterations for the best four algorithms in Caravan dataset.} 
\label{fig:caravan}
\end{figure}

\paragraph{Inhomogeneous  Gaussian target.}  In the last example we follow \cite{neal2010, rosenthal2011optimal} and 
we consider a Gaussian target  with diagonal covariance matrix 
$\Sigma = \text{diag}(\sigma^2_1,\ldots, \sigma_{100}^2)$ where the standard deviation values $\sigma_i$ 
take values in the grid $\{0.01,0.02, \dots,1\}$.  This target is challenging because
the different scaling across dimensions means that samplers with a single step size, i.e.\ without preconditioning, will adapt to the smallest dimension $x_1$ of the state 
while the chain at the higher dimensions, such as $x_{100}$, will be moving slowly exhibiting high autocorrelation. 
Note that FisherMALA and AdaMALA run  without knowing that the optimal preconditioner 
is a diagonal matrix, i.e.\  they learn a  full covariance matrix.    
Figure \ref{fig:neal}a shows the ESS scores for all 100 dimensions of $x$  for four samplers (except mMALA which has the same performance with FisherMALA), where  
we can observe that only FisherMALA is able to achieve high ESS uniformly well across all dimensions.  In contrast, 
MALA and HMC that use a single step size cannot achieve high sampling efficiency and their ESS for dimensions close to $x_{100}$ drops significantly. The same holds for
AdaMALA due to its inability to learn fast the preconditioner, as shown by the Frobenius norm values in Figure \ref{fig:neal}c and the estimated standard deviations 
in Figure \ref{fig:neal}b.  AdaMALA can eventually get very close to the optimal precondtioner but it requires hundred of thousands of adaptive steps, 
while FisherMALA learns it with only few thousand steps.

\begin{table}
  \caption{ESS scores are averages after 
repeating the simulations $10$ times under different random initializations.}
  \label{table:large_table}
  \centering
  \begin{small}
  \begin{tabular}{llll}
    \toprule
    & Max ESS & Median ESS & Min ESS \\
\midrule
\emph{GP target} ($d=100$) &  & & \\
MALA & $15.235 \pm 4.246$ & $6.326 \pm 1.934$ & $3.619 \pm 0.956$ \\
AdaMALA & $845.100 \pm 80.472$ & $662.978 \pm 81.127$ & $552.377 \pm 74.441$ \\
HMC & $17.625 \pm 6.706$ & $6.680 \pm 2.205$ & $4.315 \pm 0.889$ \\
mMALA & $2109.441 \pm 101.553$ & $2007.640 \pm 104.867$ & $1841.978 \pm 114.266$ \\
FisherMALA & $2096.259 \pm 94.751$ & $1923.753 \pm 95.820$ & $1784.962 \pm 104.440$ \\
\midrule
\emph{Pima Indian} ($d=7$) & & & \\
MALA & $106.668 \pm 29.601$ & $14.723 \pm 3.821$ & $4.061 \pm 1.587$ \\
AdaMALA & $211.948 \pm 133.363$ & $52.277 \pm 26.566$ & $6.401 \pm 3.344$ \\
HMC & $1624.773 \pm 544.777$ & $337.100 \pm 212.158$ & $6.052 \pm 2.062$ \\
mMALA & $6086.948 \pm 117.241$ & $5690.967 \pm 118.401$ & $5297.835 \pm 160.084$ \\
FisherMALA & $6437.419 \pm 207.548$ & $5981.960 \pm 156.072$ & $5628.541 \pm 168.425$ \\ 
\midrule
\emph{Caravan} ($d=87$)  & & & \\
MALA & $27.247 \pm 7.554$ & $5.890 \pm 0.398$ & $2.906 \pm 0.150$ \\
AdaMALA & $41.522 \pm 9.343$ & $7.144 \pm 0.663$ & $3.144 \pm 0.135$ \\
HMC & $787.901 \pm 173.863$ & $37.303 \pm 6.808$ & $4.238 \pm 0.532$ \\
mMALA & $179.899 \pm 61.502$ & $121.867 \pm 41.801$ & $51.414 \pm 24.995$ \\
FisherMALA & $2257.737 \pm 45.289$ & $1920.903 \pm 55.821$ & ${\bf 498.016} \pm 96.692$ \\
\midrule
\emph{MNIST} ($d=785$) & & & \\ 
MALA & $34.074 \pm 4.977$ & $7.589 \pm 0.149$ & $2.944 \pm 0.066$ \\
AdaMALA & $62.301 \pm 9.203$ & $8.188 \pm 0.399$ & $2.985 \pm 0.089$ \\
HMC & $889.386 \pm 118.050$ & $303.345 \pm 10.976$ & $114.439 \pm 20.965$ \\
mMALA & $51.589 \pm 3.447$ & $20.222 \pm 1.259$ & $5.240 \pm 0.492$ \\
FisherMALA & $1053.455 \pm 35.680$ & $811.522 \pm 19.165$ & ${\bf 439.580} \pm 52.800$ \\
  \bottomrule
  \end{tabular}
 \end{small} 
\end{table}

 
\subsection{Bayesian logistic regression}
 
 
We consider Bayesian logistic regression 
distributions of the form $\pi(\theta|Y,Z)  \propto p(Y|\theta,Z) p(\theta)$ with data $(Y,Z) = \{y_i, z_i\}_{i=1}^m$, where
$z_i \in \Real^d$ is the input vector and $y_i$ the binary label. The likelihood is
$p(Y |\theta, Z) = \prod_{i=1}^m \sigma(\theta^\top z_i)^{y_i} (1 - \sigma(\theta^\top z_i))^{1-y_i}$, where 
 $\theta \in \Real^d$ are the random parameters assigned the prior $p(\theta) = \mathcal{N}(\theta| 0, I_d)$. We consider six binary classification datasets (Australian Credit, Heart, Pima Indian, Ripley, German Credit and Caravan) with a number of data ranging from $n=250$ to $n=5822$ and
 dimensionality of the $\theta$ ranging from $3$ to $87$. We also consider a much higher $785$-dimensional example on MNIST for classifying the digits $5$ and
$6$, that has $11339$ training examples.  To make the inference problems more challenging, in the first six examples we do not standardize the inputs $z_i$ which creates 
very anisotropic posteriors over $\theta$. For the MNIST data, which initially are grey-scale images in [0, 255],  we simply divide by the maximum pixel value, i.e.\ 255, to bring the images in $[0,1]$.    
In Table \ref{table:large_table} we report the ESS for the low $7$-dimensional  Pima Indians dataset,  the medium $87$-dimensional Caravan dataset  and the higher $785$-dimensional 
MNIST dataset, while the results for the remaining datasets are shown in Appendix \ref{app:additionalresults}.    
Further,  Figure \ref{fig:caravan} shows the evolution of the unnormalized log target density $\log \{p(Y|\theta,Z) p(\theta) \}$ for the best four samplers in Caravan dataset which visualizes chain autocorrelation.  
 From all these results we can conclude that FisherMALA is  better than all other samplers, and remarkably it outperforms significantly the position-dependent mMALA, especially 
  in the high dimensional Caravan  and MNIST datasets.

\comm{
\begin{table}
  \caption{ESS Gaussian dataset.}
  \label{table:gaussian}
  \centering
  \begin{tabular}{llll}
    \toprule
\input{figures/gaussian.txt}
  \bottomrule
  \end{tabular}
\end{table}

\begin{table}
  \caption{Caravan dataset. All results  are averages after 
repeating the simulations $10$ times under different random initialisations. } 
  \label{table:caravan}
  \centering
  \begin{tabular}{llll}
    \toprule
\input{figures/caravan.txt}
  \bottomrule
  \end{tabular}
\end{table}

\begin{table}
  \caption{Pima dataset.}
  \label{table:pima}
  \centering
  \begin{tabular}{llll}
    \toprule
\input{figures/pima.txt}
  \bottomrule
  \end{tabular}
\end{table}

\begin{table}
  \caption{Mnist dataset.}
  \label{table:mnist}
  \centering
  \begin{tabular}{llll}
    \toprule
\input{figures/mnist.txt}
  \bottomrule
  \end{tabular}
\end{table}
}


%% file: sec_related.tex

 
\section{Related Work}
\label{sec:related}



There exist works that use some form of global preconditioning in gradient-based 
samplers for specialized targets such as latent Gaussian models \citep{Cotter2013mcmc, titsiasPapaspiliopoulos16}, 
which make use of the tractable Gaussian prior. 
Our method differs, since it is more agnostic to the target 
and learns a preconditioning from the history of gradients, analogously to
how traditional adaptive MCMC 
 learns from 
  states \citep{haario2001adaptive, roberts2009examples}.

Several research works use position-dependent preconditioning $A(x)$ within gradient-based samplers, 
such as MALA.  This is for example the idea behind Riemannian manifold MALA \citep{GirolamiCalderhead11} and extensions  \citep{Xifara14}.  
Similar to Riemannian manifold methods there are
approaches inspired by second order 
optimization that use the Hessian matrix, or some estimate of the Hessian, for sampling in  a MALA-style manner  
\citep{Qi_hessian-basedmarkov, gewekeTanizaki, stochasticNewtonmcmc}. 
Recently, such samplers and their time-continuous diffusion limits have been theoretically analyzed by obtaining convergence  guarantees
\citep{Chewietal2020, Li_mirrorLangevin2022}. 
All such methods form a position-dependent preconditioning 
 and not  the preconditioning we use in this paper, e.g.\ note that  $\mathcal{I}^{-1}$ we consider here requires an expectation under the target 
and thus it is always a global preconditioner rather than a position-dependent one.  Another difference is that our method has quadratic cost, 
while position-dependent preconditioning methods have cubic cost and they require computationally demanding quantities like the Hessian matrix. 
Therefore,  in order for these methods to run faster some approximation may be needed, e.g.\ low rank \citep{stochasticNewtonmcmc}  or quasi-Newton type
\citep{ZhangS11, ensemblepreconditioning}.  Furthermore,  the Bayesian logistic results in Table \ref{table:large_table} (see also Figure \ref{fig:caravan}) show
that the proposed FisherMALA method significantly outperforms manifold MALA \citep{GirolamiCalderhead11}  
in Caravan and MNIST examples, despite the fact that manifold MALA preconditions with the exact negative inverse Hessian matrix of the log target.  This could suggest that 
position-dependent preconditioning may be less effective in certain type of high-dimensional and log-concave problems.

Finally, there is recent work for learning flexible MCMC proposals by using neural  
networks \citep{song2017nice, levy2018generalizing, habib2018auxiliary, Salimans2015} 
and by adapting parameters using differentiable objectives  \citep{levy2018generalizing, neklyudov2018metropolis, TitsiasDellaportas19, Dharamshietal23}. 
Our method differs, since it does not use objective functions (which have extra cost because they require an optimization to run
in parallel with the MCMC chain), but  instead it adapts similarly to traditionally MCMC methods by accumulating information from the
observed history of the chain.

%% file: sec_conclusion.tex
\section{Conclusion}
\label{sec:conclusion}


We derived an optimal preconditioning for the Langevin diffusion by optimizing the expected squared jumped distance, 
and subsequently we developed an adaptive MCMC algorithm 
 that approximates the optimal preconditioning 
by applying an efficient quadratic cost recursion. Some possible topics for future research are: Firstly, it would be useful to 
investigate whether the 
score function differences that we use as the adaptation signal 
 introduce any bias in the estimation of the inverse Fisher matrix. Secondly, it would be interesting to extend 
 our method to learn the preconditioning for other gradient-based samplers such as Hamiltonian Monte Carlo (HMC),  
 where such a matrix there is referred to as the inverse mass matrix.  For HMC this is more complex since both the mass matrix and its inverse 
 are needed in the iteration. 
 Finally, it could be interesting to investigate adaptive schemes of the inverse Fisher matrix by using multiple parallel and interacting chains, 
 similarly to ensemble covariance matrix estimation for Langevin diffusions \cite{Garbuno2020}.


%% file: sec_appendix.tex
\appendix 

\section{Proofs
\label{app:proofs}
}

\subsection{Proof of Proposition \ref{prop:accratio}
\label{app:propone}
} 

The difference between the logarithm of the backward and forward proposals of preconditioned MALA, i.e. the quantity $\log q(x_n | y_n) - \log q(y_n | x_n )$ 
can be written (ignoring the normalizing constants of the Gaussians which trivially cancel out) as, 
 \begin{align}
& - \frac{1}{2 \sigma^2} \left( 
x_n  - y_n - \frac{\sigma^2}{2} A \nabla \log \pi(y_n)  
 \right)^\top  A^{-1}  \left( 
x_n  - y_n - \frac{\sigma^2}{2} A \nabla \log \pi(y_n)  
 \right) \nonumber \\
 & + \frac{1}{2 \sigma^2} \left( 
y_n  - x_n - \frac{\sigma^2}{2} A \nabla \log \pi(x_n)  
 \right)^\top  A^{-1}  \left( 
y_n  - x_n - \frac{\sigma^2}{2} A \nabla \log \pi(x_n)  
 \right). 
 \end{align} 
Observe that the term $\frac{1}{2 \sigma^2} (x_n - y_n)^\top A^{-1} (x_n - y_n)$  cancels out since it  appears twice 
with opposite sign. The remaining terms after some simple algebra simplify as 
 \begin{align}
& \textcolor{blue}{\frac{1}{2} \! \left( 
x_n  - y_n - \frac{\sigma^2}{4} A \nabla \log \pi(y_n)  
 \right)^\top \!  \! \nabla \log \pi(y_n)}  - \textcolor{red}{\frac{1}{2} \! \left( 
y_n  - x_n - \frac{\sigma^2}{4} A \nabla \log \pi(x_n)  
 \right)^\top \! \! \nabla \log \pi(x_n)}  \nonumber \\ 
 & =  \textcolor{blue}{h(x_n, y_n)} - \textcolor{red}{h(y_n, x_n)}
 \end{align} 
 which completes the proof. 
 
\subsection{
Proof of Proposition
\ref{prop:langevincovar}
\label{app:proptwo}
}

We assume $x_t \sim \pi(x_t)$. Then by taking the expectation of the r.h.s.\ of Eq.\ \eqref{eq:discretizesLangevin}  (where the expectation is taken w.r.t.\ $x_t$ and the 
independent Brownian motion increment $B_{t+\delta} - B_t \sim \mathcal{N}(0, \delta I_d)$) and noting that  $\Exp_{\pi(x_t)} [\nabla \log \pi(x_t)] = 0$ and $\Exp[B_{t+\delta} - B_t] =0$
we conclude that $\Exp[x_{t+\delta} - x_t] = 0$.  Then the covariance is 
\begin{align}
& \Exp[(x_{t + \delta} - x_t) (x_{t + \delta} - x_t )^\top] = \nonumber \\
& =  \Exp\left[ (\frac{\delta}{2} A \nabla \log \pi(x_t)  + \sqrt{A} (B_{t+\delta} - B_t)) (\frac{\delta}{2} A \nabla \log \pi(x_t) + \sqrt{A}  (B_{t+\delta} - B_t))^\top \right]  \nonumber \\
& =   \frac{\delta^2}{4} A \Exp_{\pi(x_t)} [\nabla \log \pi(x_t) \nabla \log \pi(x_t)^\top] A   + \delta A  \nonumber \\
& =   \frac{\delta^2}{4} A \mathcal{I}A  + \delta A,  \nonumber 
\end{align} 
where we used that $\Exp[ (B_{t+\delta} - B_t) (B_{t+\delta} - B_t)^\top] = \delta I_d$ , $\sqrt{A} \sqrt{A}^\top = A$ and that the cross covariance terms are zero.  

\subsection{Proof of Proposition \ref{prop:optimalprecond}
\label{app:propthree}
}

The expected squared jumped distance is written as 
$$
J(\delta, A) = \text{tr} \left(\frac{\delta^2}{4} A \mathcal{I} A +  \delta A \right)
=\frac{\delta^2}{4} 
\text{tr}(A \mathcal{I} A) + \delta c, 
$$
where we used the constraint 
$\text{tr}(A) = c$. Since $c$ is just a constant to minimize $J(\delta, A)$ is the same as minimizing $\text{tr}(A \mathcal{I} A)$, a quadratic convex loss since 
$\mathcal{I}$ is positive definite, under the constraint that $A$ is symmetric positive definite matrix and $\text{tr}(A)=c$.  To deal with the equality constraint  we consider the Lagrangian 
$$
\text{tr}(A \mathcal{I} A) - \lambda ( \text{tr}(A) - c). 
$$ 
By taking derivatives wrt the matrix $A$ (using the matrix derivative identities  $\frac{\partial}{\partial X} \text{tr}(X B X) = X^\top B^\top + B^\top X^\top$ and 
 $\frac{\partial}{\partial X} \text{tr}(X) = I_d$ for arbitrary $d \times d$ square matrices $X, B$) and setting to zero we see that $A$ must satisfy the linear equation
$$
A^\top \mathcal{I} + \mathcal{I} A^\top =  \lambda I_d,
$$
where we used that $\mathcal{I}$ is a symmetric matrix.  This is a set of linear equations
and given that each eigenvalue $\mu_i$ of $\mathcal{I}$ 
satisfies $0 < \mu_i < \infty$,  so that $\mathcal{I}$ is invertible,  there is an unique solution given by $A = (1/2) \lambda \mathcal{I}^{-1}$.  
The Lagrange multiplier $\lambda$  is chosen so that $\text{tr}(A)=c$ which leads to the optimal $A^*$ 
$$
A^* = \frac{c}{\sum_{i=1}^d \frac{1}{\mu_i}} \mathcal{I}^{-1}.  
$$
Note that $A^*$ turned out to be symmetric and  positive definite as  desired.  
For this $A^*$ the optimal loss value is $\text{tr}(A^* \mathcal{I} A^*) = \frac{c^2}{\sum_{i=1}  \frac{1}{\mu_i}}$, for which we 
further need to disambiguate whether this is the global minimum or maximum. We can do this by choosing a different matrix that satisfies the constraint $\text{tr}(
A)=c$ and  
compare its loss with the optimal loss  $\frac{c^2}{\sum_{i=1}^d  \frac{1}{\mu_i}}$ . For example, one such matrix is $A = \frac{c}{d} I_d$, which has loss value 
$\frac{c^2  (\sum_{i=1}^d \mu_i)}{d^2}$. Then by using the Cauchy-Schwarz inequality $d^2 = (\sum_{i=1}^d \frac{\sqrt{\mu_i}}{\sqrt{\mu_i}} )^2 \leq (\sum_{i=1}^d \mu_i ) (\sum_{i=1}^d \frac{1}{\mu}_i)$ 
we obtain $\frac{c^2  (\sum_{i=1}^d \mu_i)}{d^2} \geq  \frac{c^2  (\sum_{i=1}^d \mu_i)}{  (\sum_{i=1}^d \mu_i ) (\sum_{i=1}^d \frac{1}{\mu}_i)} = \frac{c^2}{\sum_{i=1}  \frac{1}{\mu_i}}$.
This shows that $A^*$ achieves the global minimum which completes the proof.

\subsection{Proof of Proposition \ref{prop:recursion}
\label{app:propfour}
}

We first state  and prove the following intermediate result.  
 \begin{lem} 
Suppose the positive definite matrix $I_d - z z^\top$ where $z \in \Real^d$ and $ z^\top z \leq 1$. Then, a square root matrix $R$, satisfying $R R^\top = A$, has the form $R = I_d - r z z^\top$ where
$r = \frac{1}{1 + \sqrt{1 - z^\top z}}$.
\label{lemmasqrt}
\end{lem}
\begin{proof}
We hypothesize that $R$ has the form $I_d - r z z^\top$ for some scalar $r$. Then since  $R R^\top = I_d - z z^\top$ we see that $r$ must satisfy the quadratic equation 
$r^2 z^\top z - 2 r +1=0$, which has two real solutions $\frac{1 \pm \sqrt{1 - z^\top z} }{z^\top z}$  and we will use  $\frac{1 - \sqrt{1 - z^\top z} }{z^\top z} \leq 1$ which ensures $R$ is positive definite. 
This solution can also be written as $r = \frac{1}{1 + \sqrt{1 - z^\top z}}$.
\end{proof} 

To prove  the proposition we need  to find a square root matrix $R_1$ of 
$A_1 =  \frac{1}{\lambda}
\left(I_d - \frac{s_1 s_1^\top}{\lambda + s_1^\top s_1} \right)$  where we clearly need to specify a square root matrix for  
$I_d - \frac{s_1 s_1^\top}{\lambda + s_1^\top s_1}$. We observe that by setting $z = \frac{s_1}{\sqrt{\lambda + s_1^\top s_1}}$ Lemma 1 is applicable 
so that the square root matrix is 
$$ 
R_1 = \frac{1}{\sqrt{\lambda}}
\left(I_d - r_1 \frac{s_1 s_1^\top}{\lambda + s_1^\top s_1} \right),  \ \ r_1 = \frac{1}{1 + \sqrt{\frac{\lambda}{\lambda + s_1^\top s_1}}}. 
$$
Similarly by applying again Lemma 1 we can find $R_n$ for any $n>1$.  

The computation of $R_n$ costs $O(d^2)$ per iteration.  Firstly, the vector $\phi_n = R_{n-1}^\top s_n$ is computed which is a matrix-vector multiplication. The next step is to compute the scalar $r_n$ in $O(d)$ (involving the dot product $\phi_n^\top \phi_n$) and then the scaled vector $\phi_n' = \frac{r_n}{1 + \phi_n^\top \phi_n} \phi_n$ also an $O(d)$ operation. Then we need two additional $O(d^2)$ multiplication operations to obtain firstly the vector $t_n = R_{n-1} \phi_n$ and secondly the outer vector product $t_n (\phi_n')^\top$. Finally, the update is $R_n = R_{n-1} - t_n (\phi_n')^\top$ which requires a final $O(d^2)$ addition operation of two matrices which is typically cheaper than $O(d^2)$ multiplication. Therefore,
overall the cost is $O(d^2)$.

\section{Generalizing the recursion over arbitrary learning rate sequences
\label{app:generallearningrates}
}

Suppose we have a sequence of learning rates  $\gamma_1,\gamma_2,\ldots,$. 
Then a stochastic approximation of
the Fisher matrix $\fisher$ takes the form 
$$
\fisher_n = 
\fisher_{n-1} + \gamma_n (s_n s_n^\top - \fisher_{n-1})
= (1 - \gamma_n) \fisher_{n-1} 
+ \gamma_n s_n s_n^\top, 
$$
where the sequence is initialized at $\fisher_1 = s_1 s_1^\top + \lambda I$. The inverse of the empirical Fisher is written  as
$$
\invfisher_n = \left( 
(1 - \gamma_n) \fisher_{n-1} 
+ \gamma_n s_n s_n^\top 
\right)^{-1} 
= \frac{1}{1 - \gamma_n}
\left( A_{n-1} 
- 
\frac{A_{n-1} s_n s_n^\top A_{n-1} }
{\frac{1-\gamma_n}{\gamma_n} +  s_n^\top \fisher_{n-1}^{-1} s_n}
\right),
$$
which is initialized at  
$A_1 = \frac{1}{\lambda} \left(I_d  - \frac{s_1 s_1^\top}{\lambda + s_1^\top s_1}
\right)$ for which the square root $R_1$ is the same as for the standard learning rate $\gamma_n  =1 / n$. The square root recursion for $n > 1$ takes the form 
$$
\sqrtmat_n 
= \frac{1}{\sqrt{1 - \gamma_n}} 
\left(R_{n-1} 
- 
r_n \frac{(R_{n-1} \phi_n) \phi_n^\top}
{(1-\gamma_n) / \gamma_n +  \phi_n^\top \phi_n} 
\right), \ \phi_n = \sqrtmat_n^\top s_n, \  r_n = \frac{1}{1 + \sqrt{\frac{(1-\gamma_n)/\gamma_n}{(1-\gamma_n)/\gamma_n + \phi_n^\top \phi_n}}}.
$$

\section{FisherMALA with paired mean and covariance stochastic approximation
\label{app:FisherMALA2}
}
 
Here, we derive a recursion for the empirical 
Fisher that centers the score function vectors using 
the standard procedure by recursively estimating also the mean.   
We start from the following consistent estimator of the 
inverse Fisher: 
 $$
A_n = \left( 
\frac{1}{n-1} 
\sum_{i=1}^n (s_i - \bar{s}_n)  (s_i - \bar{s}_n)^\top
+ \frac{\lambda}{n-1} I_d
\right)^{-1}, 
$$
where $\bar{s}_n = \frac{1}{n} \sum_{i=1}^n s_i$.  This follows 
the recursion 
\begin{align}
A_n & = \left( \frac{n-2}{n-1} A_{n-1}^{-1 } + \frac{1}{n} 
\delta_n \delta_n^\top \right)^{-1} = \frac{n-1}{n-2} A_{n-1} - \frac{(n-1)^2}{(n-2)^2} 
\frac{A_{n-1} \delta_n \delta_n^\top A_{n-1}}{n + \frac{n-1}{n-2}\delta_n^\top A_{n-1} \delta_n} \nonumber \\
& = \frac{1}{\lambda_{n-1}}
\left(A_{n-1} -  
\frac{A_{n-1} \delta_n \delta_n^\top A_{n-1}}{n \lambda_{n-1} + \delta_n^\top A_{n-1} \delta_n} \right). \nonumber 
\end{align}
Here, $\delta_n = s_n - \bar{s}_{n-1} $ and we defined the sequence of scalars $\lambda_{n} = \frac{n-1}{n}$, for $n \geq 2$ while the starting point of this sequence $n=1$ we define it to be equal to the parameter parameter $\lambda$, i.e.\ $\lambda_1 = \lambda > 0$. 
The recursion starts at 
$A_2$ given by 
$$
A_2 = \left( 
\frac{1}{2} \delta_2 \delta_2^\top + \lambda_1 I 
\right)^{-1}
= \frac{1}{\lambda_1}
\left(I_d - \frac{\delta_2 \delta_2^\top}{2 \lambda_1 + \delta_2^\top \delta_2} \right),
$$
where $\delta_2 = s_2 - s_1$.  Along with the above we recursively estimate also 
 the mean vector (for $n \geq 1$):  $
\bar{s}_n = \frac{n-1}{n}
\bar{s}_{n-1} + \frac{1}{n} s_n$. 


To express a recursion of  square root matrix, such that $A_n = R_n R_n^\top$ we first write 
\begin{align}
A_n & = \frac{1}{\lambda_{n-1}} R_{n-1}
\left(I_d -  
\frac{R_{n-1}^\top \delta_n \delta_n^\top R_{n-1}}{n \lambda_{n-1} + \delta_n^\top A_{n-1} \delta_n} \right) R_{n-1}^\top\nonumber \\
& = \frac{1}{\lambda_{n-1}} R_{n-1}
\left(I_d -  
\frac{\phi_n \phi_n^\top}{n \lambda_{n-1} + \phi_n^\top  \phi_n} \right) R_{n-1}^\top. \nonumber 
\end{align}
Then we can recognize the square 
root recursion as 
$$
R_n = \frac{1}{\sqrt{\lambda_{n-1}}}
R_{n-1} \left(I_d -  
r_n \frac{\phi_n \phi_n^\top}{n \lambda_{n-1} + \phi_n^\top  \phi_n} \right), \ \ r_n = \frac{1}{1 + \sqrt{\frac{n \lambda_{n-1}}{n\lambda_{n-1} + \phi_n^\top \phi_n}}},
$$
which is initialized at 
$
R_2 = \frac{1}{\sqrt{\lambda_1}}
\left(I_d -  
r_2 \frac{\delta_2 \delta_2^\top}{2 \lambda_1 + \delta_2^\top  \delta_2} \right), \ \ r_2 = \frac{1}{1 + \sqrt{\frac{2 \lambda_1}{2\lambda_1 + \delta_2^\top \delta_2}}}.
$

\section{Initialization of AdaMALA
\label{sec:initializeAdaMALA}
} 
 
 To initialize AdaMALA we first perform $n_0=500$ iterations with simple MALA where we adapt the step size parameter $\sigma^2$. Thus, this part of the initialization 
 is exactly the same used by FisherMALA. However, for AdaMALA we do an additional set of $n_0=500$ iterations where simple MALA still runs 
 and collects samples which are used to sequentially update the empirical covariance matrix $\Sigma_n$. The purpose of this second phase is to play the role of  "warm-up" and provide a reasonable 
 initialization for $\Sigma_n$.  After the second phase (so in total $1000$ iterations)  
 AdaMALA starts running having as a preconditioner $\Sigma_n$, which keeps adapted in every iteration until the last burn-in iteration.

\section{Additional results
\label{app:additionalresults}}

\subsection{The step size $\sigma^2$ is maximized when preconditioning becomes effective 
\label{app:step_size}
}

 To experimentally backup 
our claims in Section \ref{sec:optimalA} that the discretization step size, denoted there by $\delta$ or $\sigma^2$, 
gets large when the preconditioner is selected efficiently,  in Figure  \ref{fig:step_size} we report the final learned  values (after  burn-in adaptation iterations) of $\sigma^2$  for MALA, AdaMALA and FisherMALA. For all these three algorithms the values of $\sigma^2$ are comparable because  all use an overall preconditioning 
of the form $\frac{\sigma^2}{\frac{1}{d}\text{tr}(A)} A$  and only the matrix $A$  is changing among them. For example, simple MALA sets this matrix to 
$A = I_d$, while AdaMALA and FisherMALA use their own procedures to learn more complex matrices. Figure \ref{fig:step_size}  shows the estimated $\sigma^2$, for 
the four datasets reported in the main text in Table  \ref{table:large_table}. This shows that FisherMALA achieves significantly larger $\sigma^2$ 
in all cases, which can be orders of magnitude larger than the two other algorithms (note the $y$ axis in Figure  \ref{fig:step_size} is in log scale).

 \begin{figure*}[!htb]
\centering
\begin{tabular}{cc}
GP target & Pima Indian  \\
\includegraphics[scale=0.3]
{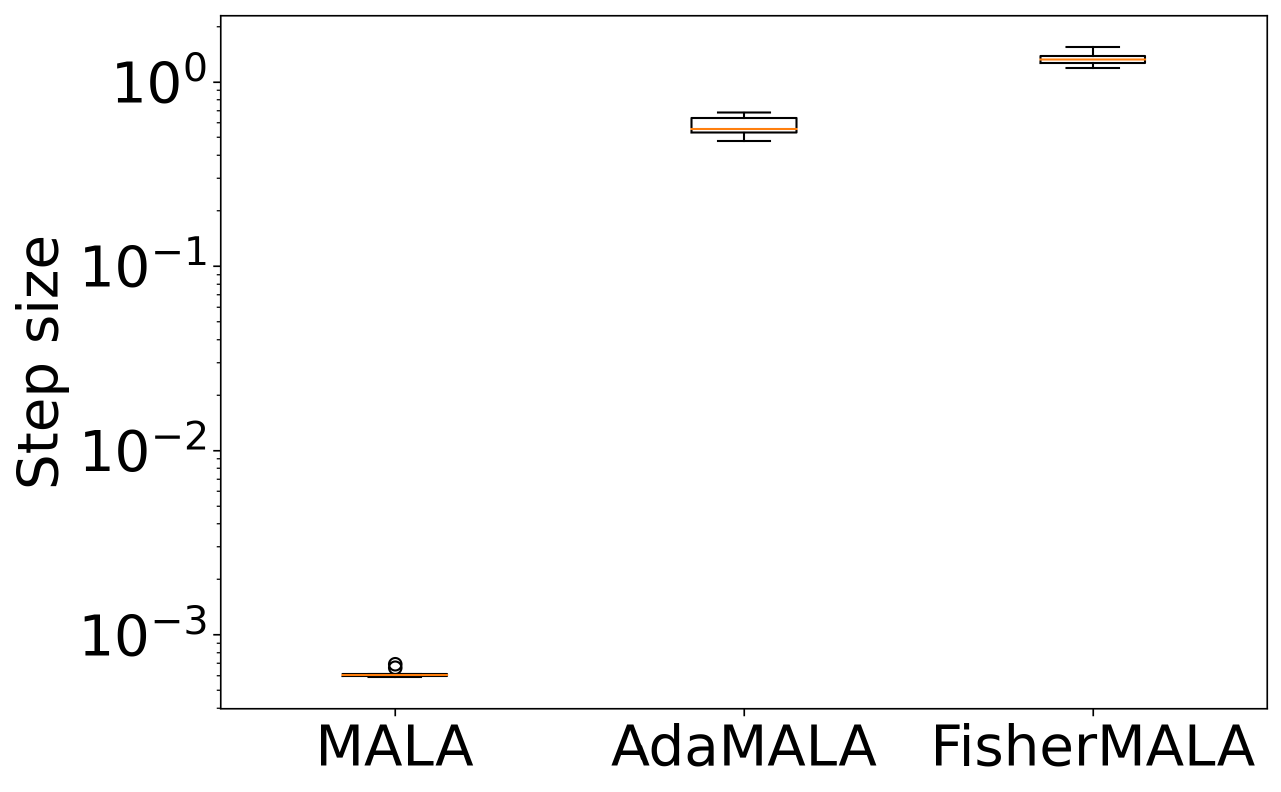} &
\includegraphics[scale=0.3] 
{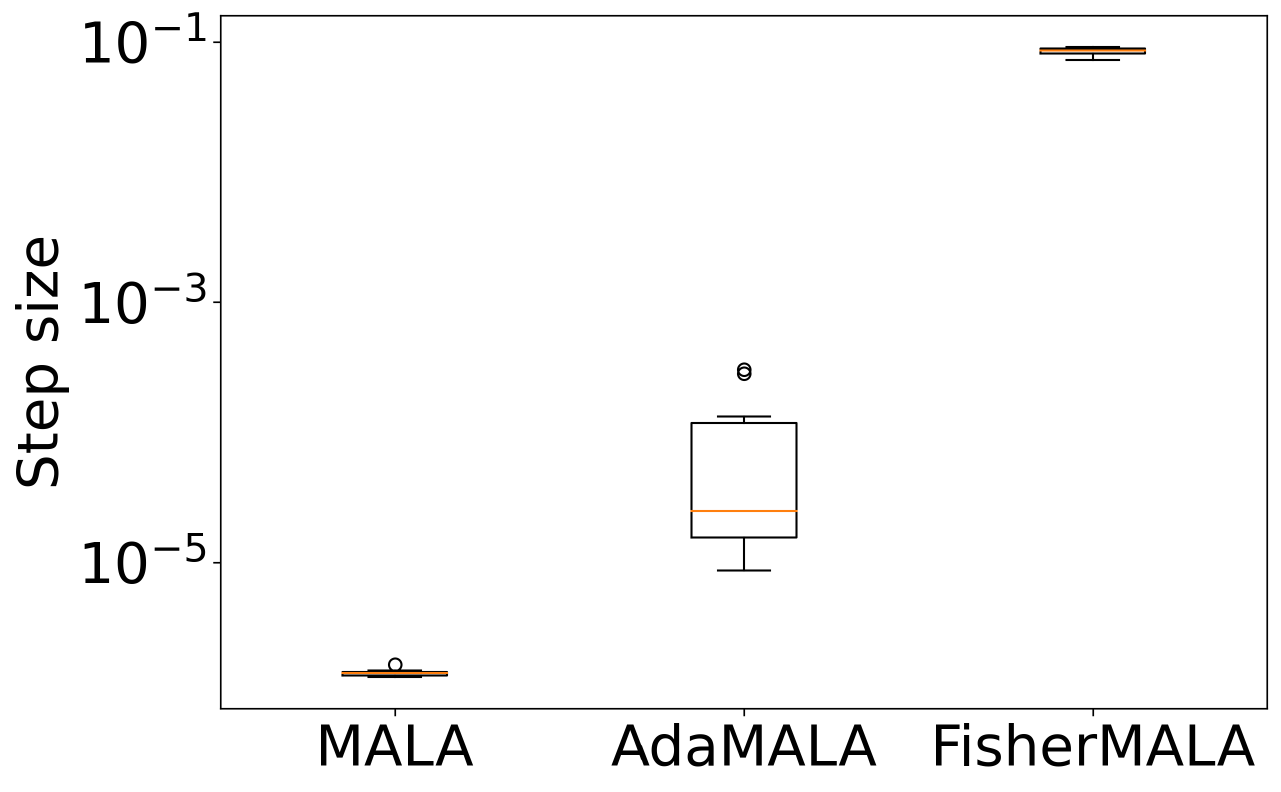} \\
Caravan & MNIST \\
\includegraphics[scale=0.3]
{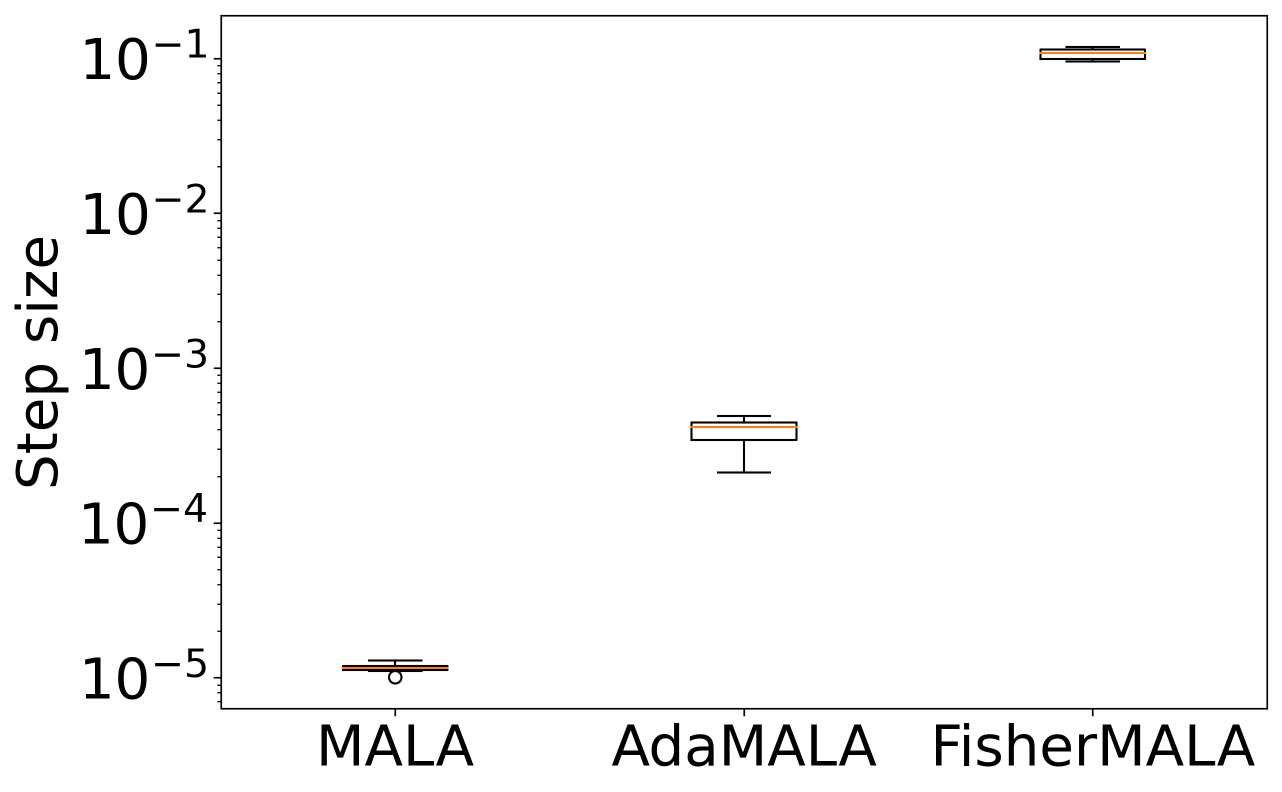} & 
\includegraphics[scale=0.3]
{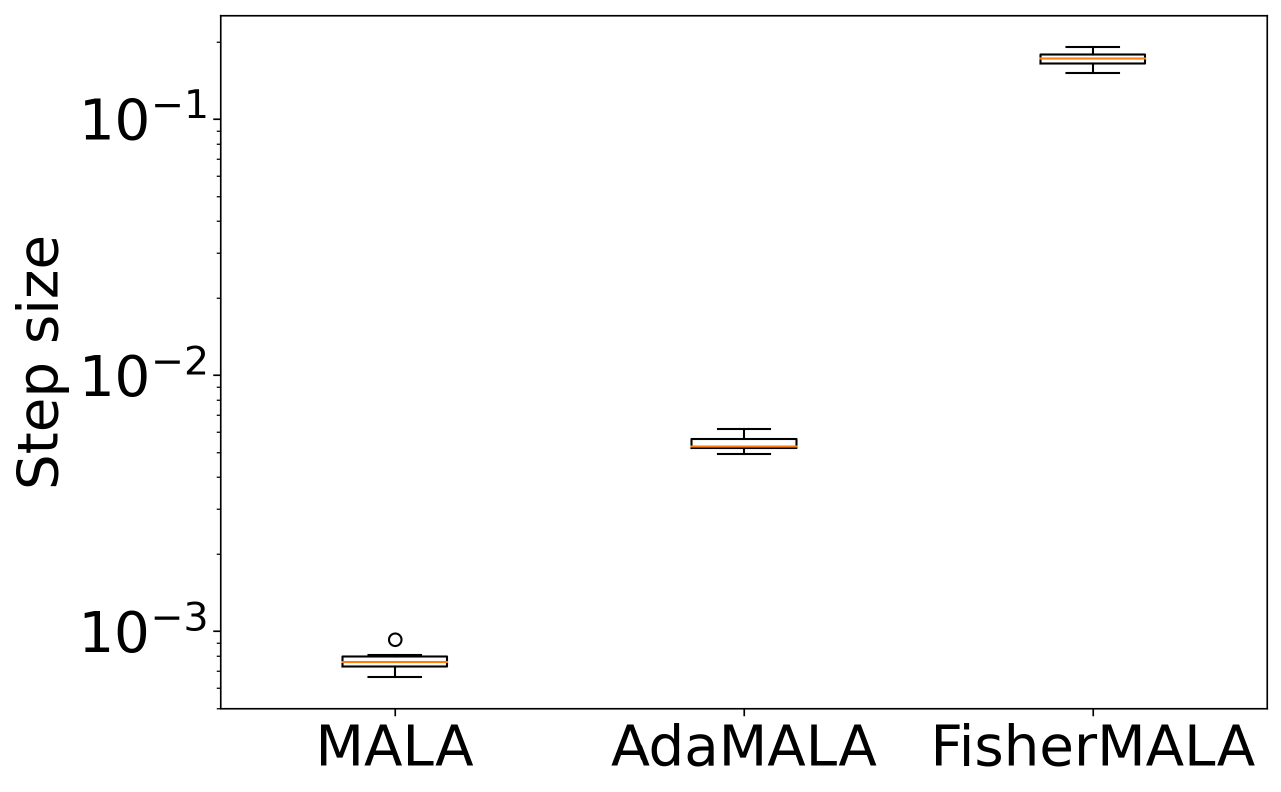}  
\end{tabular}
\caption{It shows the estimated values of $\sigma^2$ for MALA, AdaMALA and FisherMALA using boxplots (each computed from the 10 random repeats;  see Table  \ref{table:large_table}) 
for the four datasets presented in Table \ref{table:large_table}. For better visibility the $y$ axis is shown in log scale.} 
\label{fig:step_size}
\end{figure*}

\subsection{Additional plots and tables}

Figure \ref{fig:2dgaussian_appendix} and \ref{fig:gp_appendix} display additional visualizations for the 2-D Gaussian and the GP target experiments. Tables   \ref{table:neal}-\ref{table:ripley} 
provide the ESS scores for the inhomogeneous Gaussian  target and all remaining Bayesian logistic regression datasets, that were not included in the main paper.   Bold 
font in the "Min ESS" entry in the tables indicates statistical significance.  
Similarly, Figures \ref{fig:gp_logtarget}-\ref{fig:ripley_logtarget} show the log target values across iterations for the four best samplers, i.e.\ excluding simple MALA which is the least performing method.

 \begin{figure*}[!htb]
\centering
\begin{tabular}{ccc}
\includegraphics[scale=0.23]
{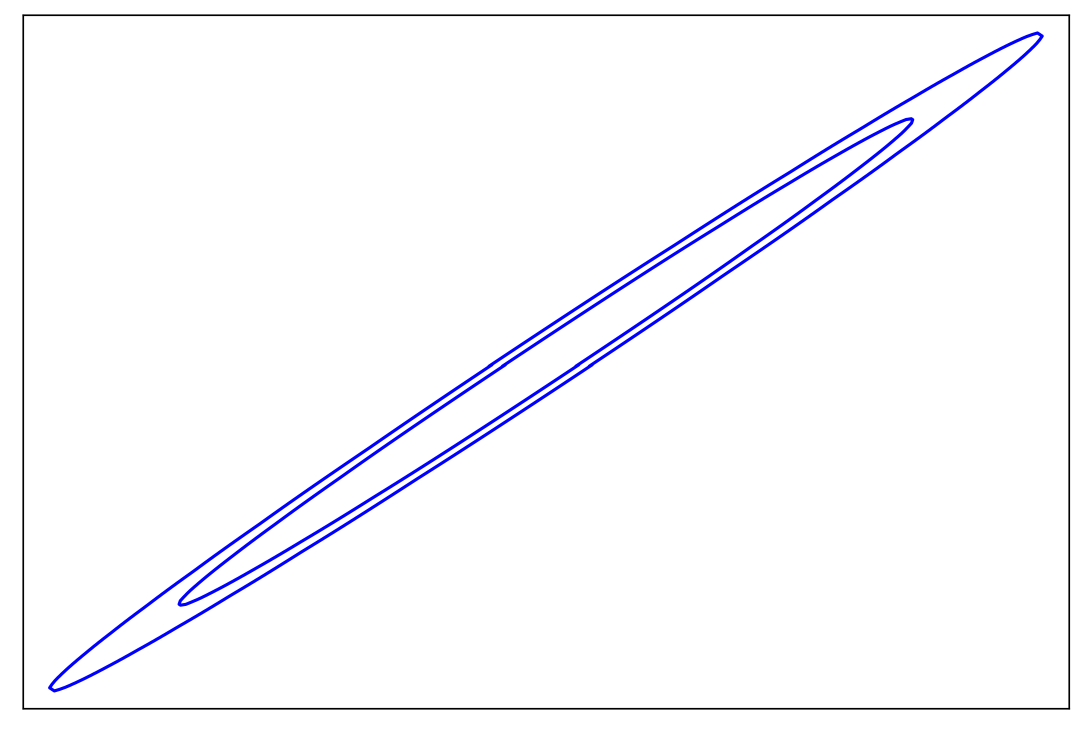} &
\includegraphics[scale=0.23]
{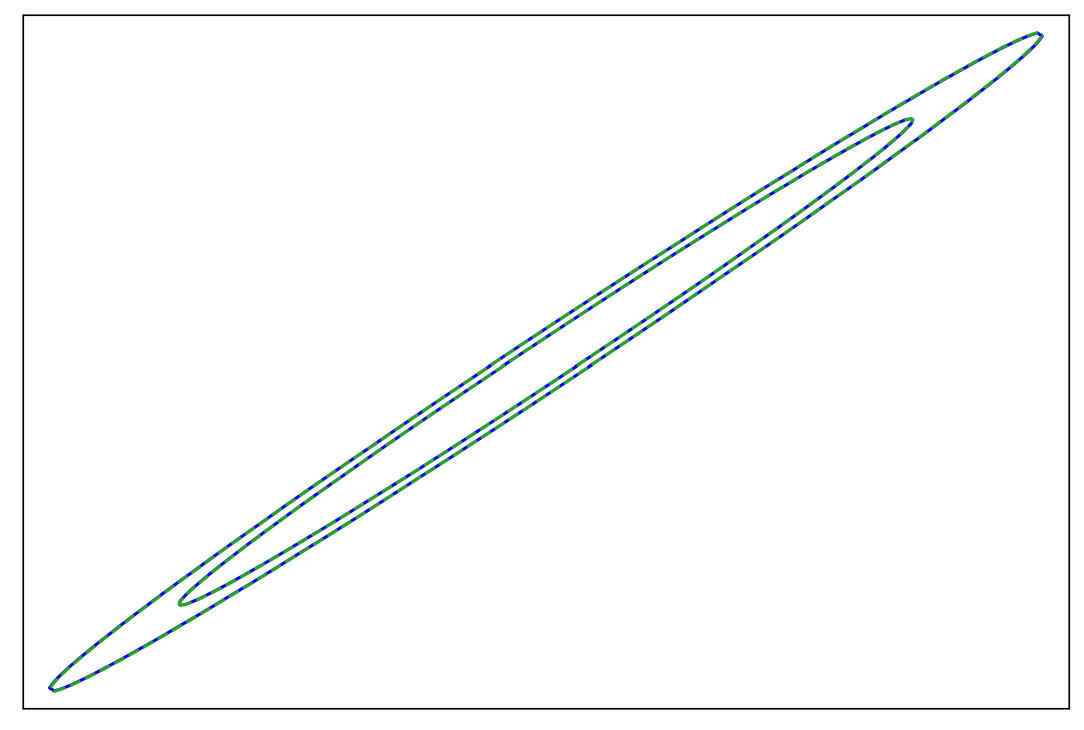} &
\includegraphics[scale=0.23]
{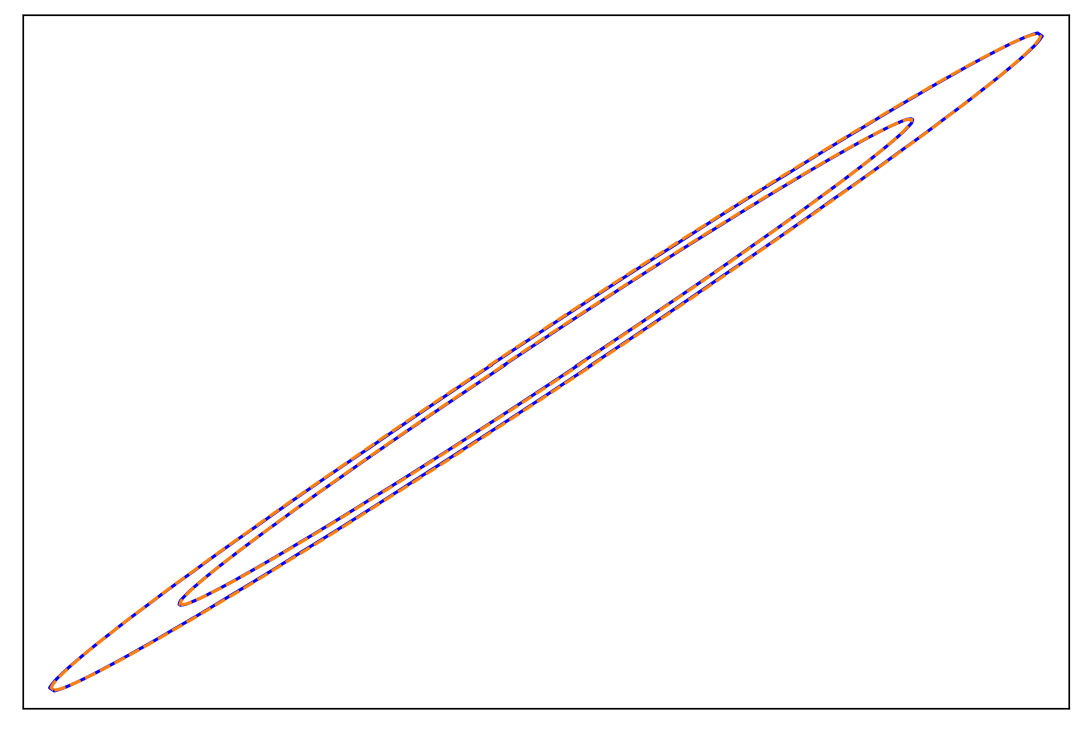} \\
 (a) & (b) & (c)
\end{tabular}
\caption{
Panel (a) shows the true covariance of the 2-D Gaussian. Panel (b) shows the estimated covariance by FisherMALA (dashed green line), where for 
comparison the true covariance is also shown in blue. Panel (c) shows the estimated covariance by AdaMALA (dashed red line).} 
\label{fig:2dgaussian_appendix}
\end{figure*}

 \begin{figure*}[!htb]
\centering
\begin{tabular}{ccc}
\includegraphics[scale=0.3]
{gaussian_true_cov.pdf}  &
\includegraphics[scale=0.3]
{gaussian_fisher_mala_raoblack_True_estimatedImage_cov.pdf} &
\includegraphics[scale=0.3]
{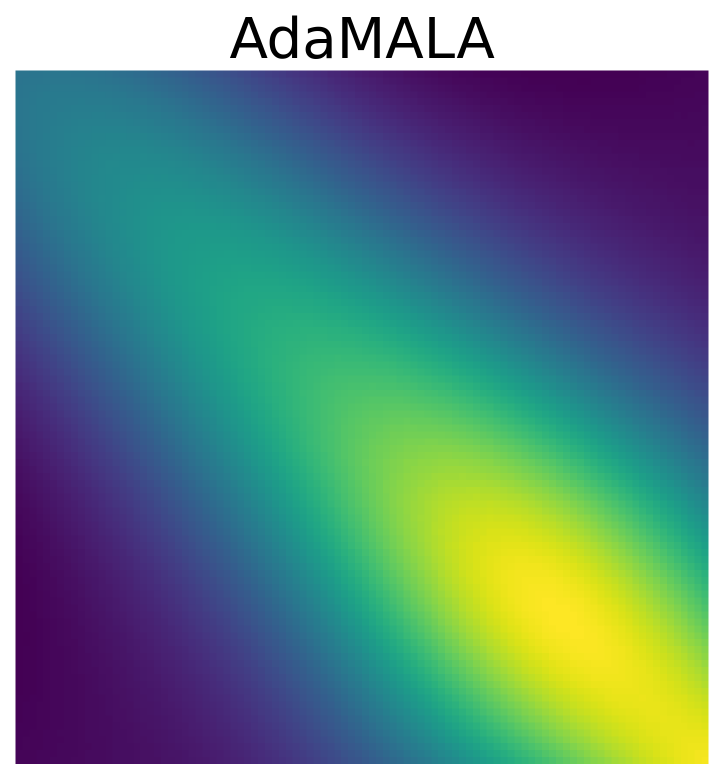}  
\end{tabular}
\caption{The covariance matrices for the GP target, where in the right panel is the covariance estimated by AdaMALA which was not displayed in  Figure
 \ref{fig:2dgaussian_and_gp} in the main text.} 
\label{fig:gp_appendix}
\end{figure*}


\begin{table}
  \caption{ESS scores for the inhomogeneous Gaussian target.}
  \label{table:neal}
  \centering
  \begin{tabular}{llll}
    \toprule
& Max ESS & Median ESS & Min ESS \\
MALA & $13695.291 \pm 1369.515$ & $9.793 \pm 0.655$ & $2.943 \pm 0.130$ \\
AdaMALA & $4310.690 \pm 606.618$ & $70.802 \pm 14.912$ & $9.225 \pm 3.272$ \\
HMC & $19362.103 \pm 1372.400$ & $381.205 \pm 101.781$ & $42.033 \pm 33.080$ \\
mMALA & $2354.354 \pm 65.835$ & $2014.801 \pm 23.713$ & $1490.119 \pm 108.745$ \\
FisherMALA & $2347.340 \pm 70.234$ & $2002.579 \pm 30.001$ & $1500.983 \pm 67.087$ \\
  \bottomrule
  \end{tabular}
\end{table}




\begin{table}
  \caption{ESS scores for the Heart dataset.}
  \label{table:heart}
  \centering
  \begin{tabular}{llll}
    \toprule
    & Max ESS & Median ESS & Min ESS \\
MALA & $68.774 \pm 25.304$ & $5.354 \pm 1.056$ & $2.898 \pm 0.104$ \\
AdaMALA & $208.636 \pm 124.762$ & $14.762 \pm 9.134$ & $3.781 \pm 0.731$ \\
HMC & $387.321 \pm 311.673$ & $12.991 \pm 4.009$ & $4.064 \pm 1.120$ \\
mMALA & $878.858 \pm 1079.674$ & $789.356 \pm 969.806$ & $651.793 \pm 806.477$ \\
FisherMALA & $4864.278 \pm 103.277$ & $4474.288 \pm 102.029$ & ${\bf 3954.793} \pm 199.832$ \\
  \bottomrule
  \end{tabular}
\end{table}

\begin{table}
  \caption{ESS scores for the  German Credit dataset.}
  \label{table:german}
  \centering
  \begin{tabular}{llll}
    \toprule
    & Max ESS & Median ESS & Min ESS \\
MALA & $262.206 \pm 211.839$ & $5.932 \pm 0.668$ & $2.972 \pm 0.212$ \\
AdaMALA & $223.592 \pm 111.914$ & $16.111 \pm 5.058$ & $3.774 \pm 0.653$ \\
HMC & $10439.824 \pm 9572.157$ & $45.872 \pm 7.823$ & $5.431 \pm 1.257$ \\
mMALA & $3066.605 \pm 100.768$ & $2767.022 \pm 94.222$ & $2342.902 \pm 112.610$ \\
FisherMALA & $3951.807 \pm 78.858$ & $3582.184 \pm 90.551$ & ${\bf 3011.483} \pm 258.154$ \\
  \bottomrule
  \end{tabular}
\end{table}

\begin{table}
  \caption{ESS scores for the Australian Credit dataset.}
  \label{table:australian}
  \centering
  \begin{tabular}{llll}
    \toprule
    & Max ESS & Median ESS & Min ESS \\
MALA & $15.627 \pm 12.892$ & $3.823 \pm 1.166$ & $2.611 \pm 0.538$ \\
AdaMALA & $1525.373 \pm 1600.986$ & $6.986 \pm 3.200$ & $3.297 \pm 0.456$ \\
HMC & $1282.235 \pm 932.038$ & $6.966 \pm 1.249$ & $2.856 \pm 0.095$ \\
mMALA & $2609.462 \pm 881.967$ & $2308.175 \pm 776.872$ & $1869.364 \pm 630.880$ \\
FisherMALA & $4732.724 \pm 116.074$ & $4361.969 \pm 104.750$ & ${\bf 3772.086} \pm 265.170$ \\
  \bottomrule
  \end{tabular}
\end{table}

\begin{table}
  \caption{ESS scores for the Ripley dataset.}
  \label{table:ripley}
  \centering
  \begin{tabular}{llll}
    \toprule
    & Max ESS & Median ESS & Min ESS \\
MALA & $2058.325 \pm 180.839$ & $496.981 \pm 68.029$ & $427.492 \pm 60.006$ \\
AdaMALA & $9678.793 \pm 384.295$ & $9497.814 \pm 463.059$ & $9272.026 \pm 412.361$ \\
HMC & $18403.796 \pm 3202.136$ & $18254.161 \pm 3513.550$ & $7644.709 \pm 2288.559$ \\
mMALA & $9333.633 \pm 280.238$ & $8941.579 \pm 288.223$ & $8655.640 \pm 396.106$ \\
FisherMALA & $9875.968 \pm 218.801$ & $9673.009 \pm 280.759$ & $9244.631 \pm 559.137$ \\
  \bottomrule
  \end{tabular}
\end{table}

\begin{figure}
\centering
\begin{tabular}{cccc}
\includegraphics[scale=0.19]
{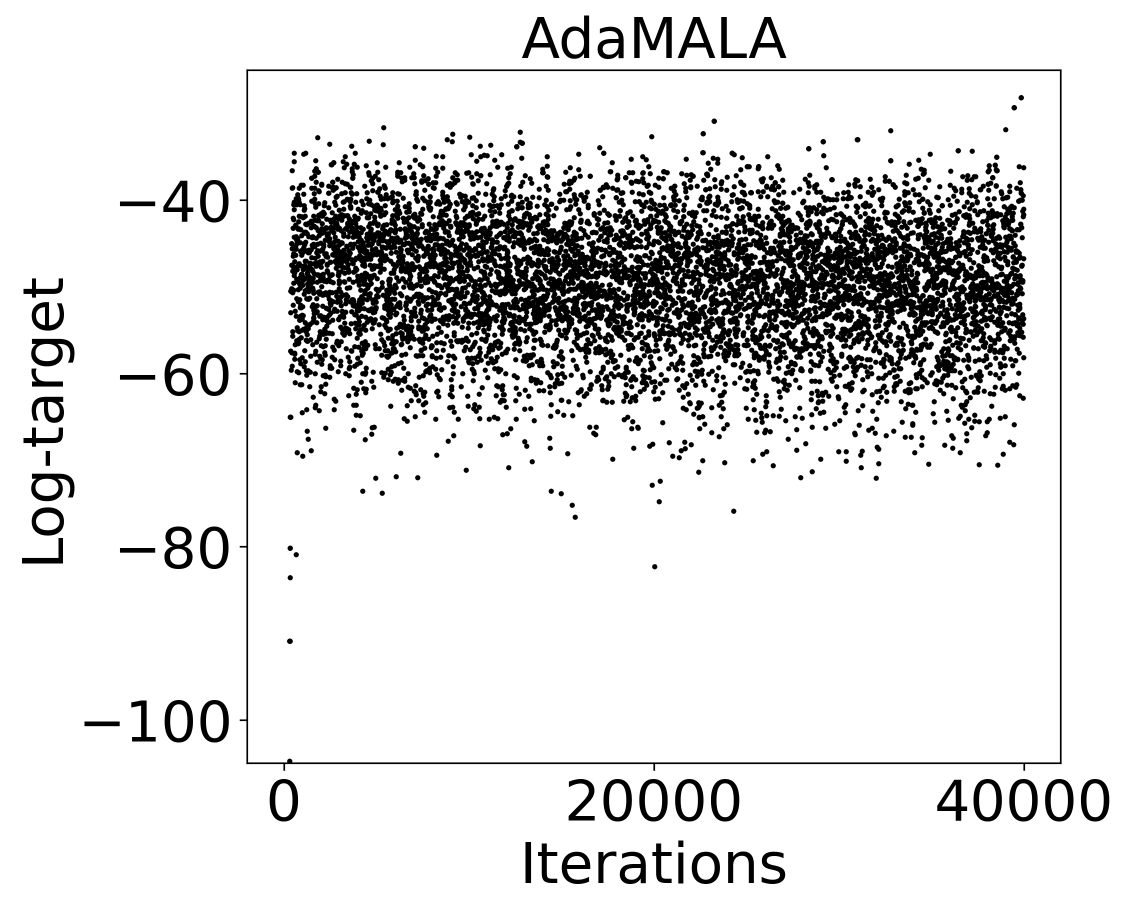} &
\includegraphics[scale=0.19]
{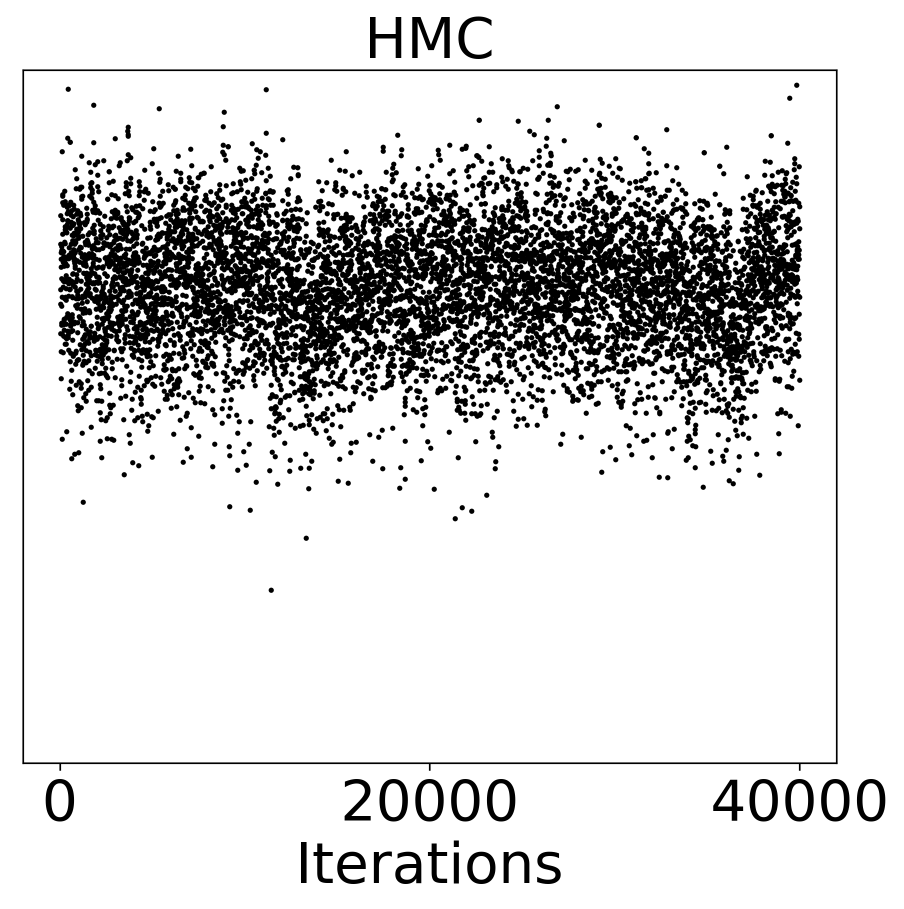} &
\includegraphics[scale=0.19]
{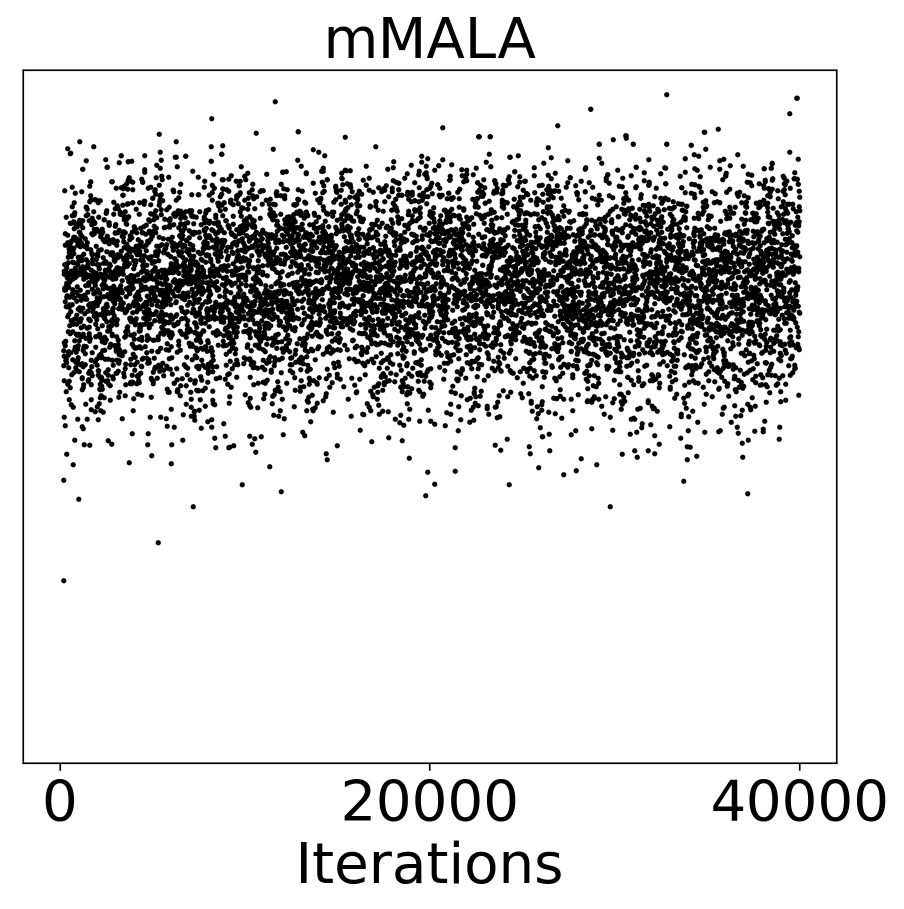} &
\includegraphics[scale=0.19]
{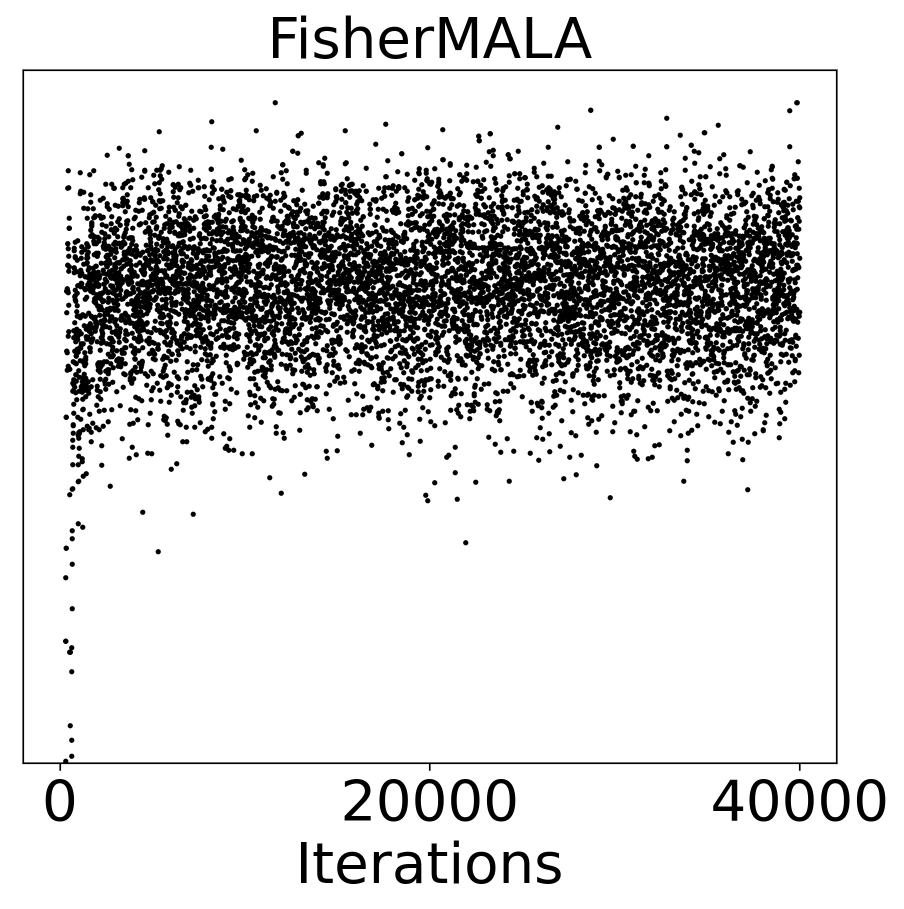}
\end{tabular}
\caption{The evolution of  the log-target across iterations in the GP target.} 
\label{fig:gp_logtarget}
\end{figure}

\begin{figure}
\centering
\begin{tabular}{cccc}
\includegraphics[scale=0.19]
{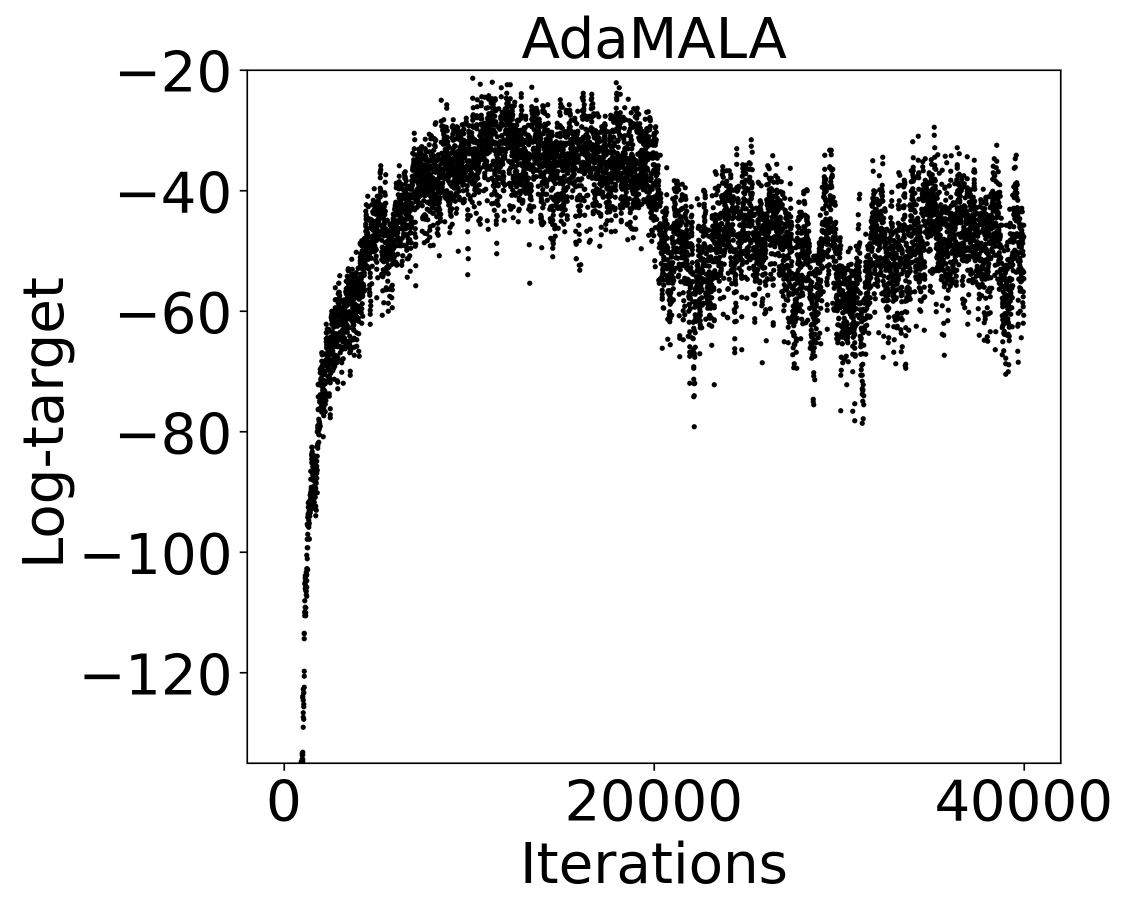} &
\includegraphics[scale=0.19]
{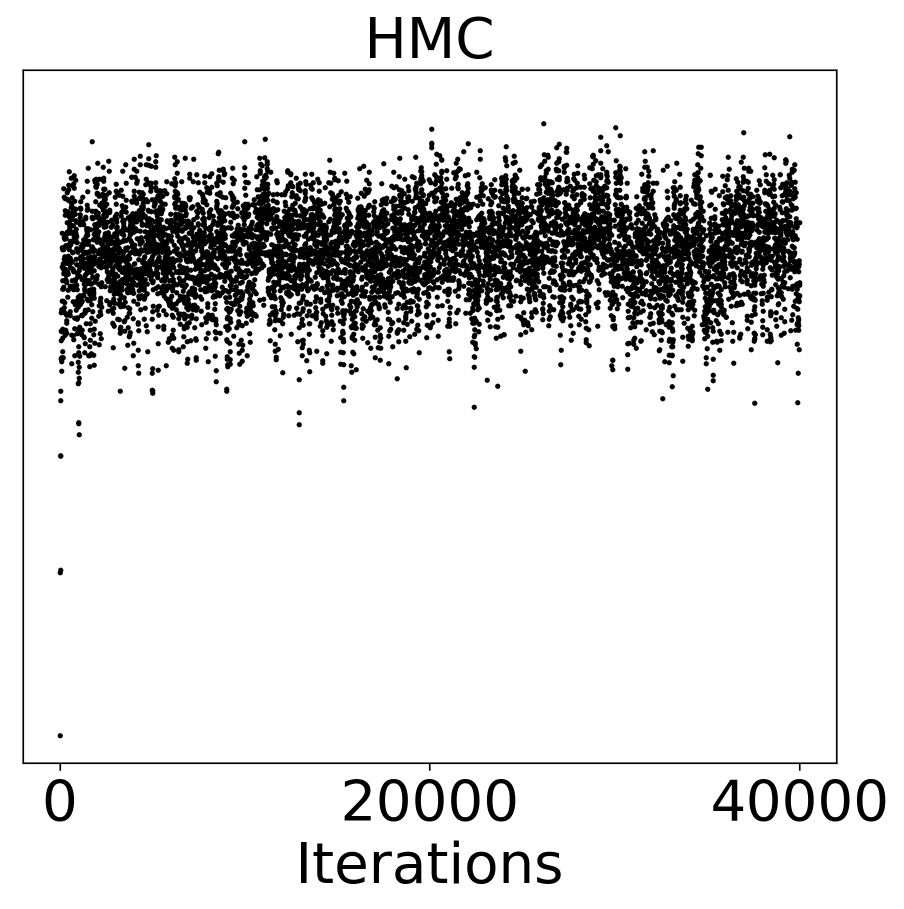} &
\includegraphics[scale=0.19]
{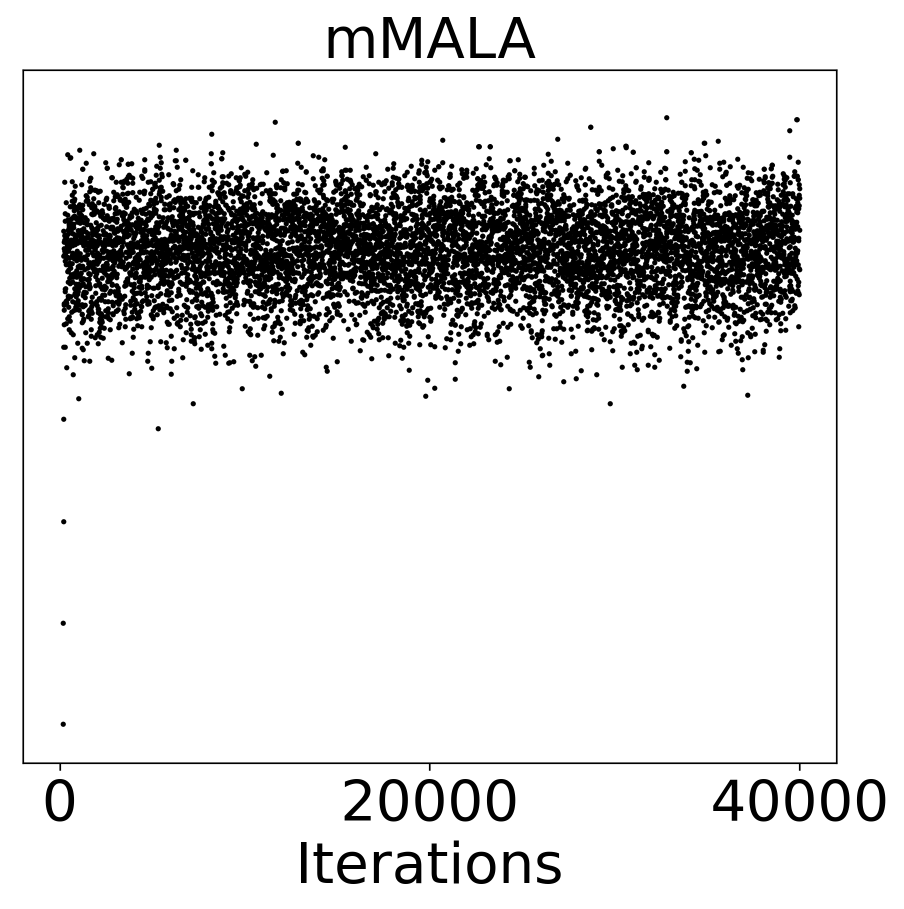} &
\includegraphics[scale=0.19]
{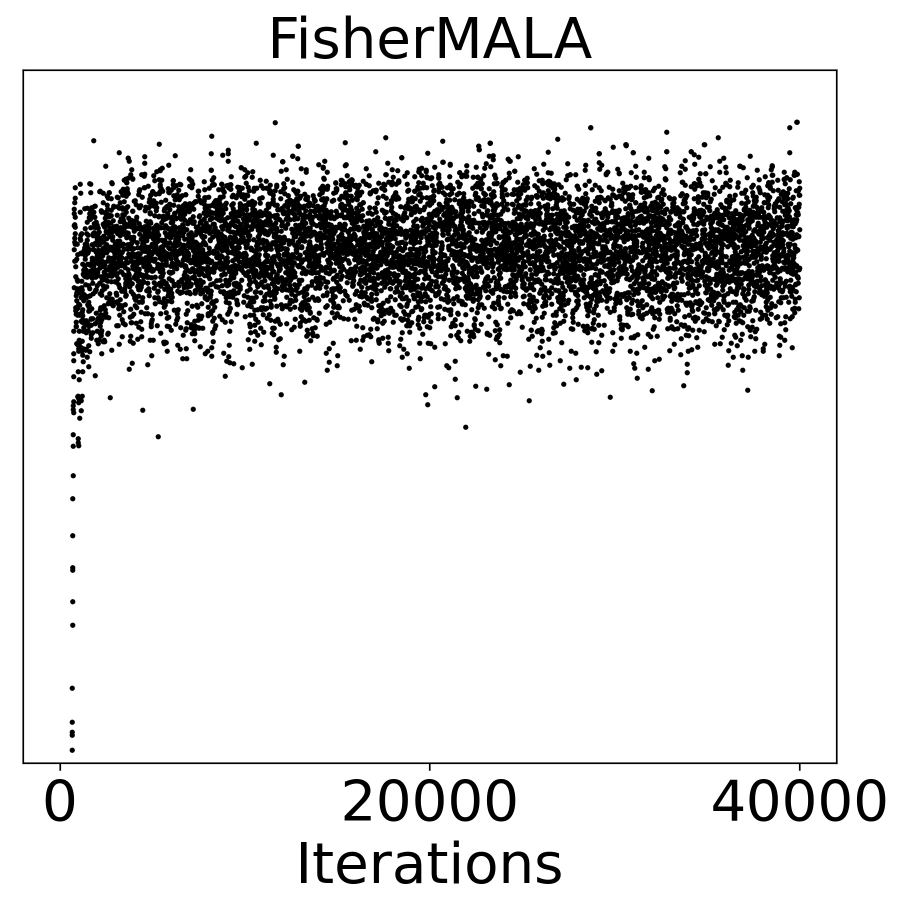}
\end{tabular}
\caption{The evolution of  the log-target across iterations in the inhomogeneous Gaussian target.} 
\label{fig:neal_logtarget}
\end{figure}

\begin{figure}
\centering
\begin{tabular}{cccc}
\includegraphics[scale=0.19]
{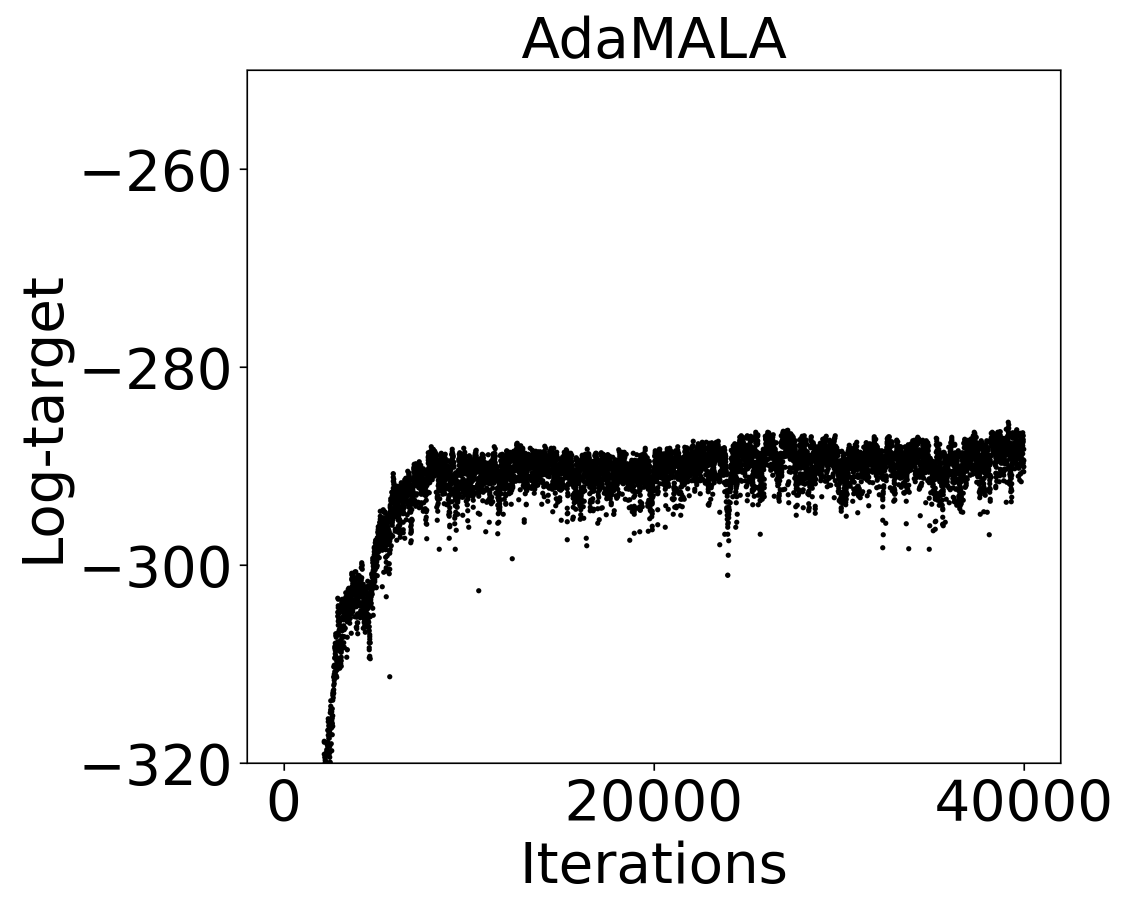} &
\includegraphics[scale=0.19]
{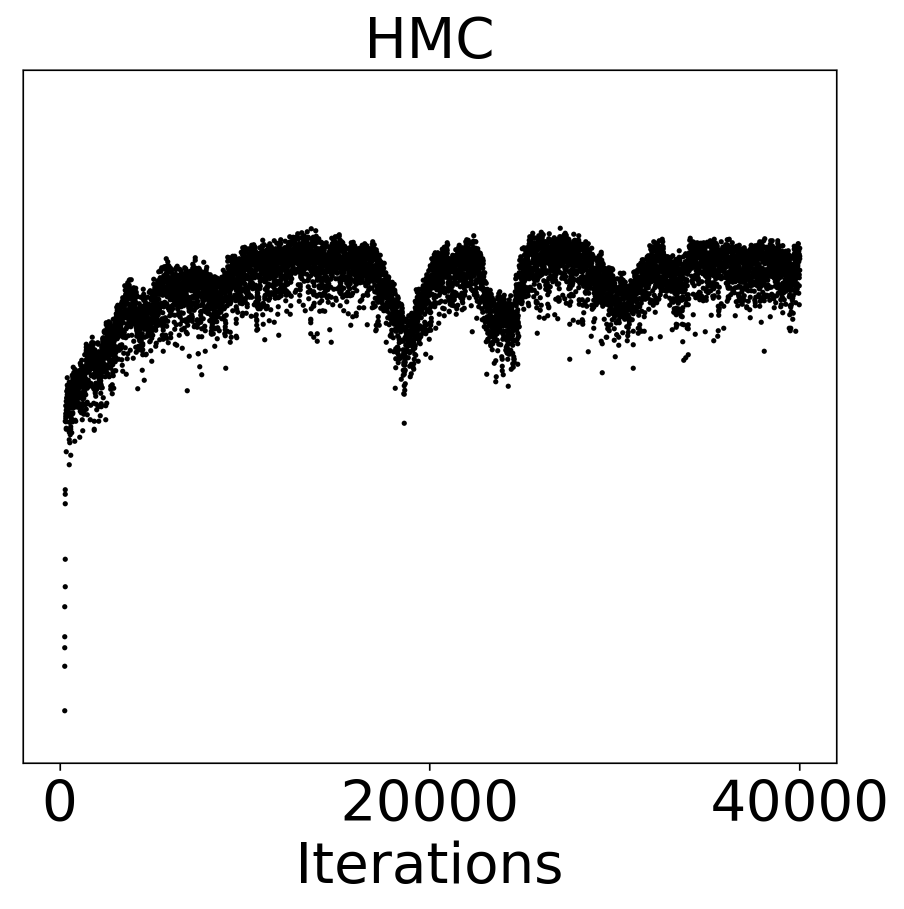} &
\includegraphics[scale=0.19]
{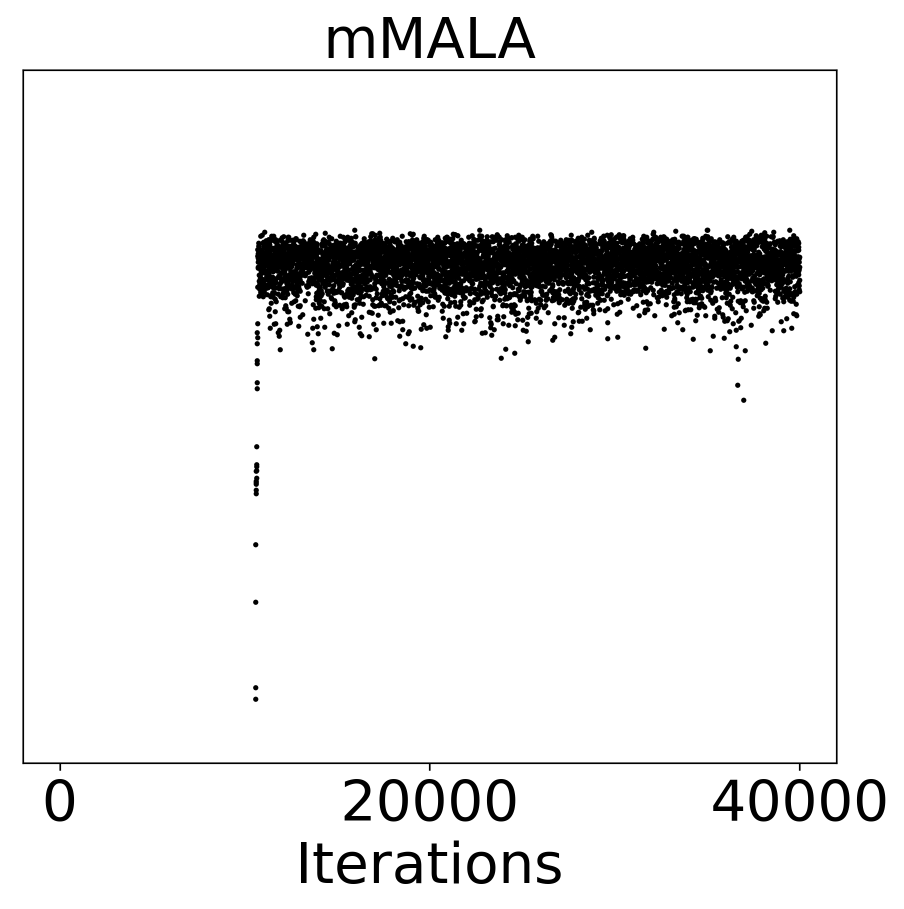} &
\includegraphics[scale=0.19]
{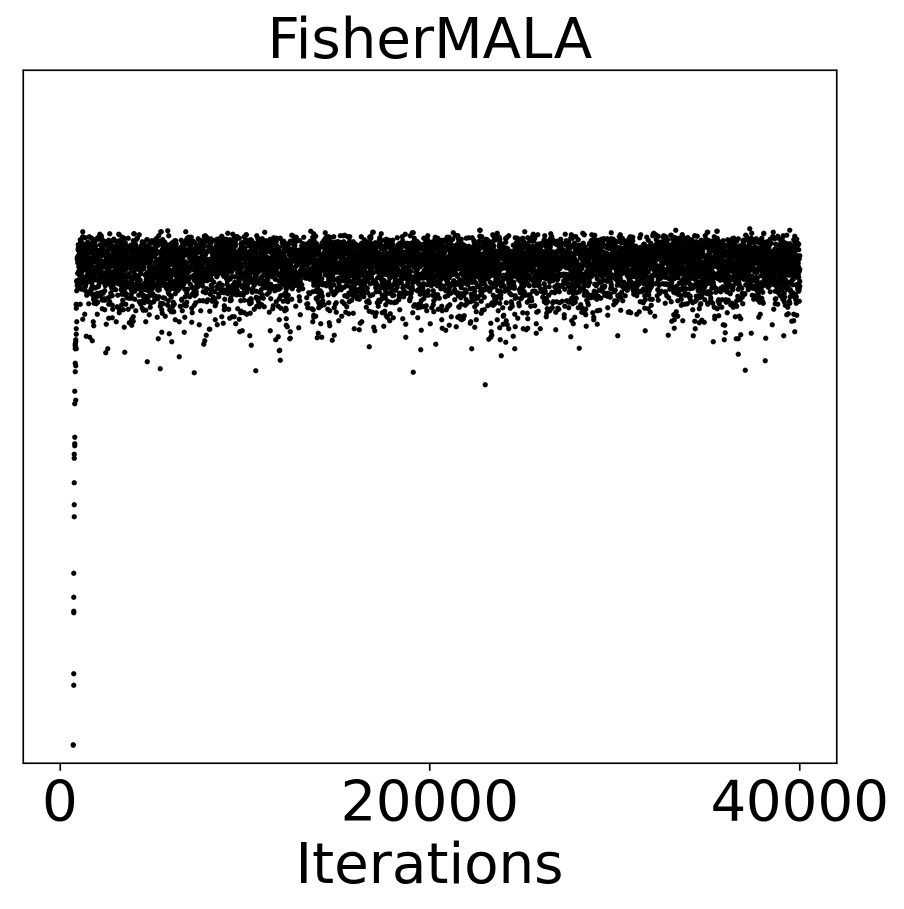}
\end{tabular}
\caption{The evolution of  the log-target across iterations in Pima Indians dataset.} 
\label{fig:german_logtarget}
\end{figure}


\begin{figure}
\centering
\begin{tabular}{cccc}
\includegraphics[scale=0.19]
{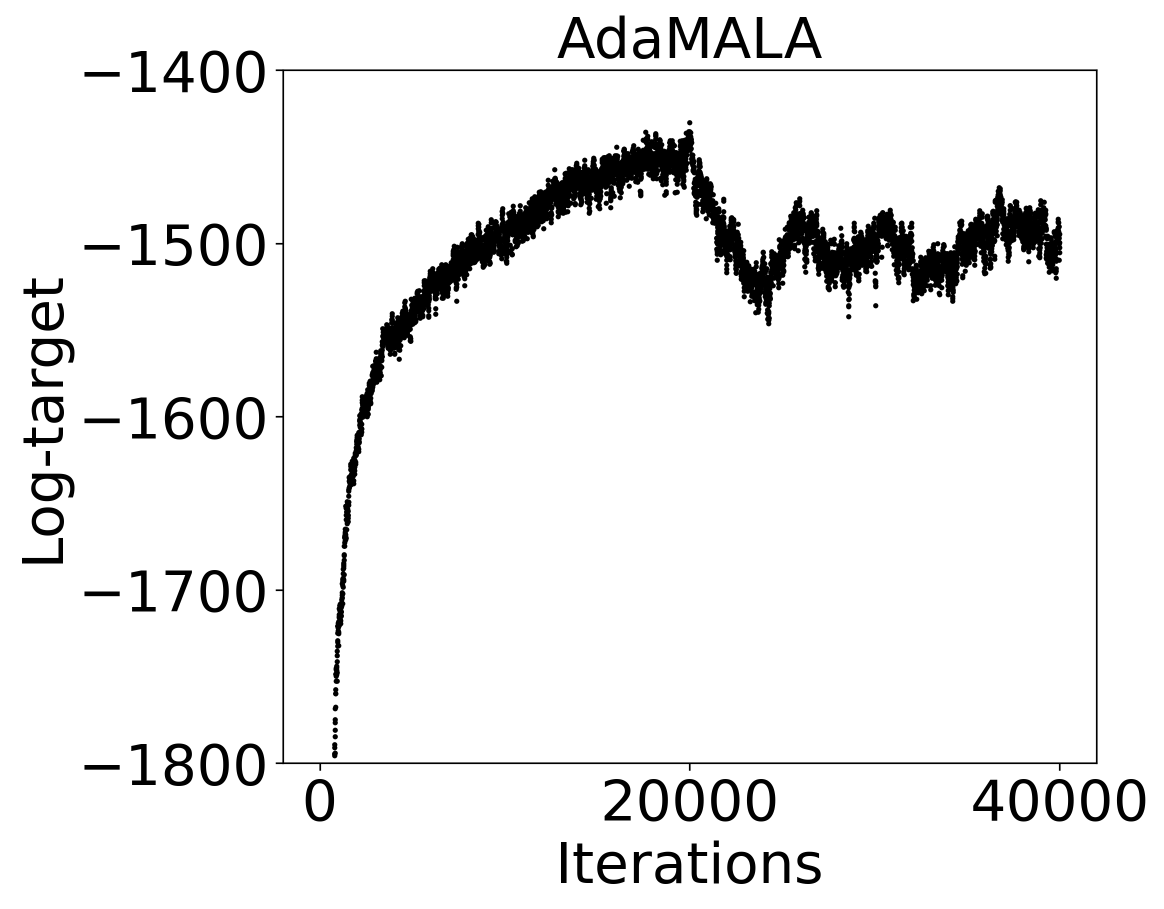} &
\includegraphics[scale=0.19]
{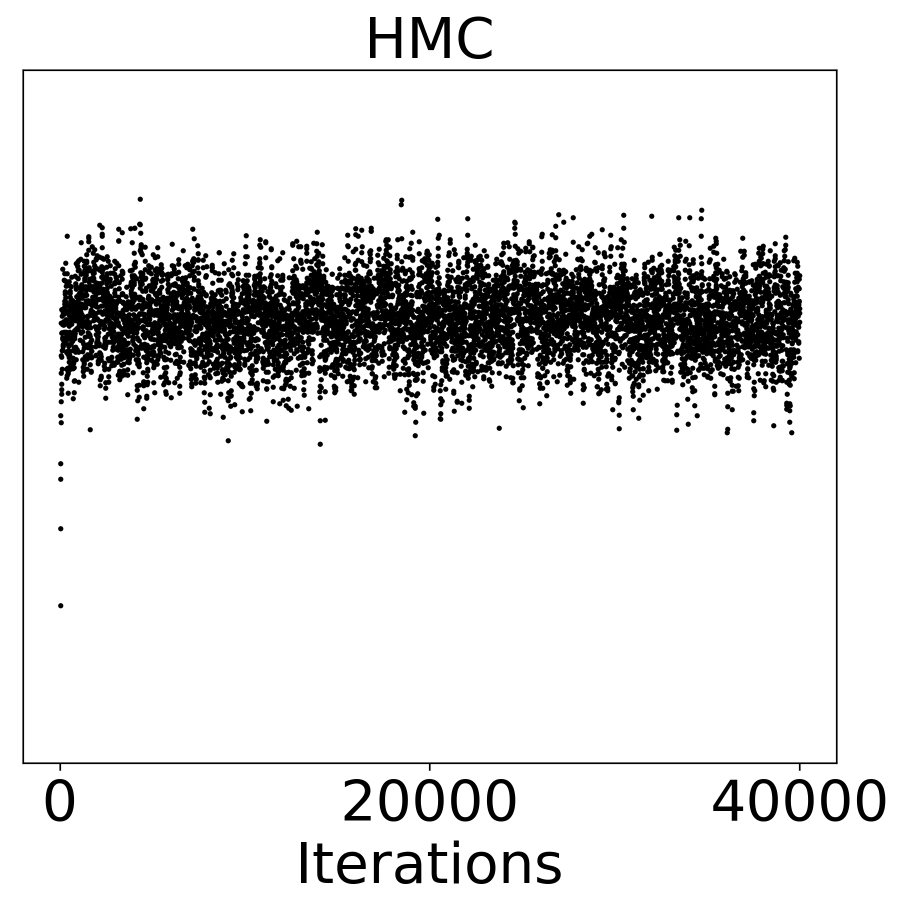} &
\includegraphics[scale=0.19]
{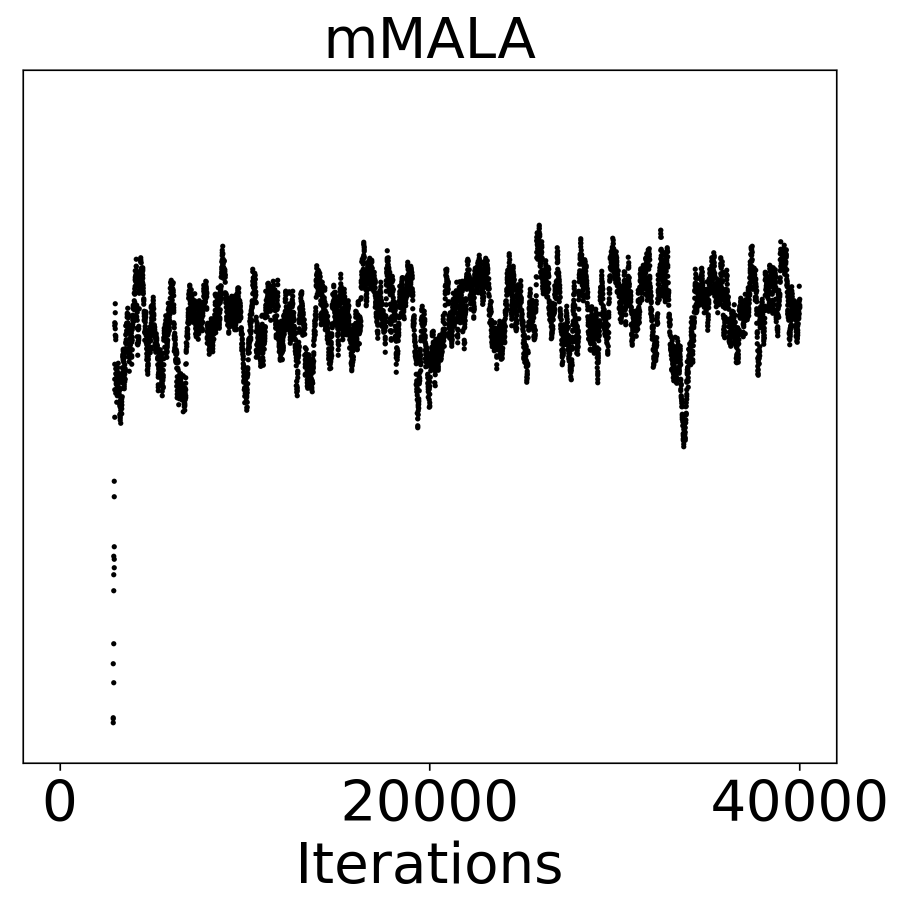} &
\includegraphics[scale=0.19]
{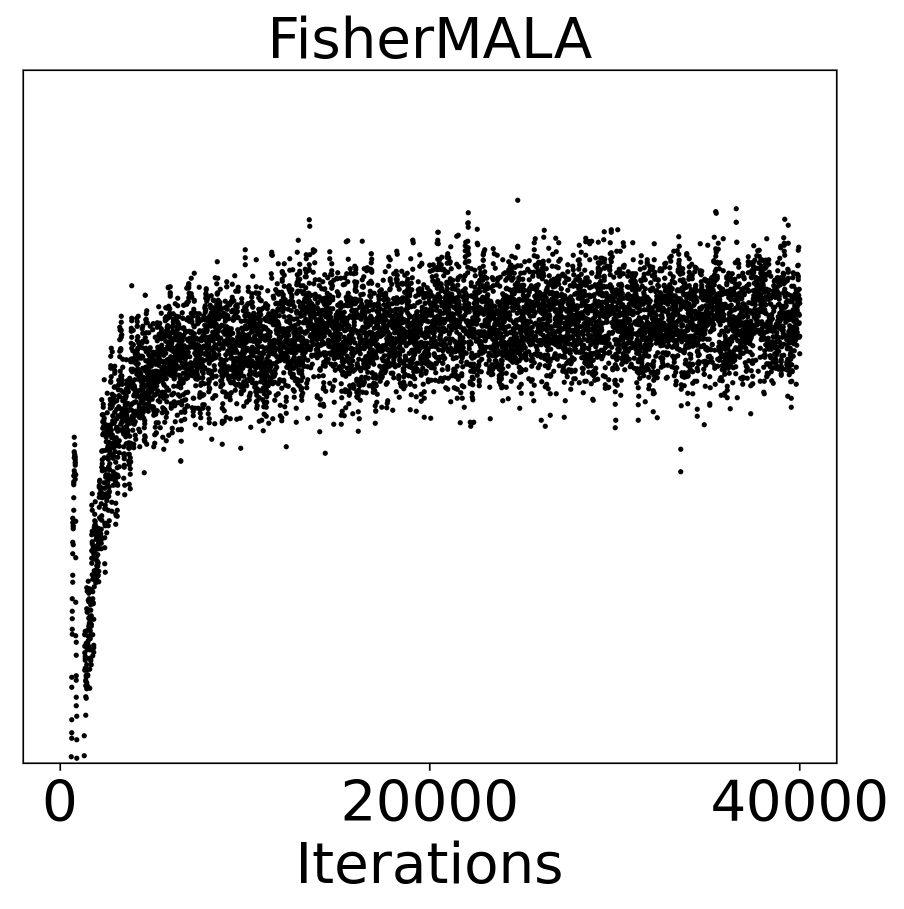}
\end{tabular}
\caption{The evolution of  the log-target across iterations in MNIST dataset.} 
\label{fig:mnist_logtarget}
\end{figure}

\begin{figure}
\centering
\begin{tabular}{cccc}
\includegraphics[scale=0.19]
{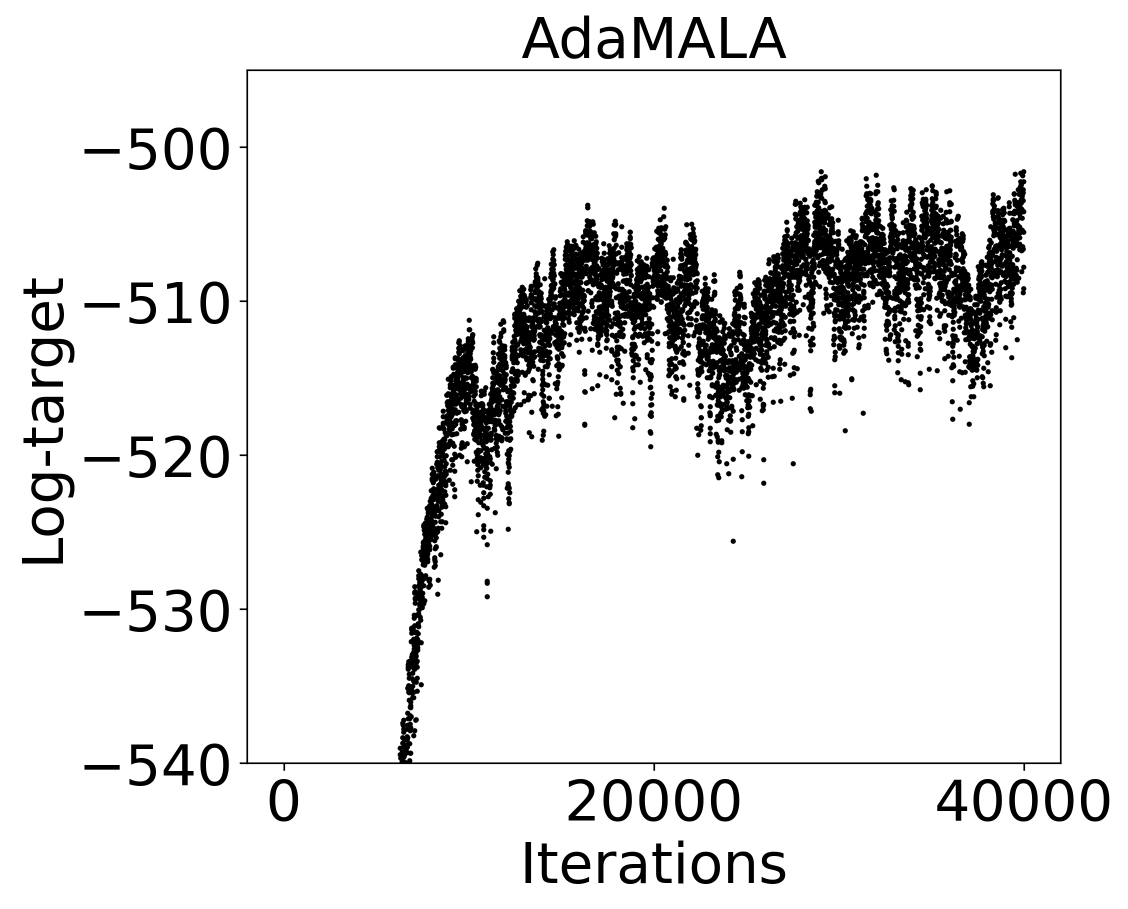} &
\includegraphics[scale=0.19]
{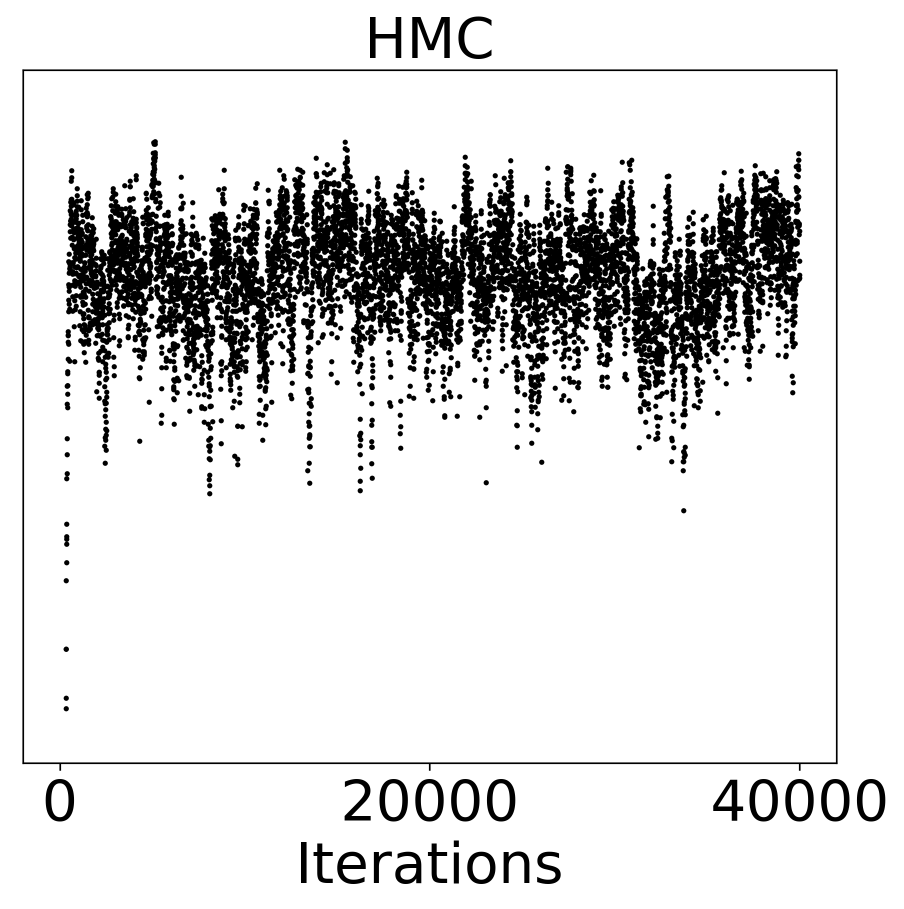} &
\includegraphics[scale=0.19]
{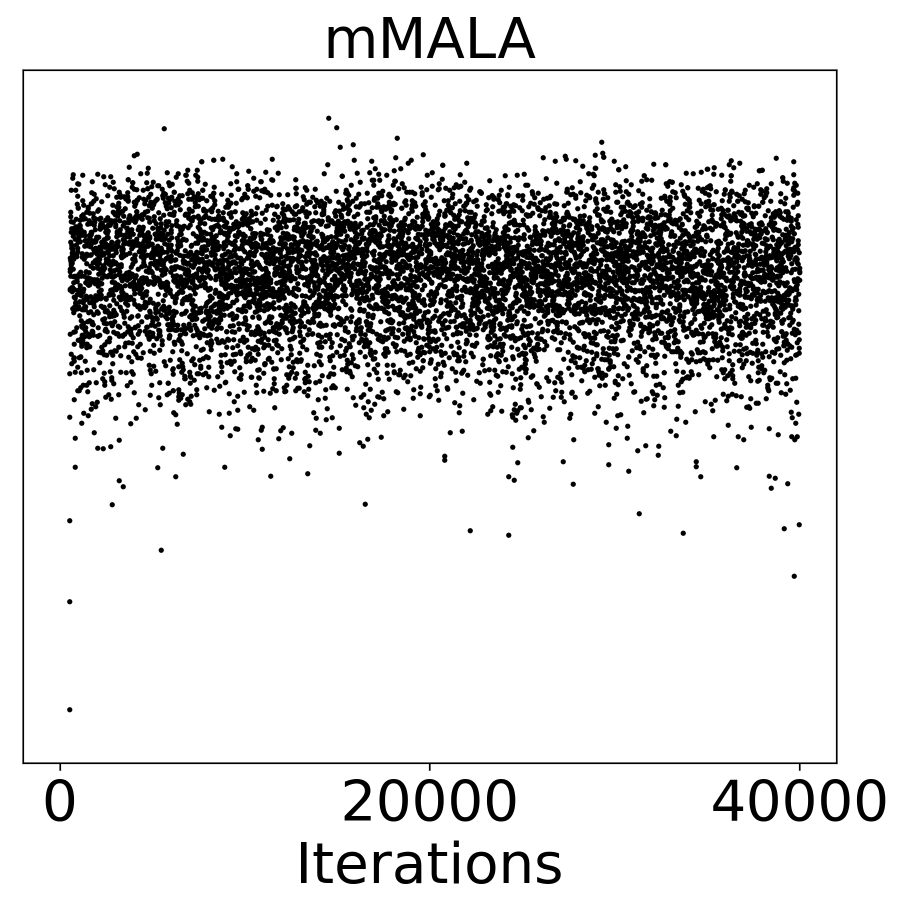} &
\includegraphics[scale=0.19]
{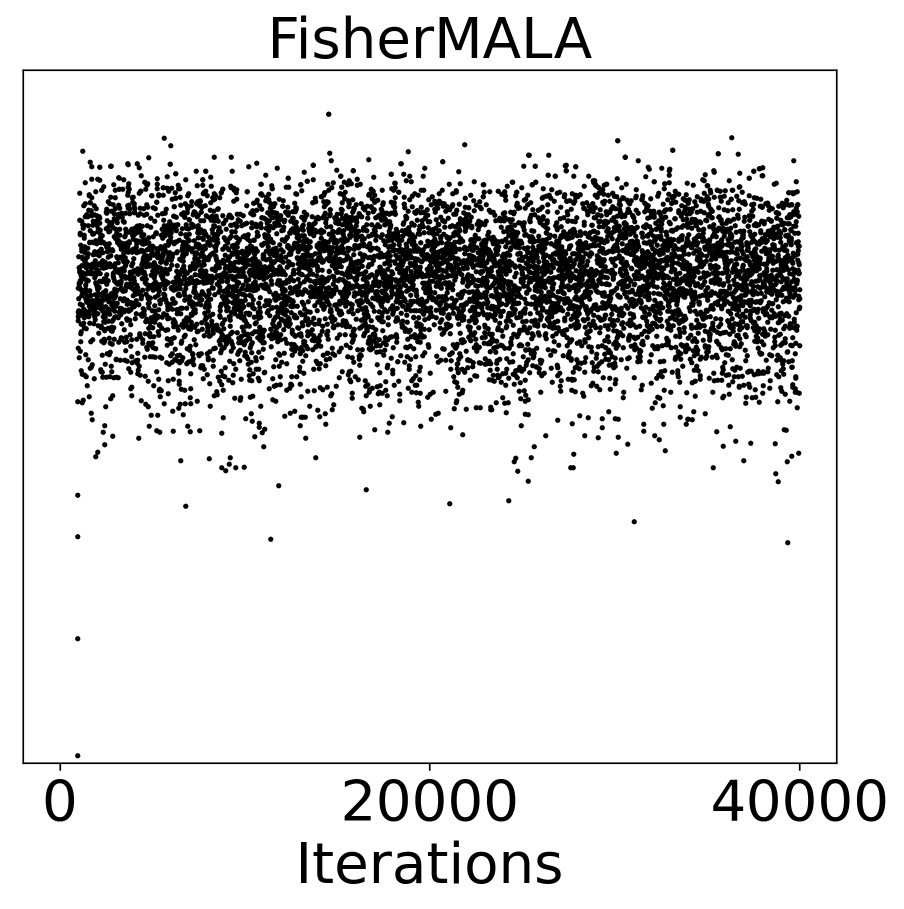}
\end{tabular}
\caption{The evolution of  the log-target across iterations in German Credit dataset.} 
\label{fig:german_logtarget}
\end{figure}

\begin{figure}
\centering
\begin{tabular}{cccc}
\includegraphics[scale=0.19]
{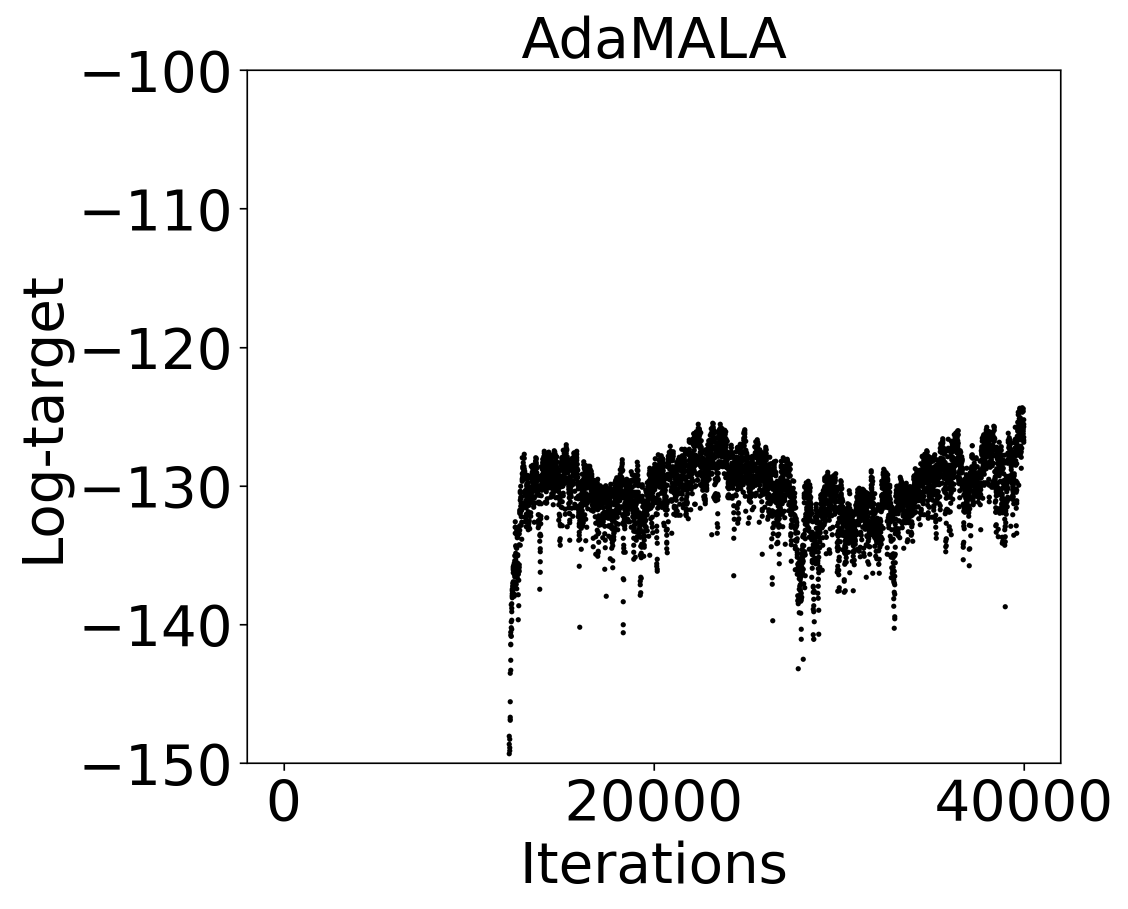} &
\includegraphics[scale=0.19]
{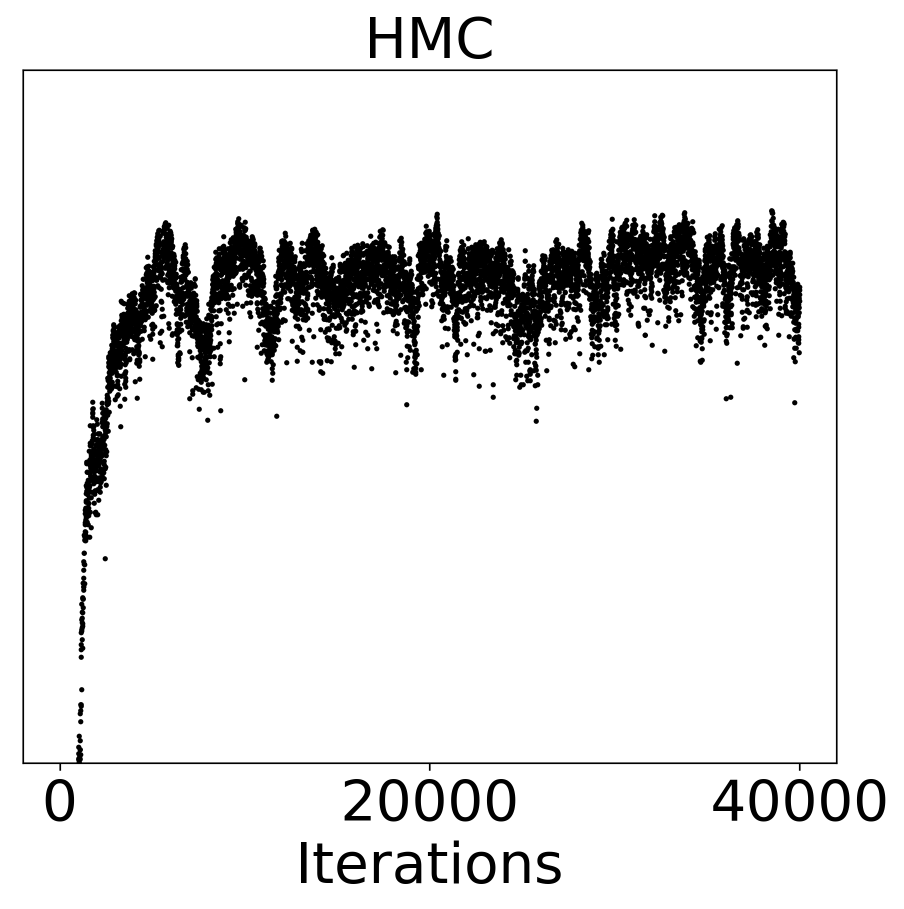} &
\includegraphics[scale=0.19]
{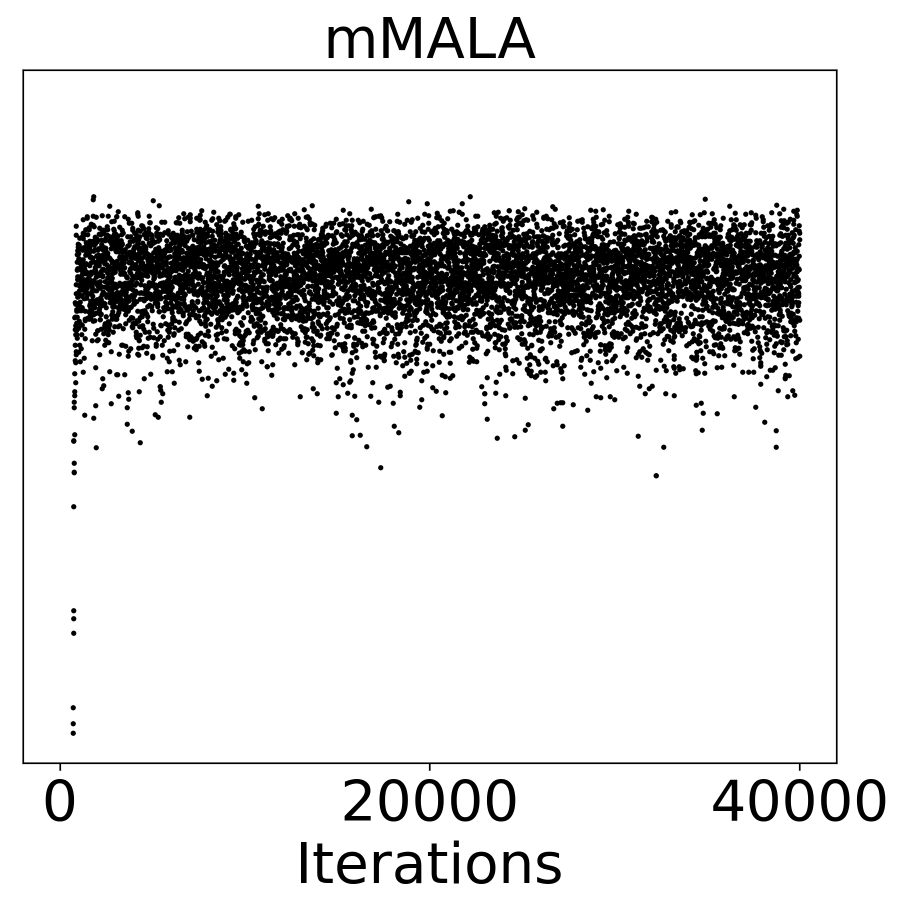} &
\includegraphics[scale=0.19]
{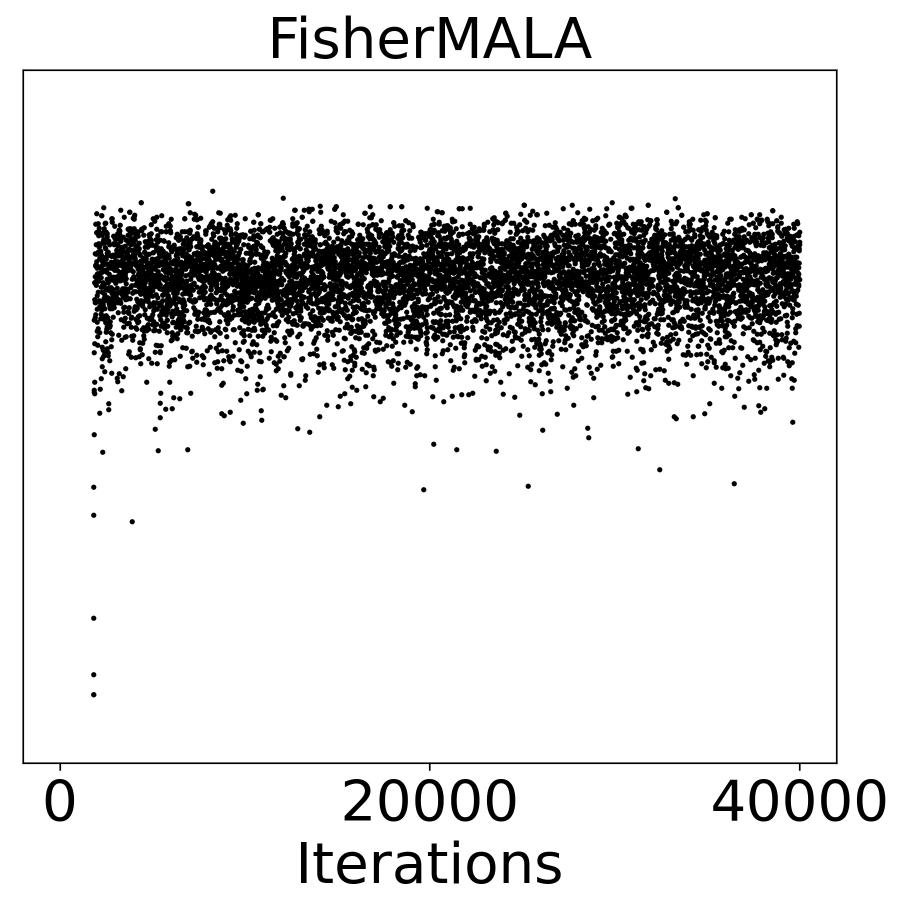}
\end{tabular}
\caption{The evolution of  the log-target across iterations in Heart dataset.} 
\label{fig:heart_logtarget}
\end{figure}

\begin{figure}
\centering
\begin{tabular}{cccc}
\includegraphics[scale=0.19]
{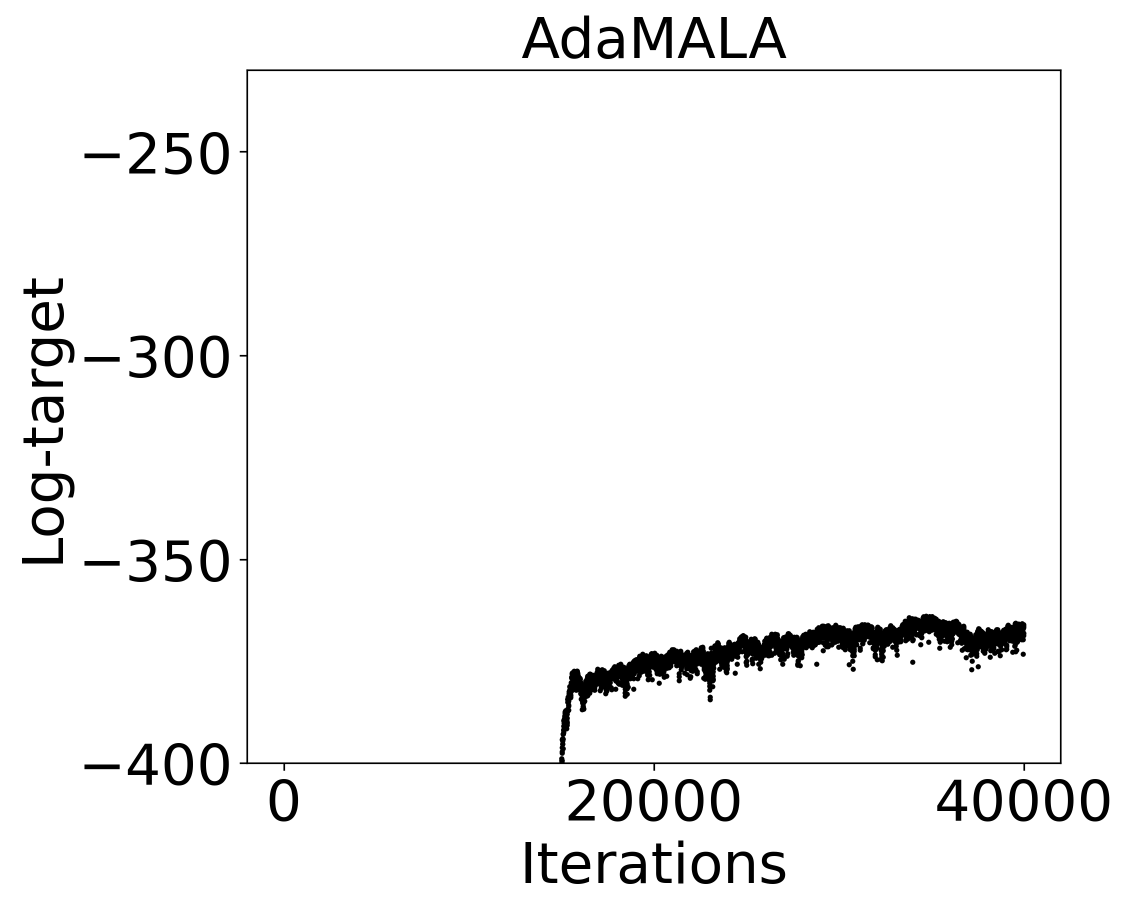} &
\includegraphics[scale=0.19]
{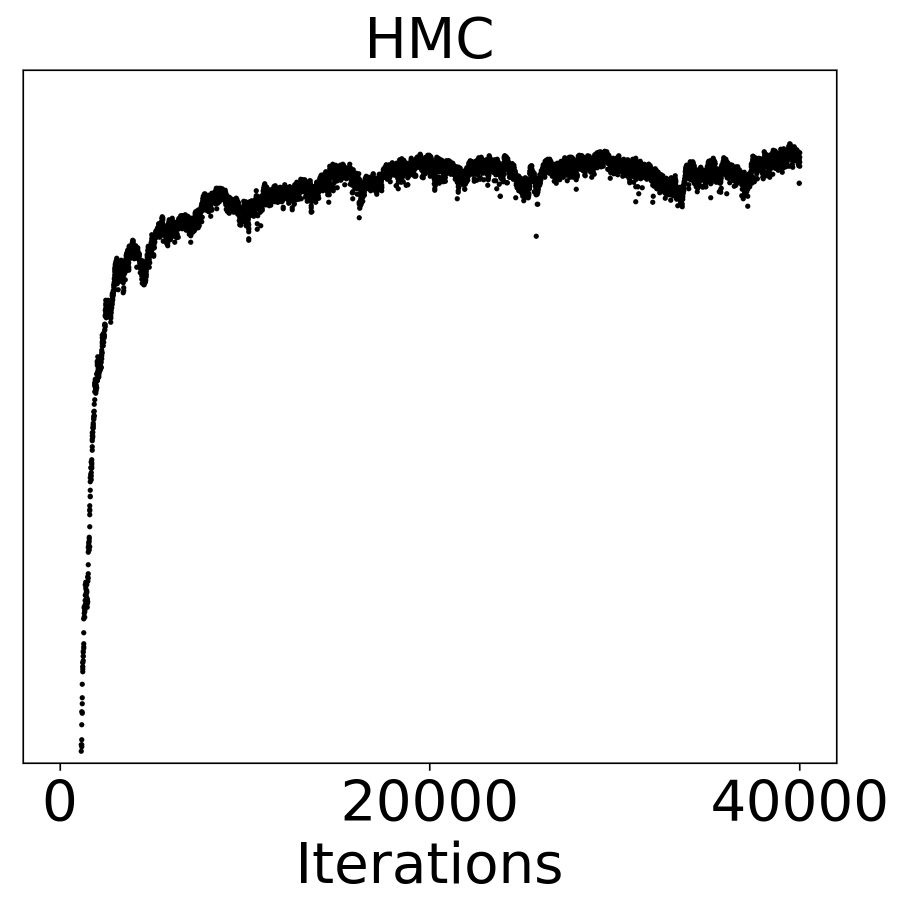} &
\includegraphics[scale=0.19]
{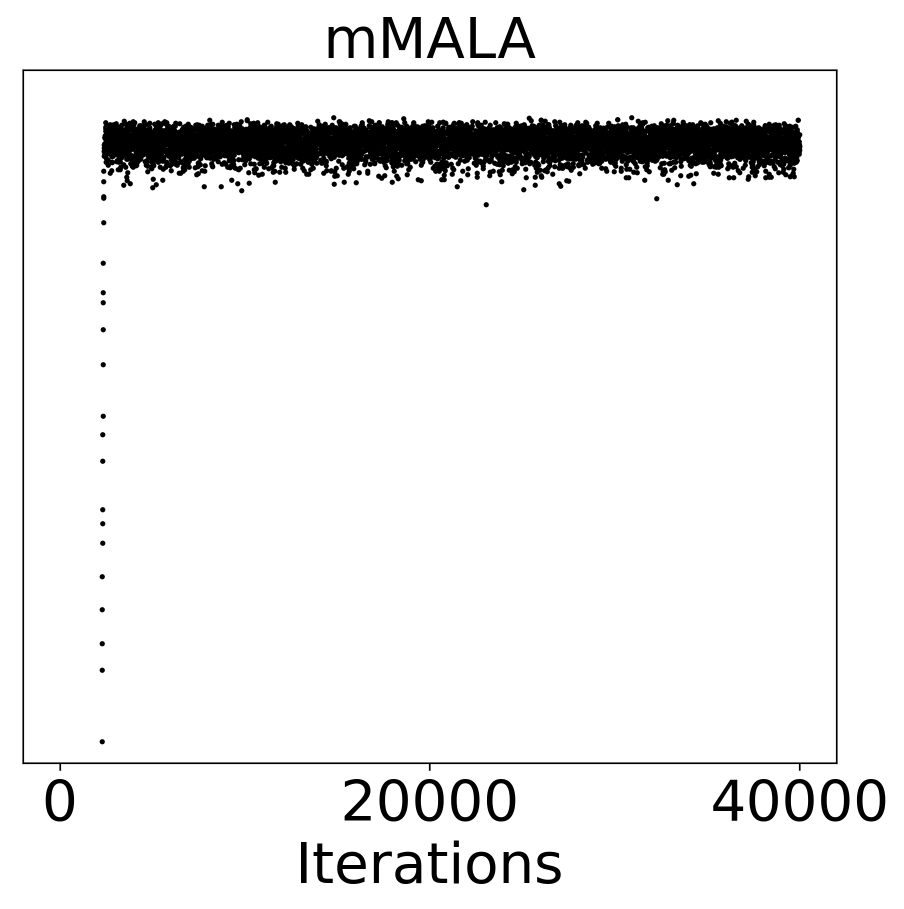} &
\includegraphics[scale=0.19]
{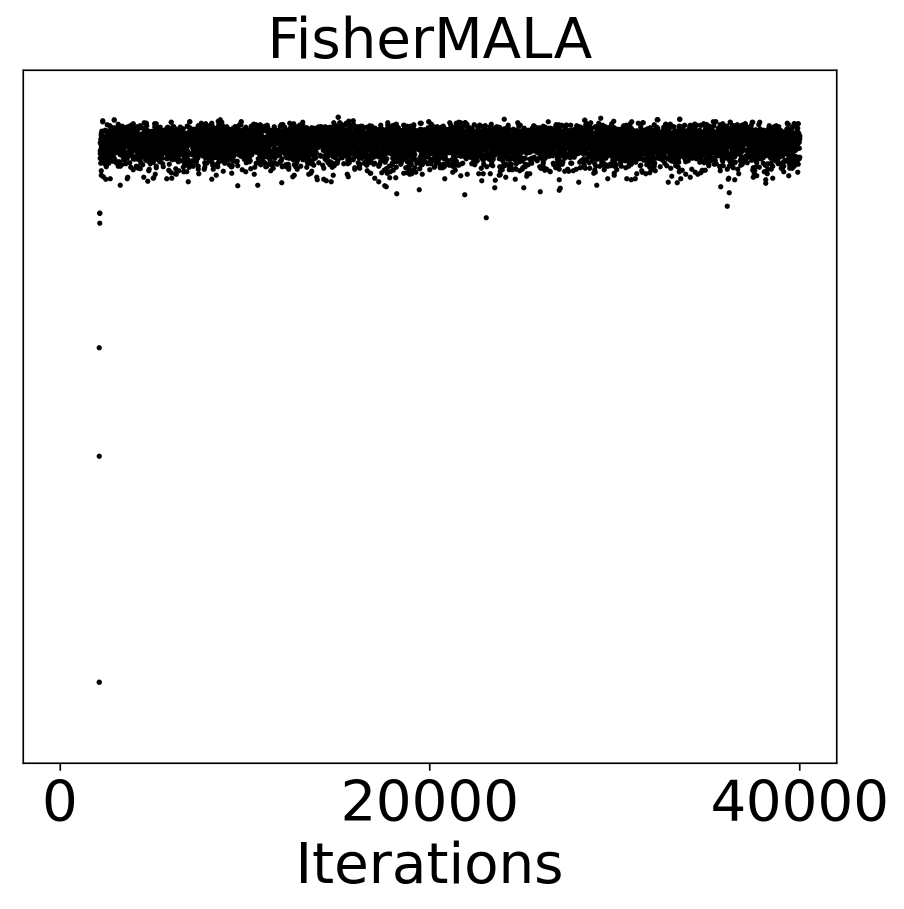}
\end{tabular}
\caption{The evolution of  the log-target across iterations in Australian Credit dataset.} 
\label{fig:australian_logtarget}
\end{figure}

\begin{figure}
\centering
\begin{tabular}{cccc}
\includegraphics[scale=0.19]
{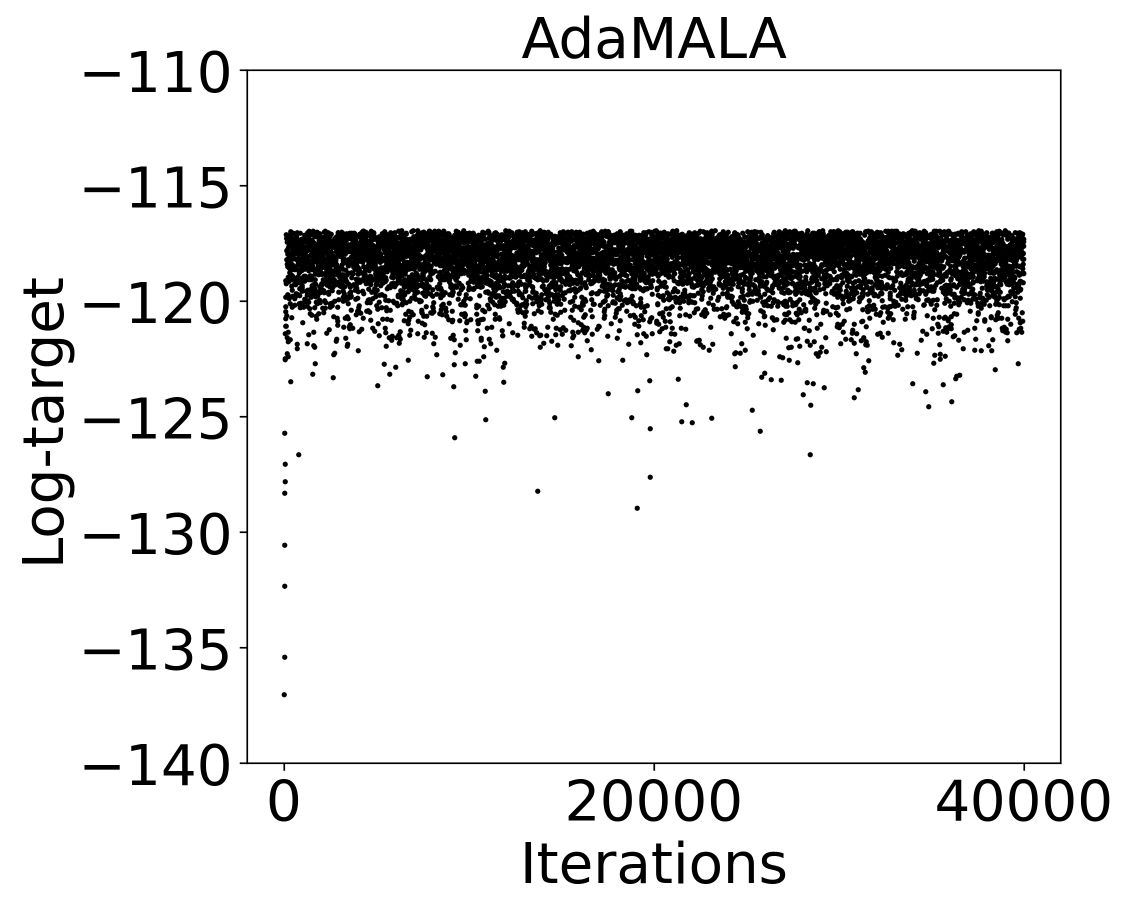} &
\includegraphics[scale=0.19]
{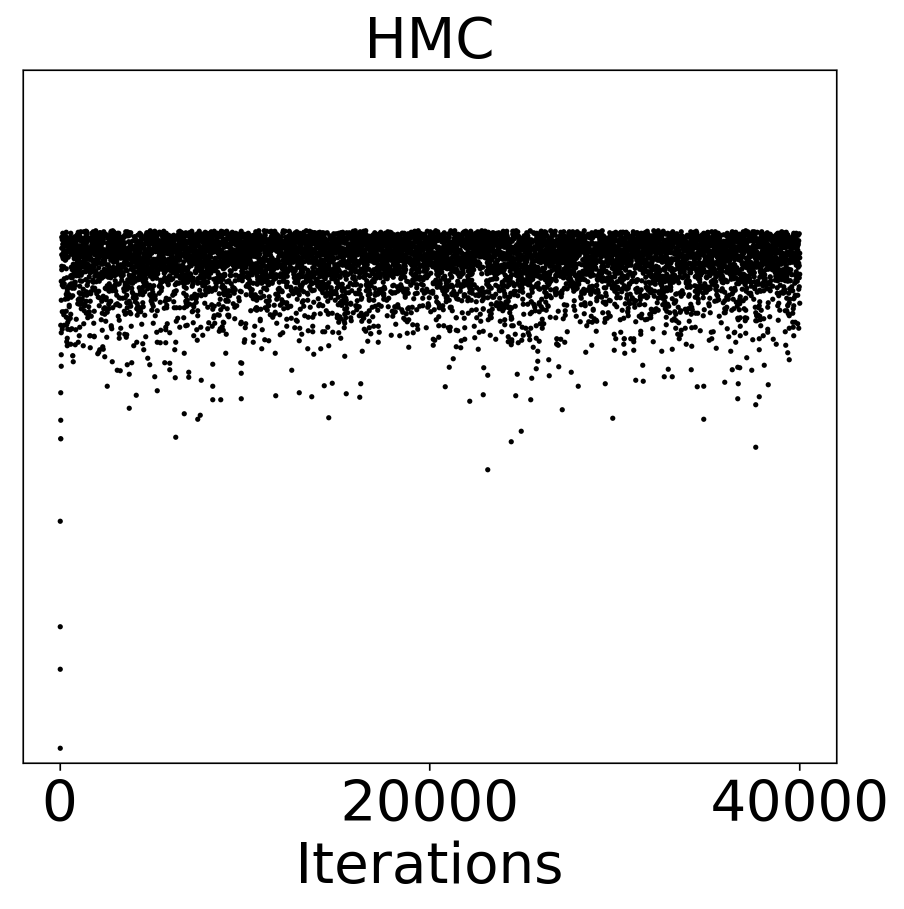} &
\includegraphics[scale=0.19]
{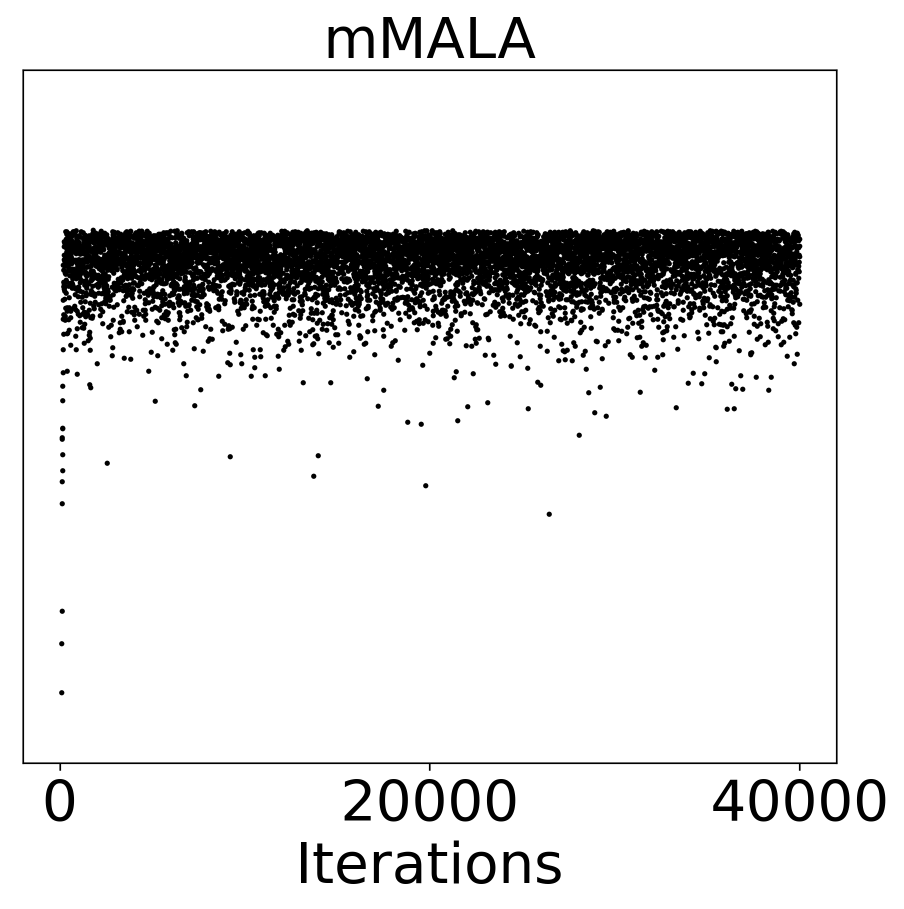} &
\includegraphics[scale=0.19]
{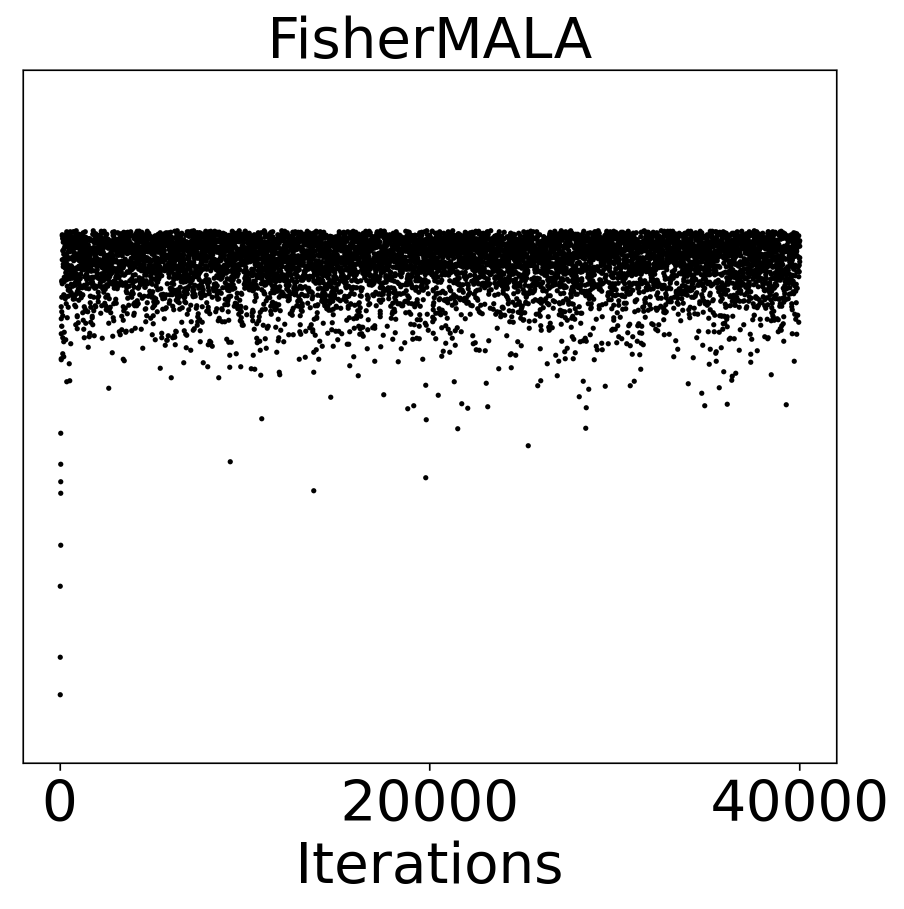}
\end{tabular}
\caption{The evolution of  the log-target across iterations in Ripley dataset.} 
\label{fig:ripley_logtarget}
\end{figure}

\subsection{The effect of Raoblackwellization and comparison with paired stochastic estimation
\label{sec:effectRaoblackwell}
} 

Finally, we compare three versions of FisherMALA: (i) The one that uses the Raoblackwellized signal $s_n^\delta$ from Eq.\ \eqref{eq:differenceRB}, which is our main 
proposed method used in the main paper and all previous results (in this section we will denote this as FisherMALA-with-RB), (ii)
 the one that uses the initial score function difference from Eq.\ \eqref{eq:differenceNoRB} (FisherMALA-no-RB) and (iii) 
 and FisherMALA with paired mean and covariance stochastic estimation (FisherMALA-paired-est) as descibed in Appendix
\ref{app:FisherMALA2}. Table \ref{table:all_rb} compares the three versions of FisherMALA  in terms of ESS for all problems, 
which shows that FisherMALA-paired-est is significantly worse than the other two methods that 
learn based on score function increments.   These two latter methods, FisherMALA-with-RB and FisherMALA-no-RB,  
 have  similar performance without significant difference (the highest difference in terms of Min ESS is in Pima Indians dataset, but still not statistically significant).  
  
 Figure  \ref{fig:raoblackwell} displays the Frobenius norms for FisherMALA with Raoblackwellization and FisherMALA without Raoblackwellization in 
 the two $100$-dimensional Gaussian targets. It shows that the Raoblackwellized signal $s_n^\delta$ leads to slightly faster convergence, 
 which agrees with the theory that says that Raoblackwellization should reduce the variance.

 Finally,  Table \ref{table:non_centered} reports numerical performance of the non-centered version of FisherMALA where we learn directly from the score function 
 vectors $s_n$, i.e.\ without centering  or using score function increments. From this table we can see that FisherMALA (non-centered) performs worse than the 
 other FisherMALA variants, and only on Ripley dataset works equally well with the rest.

\begin{figure}
\centering
\begin{tabular}{cc}
\includegraphics[scale=0.25]
{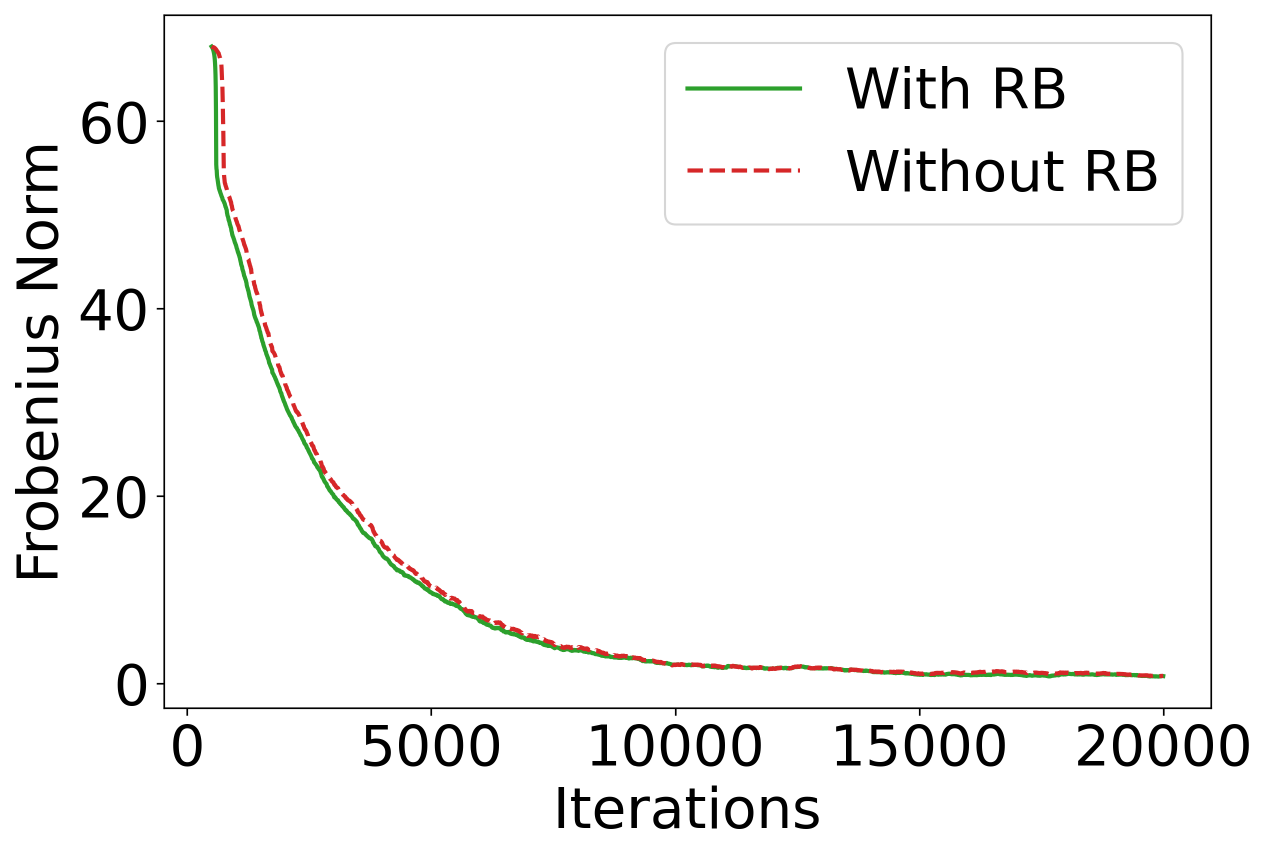} &
\includegraphics[scale=0.25]
{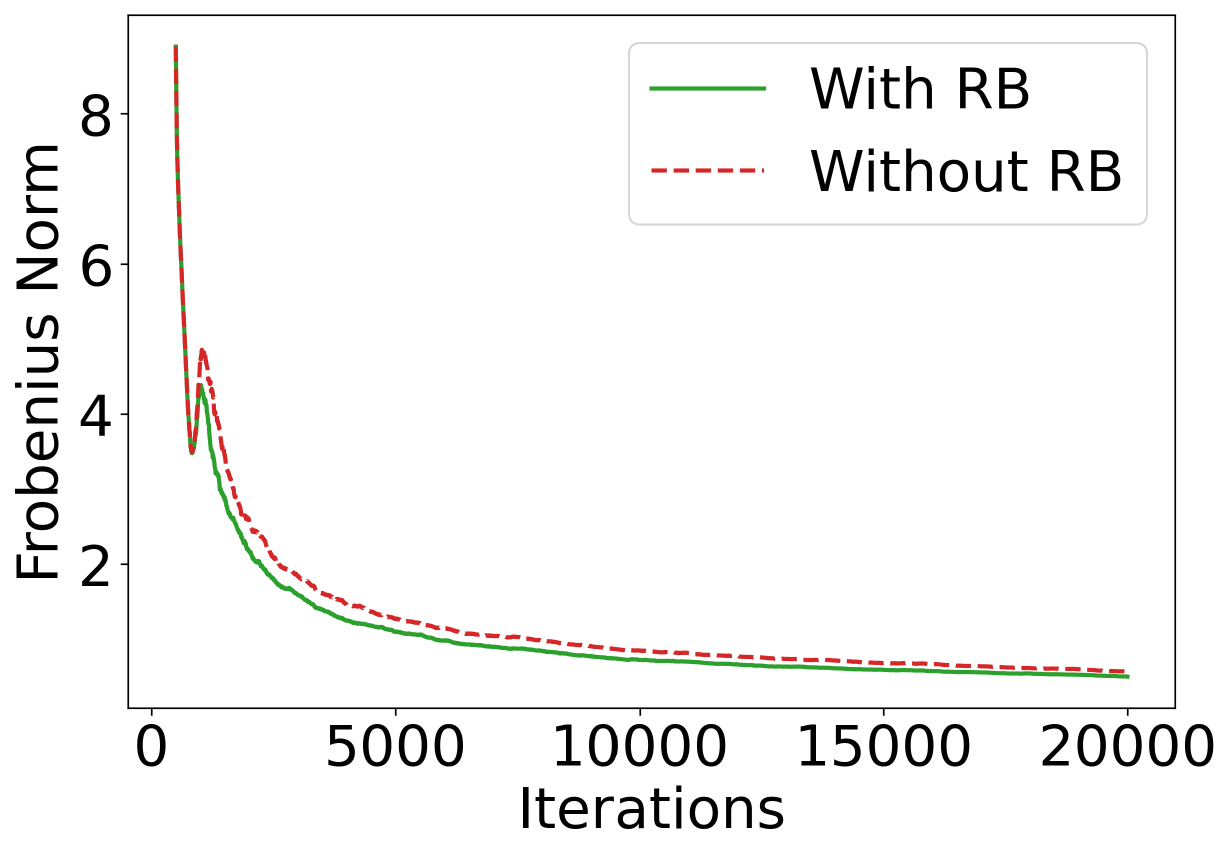} 
\end{tabular}
\caption{The effect of Raoblackwellization. Left panel shows the evolution of the Frobenius norm in the GP target and right panel for the inhomogeneous 
Gaussian target.}
\label{fig:raoblackwell}
\end{figure}

\begin{table}
  \caption{Comparison of ESS scores for three versions of FisherMALA:  the first with Raoblackwellized score function differences in 
  \eqref{eq:differenceRB}, the second based on the initial adaptation signal of score function differences from \eqref{eq:differenceNoRB}, and the third 
  based on paired stochastic estimation.}
  \label{table:all_rb}
  \centering
  \begin{tabular}{llll}
    \toprule
& Max ESS & Median ESS & Min ESS \\
\midrule
\emph{GP target} & & & \\
FisherMALA-with-RB & $2096.259 \pm 94.751$ & $1923.753 \pm 95.820$ & $1784.962 \pm 104.440$ \\
FisherMALA-no-RB & $2064.940 \pm 87.943$ & $1916.990 \pm 85.208$ & $1794.114 \pm 103.711$ \\
FisherMALA-paired-est & $1802.141 \pm 142.784$ & $1583.570 \pm 109.241$ & $1226.303 \pm 244.752$ \\
\midrule
\emph{Inhomog. Gaussian} & & & \\  
FisherMALA-with-RB & $2347.340 \pm 70.234$ & $2002.579 \pm 30.001$ & $1500.983 \pm 67.087$ \\
FisherMALA-no-RB & $2351.481 \pm 78.894$ & $2012.243 \pm 30.024$ & $1489.617 \pm 133.619$ \\
FisherMALA-paired-est & $1941.994 \pm 106.710$ & $1147.138 \pm 61.591$ & $109.160 \pm 57.998$ \\
\midrule
\emph{Heart} & & & \\
FisherMALA-with-RB & $4864.278 \pm 103.277$ & $4474.288 \pm 102.029$ & $3954.793 \pm 199.832$ \\
FisherMALA-no-RB & $4893.063 \pm 107.068$ & $4455.591 \pm 98.542$ & $3977.741 \pm 194.922$ \\
FisherMALA-paired-est & $4804.365 \pm 176.747$ & $2519.187 \pm 693.945$ & $441.434 \pm 386.287$ \\
\midrule
\emph{German Credit} & & & \\ 
FisherMALA-with-RB & $3951.807 \pm 78.858$ & $3582.184 \pm 90.551$ & $3011.483 \pm 258.154$ \\
FisherMALA-no-RB & $3979.744 \pm 79.647$ & $3616.894 \pm 104.722$ & $3031.384 \pm 228.345$ \\
FisherMALA-paired-est & $3960.773 \pm 105.169$ & $3097.557 \pm 252.619$ & $397.034 \pm 244.768$ \\
\midrule
\emph{Australian Credit} & & & \\ 
FisherMALA-with-RB & $4732.724 \pm 116.074$ & $4361.969 \pm 104.750$ & $3772.086 \pm 265.170$ \\
FisherMALA-no-RB & $4711.549 \pm 115.329$ & $4364.347 \pm 95.004$ & $3790.949 \pm 253.464$ \\
FisherMALA-paired-est & $4887.606 \pm 173.626$ & $3603.765 \pm 725.018$ & $84.202 \pm 44.750$ \\
\midrule
\emph{Ripley} & & & \\          
FisherMALA-with-RB & $9875.968 \pm 218.801$ & $9673.009 \pm 280.759$ & $9244.631 \pm 559.137$ \\
FisherMALA-no-RB & $9852.895 \pm 281.295$ & $9679.384 \pm 303.946$ & $9272.040 \pm 581.732$ \\
FisherMALA-paired-est & $9869.053 \pm 321.031$ & $9598.430 \pm 330.766$ & $9217.330 \pm 584.224$ \\
\midrule
\emph{Pima Indians} & & & \\ 
FisherMALA-with-RB & $6437.419 \pm 207.548$ & $5981.960 \pm 156.072$ & $5628.541 \pm 168.425$ \\
FisherMALA-no-RB & $6448.999 \pm 199.817$ & $5977.292 \pm 122.852$ & $5585.217 \pm 160.586$ \\
FisherMALA-paired-est & $6048.419 \pm 650.262$ & $2618.271 \pm 889.425$ & $788.687 \pm 388.978$ \\
\midrule
\emph{Caravan} & & & \\  
FisherMALA-with-RB & $2257.737 \pm 45.289$ & $1920.903 \pm 55.821$ & $498.016 \pm 96.692$ \\
FisherMALA-no-RB & $2241.262 \pm 47.873$ & $1908.045 \pm 62.430$ & $509.913 \pm 115.563$ \\
FisherMALA-paired-est & $1930.109 \pm 208.848$ & $1107.987 \pm 83.439$ & $87.456 \pm 90.858$ \\
\midrule
\emph{MNIST} & & & \\   
FisherMALA-with-RB & $1053.455 \pm 35.680$ & $811.522 \pm 19.165$ & $439.580 \pm 52.800$ \\
FisherMALA-no-RB & $1036.138 \pm 32.399$ & $803.210 \pm 16.163$ & $437.325 \pm 40.040$ \\
FisherMALA-paired est & $301.055 \pm 37.597$ & $13.819 \pm 1.127$ & $3.176 \pm 0.113$ \\
  \bottomrule
  \end{tabular}
\end{table}

{\small 
\begin{table}
  \caption{Performance of FisherMALA (non-centered), in a subset of the targets,  
  which learns directly from the score function vectors $s_n$.}
  \label{table:non_centered}
  \centering
  \begin{tabular}{llll}
    \toprule
& Max ESS & Median ESS & Min ESS \\
\midrule
\emph{GP target} & & & \\
FisherMALA (non-centered) & $1740.943 \pm 157.871$ & $518.924 \pm 579.639$ & $48.218 \pm 117.349$ \\
\midrule
\emph{Ripley} & & & \\          
FisherMALA (non-centered) & $9881.540 \pm 353.377$ & $9636.357 \pm 313.009$ & $9237.885 \pm 710.741$ \\
\midrule
\emph{Pima Indians} & & & \\ 
FisherMALA (non-centered) & $5520.181 \pm 1781.518$ & $474.990 \pm 587.788$ & $65.313 \pm 59.316$ \\
\midrule
\emph{Caravan} & & & \\  
FisherMALA (non-centered) & $1602.723 \pm 164.497$ & $14.226 \pm 4.429$ & $3.298 \pm 0.141$ \\
\midrule
\emph{MNIST} & & & \\   
FisherMALA (non-centered) & $271.629 \pm 22.918$ & $22.147 \pm 1.683$ & $3.744 \pm 0.139$ \\
  \bottomrule
  \end{tabular}
\end{table}
}